\pdfoutput=1
\documentclass[11pt]{article}

\usepackage[T1]{fontenc}
\usepackage[preprint]{nexais}
\usepackage{newtxtext,newtxmath}
\usepackage[a4paper,margin=26mm]{geometry}
\usepackage[protrusion=true,expansion=false]{microtype}

\usepackage[round,authoryear]{natbib}
\usepackage{algorithm}
\usepackage{algorithmic}
\usepackage{mathrsfs}
\usepackage{xfrac}
\usepackage[flushleft]{threeparttable}
\usepackage[font=small,labelfont=bf]{caption}
\usepackage[shortlabels]{enumitem}
\usepackage{flafter}

\usepackage[breaklinks=true]{hyperref}
\nexaishypersetup
\usepackage[nameinlink,noabbrev]{cleveref}

\setlist{nosep,leftmargin=1.6em}
\newcommand{\rF}{\mathrm{F}}

\renewcommand{\vec}{\mathrm{vec}}
\newcommand{\rs}{\mathrm{s}}

\newcommand{\lv}{\left\langle}
\newcommand{\rv}{\right\rangle}

\renewcommand{\tr}{\textrm{tr}}
\newcommand{\pen}{{\scriptstyle\textrm{pen}}}
\newcommand{\appr}{{\scriptstyle\textrm{appr}}}
\newcommand{\ols}{{\scriptstyle\textrm{ols}}}
\newcommand{\iid}{i.i.d.\,}
\newcommand{\Rom}[1]{\text{\uppercase\expandafter{\romannumeral #1\relax}}}
\newcommand{\nn}{\nonumber}

\newtheorem{condition}{Condition}[section]

\begin{document}

\title{Online Generalized Sparse Regression: How Does Overparametrization Help?}

\author[1]{Shuoguang Yang\thanks{E-mail: \href{mailto:yangsg@ust.hk}{yangsg@ust.hk}. }}
\author[2]{Qiang Sun\thanks{Corresponding author; E-mail: \href{mailto:qsunstats@gmail.com}{qsunstats@gmail.com}.}}

\affiliation[1]{HKUST}
\affiliation[2]{University of Toronto and MBZUAI}

\paperabstract{%

Regularized sparse regression has been extensively studied in the offline setting,  but online formulation remains relatively under-explored. This gap stems from four key challenges: (i) the infeasibility of dynamically updating the regularization parameter in every online round, (ii) managing storage and memory complexity, (iii) enabling real-time computation via \textit{closed-form updates} rather than solving full optimization problems at each round, and (iv) achieving optimal statistical guarantees under realistic assumptions. In this paper, we propose an online generalized-sparsity-constrained regression framework, focusing on online cardinality-constrained linear regression and low-rank matrix sensing. Unlike online regularized regression, our constrained formulation eliminates the need for dynamic parameter tuning. We introduce an efficient online hard-thresholding algorithm that performs closed-form updates and requires storing only summary statistics, making it computationally, memory, and storage efficient. Despite the inherent nonconvexity and combinatorial nature of the formulation, our algorithm achieves global convergence at the optimal statistical rate under realistic assumptions, provided that the projection set is properly overparameterized. Numerical experiments demonstrate that our method consistently outperforms state-of-the-art alternatives.

}

\date{August 2026}
\pdfsubject{Research article on online generalized sparse regression}
\pdfkeywords{cardinality constraints, generalized sparsity, online hard thresholding, optimality, rank constraints, streaming data}

\maketitle

\tableofcontents

\section{Introduction}

Modern data analysis often involves a massive number of features, which poses both statistical and computational challenges. A common assumption is that the underlying signals are sparse, a property believed to hold in many applications such as genome-wide association studies, image processing, signal processing, and medical imaging. To exploit this sparse structure, a popular approach is to adopt the penalized empirical risk minimization approach:
\#\label{eq:offline}
\widehat \Theta^\pen =
\argmin_{\Theta \in \RR^{d_1 \times d_2}} \left \{ \cL(\Theta) + \cR(\Theta;\lambda) \right\},
\#
where $\cL(\cdot)$ is a loss function, $\Theta \in \RR^{d_1\times d_2}$ is a matrix or a vector ($d_2 = 1$),  $\cR(\Theta;\lambda)$ is a sparsity-inducing regularizer such as the LASSO penalty
\citep{tibs1996regression} or the nuclear norm penalty \citep{recht2010guaranteed},
and $\lambda$ is a regularization parameter. Both statistical  and computational  aspects of this penalized approach have been  studied extensively  in the offline setting,  where all data are readily accessible at once \citep{agarwal2012fast, loh2013regularized, fan2018lamm}. By contrast, sparse regression in the online setting, where data arrive sequentially, has received relatively little attention.

Consider the online setting, where in each round $t\geq 1$,  we receive a label-feature pair $(Y_t, X_t)\in \RR \times \RR^{d_1\times d_2}$.  Our goal is to construct an  estimator $\Theta_t$ in real time  to  estimate the ground-truth coefficient matrix $\Theta^*\in\RR^{d_1 \times d_2}$.  Although online regression and stochastic optimization have been studied extensively \citep{kushner:1997,rakhlin2011making, fan2018statistical}, extending the offline sparse regression~\eqref{eq:offline} to the online setting  is challenging  due to presence of the regularizer  $\cR(\Theta, \lambda)$. The main challenges are fourfold:
\begin{enumerate}
\item  \textbf{Infeasibility of dynamic regularization.} Achieving optimal statistical performance at each round requires dynamically updating the regularization parameter $\lambda$, which depends on both the unknown error distribution and the growing sample size.  Manually tuning $\lambda$ in every round is often infeasible in practice.

\item  \textbf{Memory and storage complexity.} Naively storing all past data up to round $t$ leads to $\cO(t d_1 d_2)$ storage and memory costs, which are impractical for large-scale settings.

\item  \textbf{Requirement for real-time computation.} Parameter updates must be computationally simple  to enable real-time operation,  especially in large-scale applications.

\item  \textbf{Realistic assumptions.} Achieving optimal statistical guarantees under realistic assumptions, such as allowing arbitrary regularized/sparse condition numbers, is essential for practical deployment.
\end{enumerate}

Several existing methods have attempted to address these challenges, but none simultaneously resolve all four. For instance, \cite{kale2017adaptive} proposed an online sparse linear regression algorithm that periodically solves a randomized Dantzig-selector linear program \citep{candes2007dantzig} and then applies a sparse projection; the linear program is solved when the round index is a power of two. The method also stores an estimated design matrix whose number of rows grows with $t$, so it does not provide bounded storage. Moreover, their results require a bound on the restricted isometry property constant $\leq 1/5$, implying a sparse condition number $\leq 3/2$, which is restrictive, since real-world high-dimensional datasets can have much larger condition numbers. \cite{steinhardt2014statistics} proposed a streaming sparse-regression method with stochastic-gradient-like computational and memory requirements, but its support-recovery theory requires an irrepresentability condition \citep{zhao2006model}. \cite{yang2022streaming} introduced a new loss function and a memory-efficient online linearized LASSO algorithm that allows arbitrary condition numbers, but each round is still defined by solving an $\ell_1$-regularized optimization problem rather than by a fixed number of closed-form updates.
This naturally raises the question:
\begin{quote}
   \it Can we design online algorithms that are storage- and memory-efficient, admit simple updates at each round, and achieve optimal statistical convergence guarantees under reasonable assumptions, without requiring dynamic regularization?
\end{quote}

This paper provides a positive answer to the question above. Instead of a regularized approach, we propose the following online constrained sparse regression framework:
\#\label{eq:online}
&\widehat \Theta_t \in \argmin_{\Theta \in \RR^{d_1 \times d_2}} \left  \{ \cL_t(\Theta):  \Upsilon(\Theta ) \leq s^* \right \} ~~\text{for}~~t\geq 1,
\#
where
 $\Upsilon(\Theta)$ represents the \emph{generalized sparsity} of $\Theta$, and  $s^*$ denotes the unknown true generalized sparsity of the ground-truth coefficient matrix $\Theta^*$.  We focus on problems: (i) online cardinality-constrained linear regression, where $\Theta \in \RR^{d}$ is a vector and $\Upsilon(\Theta) = \| \Theta\|_0$, the number of non-zero entries; (ii) online rank-constrained matrix sensing, where $\Theta \in \RR^{d_1\times d_2}$ is a matrix and $\Upsilon(\Theta) = \rank(\Theta)$. A solution to problem (i) is called \emph{sparse} if it has a small cardinality, and a solution to problem (ii) is called \emph{low-rank} if it has a small rank.

One immediate advantage of formulating online generalized sparse regression in this constrained form is that it eliminates the need to dynamically tune a regularization parameter such as $\lambda$ in~\eqref{eq:offline}. The drawback is also apparent: the resulting generalized-sparsity-constrained problem is combinatorial and typically NP-hard \citep{natarajan1995sparse,foster2015variable}.
Assuming natural sparse strong convexity and smoothness conditions, which often hold in practice, \cite{jain2014iterative} showed that projected gradient descent with an enlarged projection set $\{\Theta : \|\Theta\|_0 \leq s = 32 \kappa^2 s^*\}$ and a strict step size $\eta = 2/(3L)$ converges to the underlying $\Theta^*$ in function value, where $\kappa$ and $L$ are the sparse condition number and smoothness parameter, respectively\footnote{Sparse condition number, sparse strong convexity, and smoothness parameters are defined formally in Section~\ref{sec:main}.}. However, efficient online algorithms that admit simple updates and provable statistical guarantees remain unavailable, particularly when the loss function changes across rounds. Additionally, strictly requiring a step size of $2/(3L)$ may be overly restrictive.

We propose an online hard thresholding (OHT) algorithm to solve \eqref{eq:online}. Instead of solving each subproblem exactly, we perform $K$ hard thresholding steps in each online round:
\#\label{eq:online_K}
\Theta_{t, k +1} = \Pi_{\cC_s} \left( \Theta_{t,k} - \eta_{t,k}  \nabla \cL_{t}( \Theta_{t, k }) \right), \ \ \text{ for } k =0,1,\cdots, K -1,
\#
where $\Pi_{\cC_s}(\cdot)$ denotes the projection operator  onto the set $\cC_s$ and $\Theta_0 = \Theta_{0, 0}$ is the initialization point.  When $\Upsilon(\Theta) = \| \Theta\|_0$, $\Pi_{\cC_s}(\cdot)$ performs hard thresholding by retaining the $s$ entries with the largest magnitude and setting the rest to zero.   When $\Upsilon(\Theta) = \rank(\Theta)$, $\Pi_{\cC_s}(\cdot)$ returns the best rank-$s$ approximation.

Assuming  generalized sparse strong convexity and smoothness, and setting the step size $\eta_{t,k} = \eta \leq 1/L$ with $L$ being  the sparse strong smoothness parameter, our algorithm with an overparameterized projection set  $s > \kappa^2 s^*$  converges geometrically to the ground-truth  coefficients, up to the optimal statistical error of $s\sigma_x^2/t$, where $ \sigma_x^2$ is a constant depending on the design and noise.  Informally, we have
\#\label{result:informal}
\EE\left[ \|\Theta_{t, K} - \Theta^* \|_2^2\right]\lesssim \underbrace{\delta^t \cdot \EE\left[\|\Theta_{0} - \Theta^*\|_2^2\right]}_{\text{geometric convergence}} + \underbrace{\frac{s^* \sigma_x^2}{t}}_{\text{opt. stat. error}},~~\text{for some}~\delta\in(0,1).
\#
For logarithmically large round $t$, our algorithm achieves the optimal statistical error,  as if one had solved the offline constrained program had been solved exactly using all data up to round $t$. Moreover, our algorithm achieves constant memory and storage complexity: $\cO(d_1^2)$ for linear regression and $\cO( d_1^2 d_2^2 )$ for  matrix sensing. Crucially, these complexities do not scale with the number of online rounds $t$. The result holds for arbitrarily large sparse condition numbers.

Remarkably, this is the first algorithm to simultaneously overcome all major challenges: eliminating the need for dynamic regularization, ensuring storage and memory efficiency, admitting closed-form per-round updates, and achieving optimal statistical guarantees under realistic assumptions. To establish generalized sparse strong convexity and smoothness, we initialize with a batch of samples and then stream on top of it; we also extend our results to scenarios where no such initial batch is available.

\subsection{Related Work}

We review additional works related to ours.  Online regularized sparse optimization has been studied in several works such as \cite{bertsekas2011incremental,duchi2011adaptive,xiao2009dual}, which analyzed the convergence behavior of various online methods. However, these studies assumed a fixed regularization parameter, and thus  their results are not applicable to our setting.  \cite{han2024online} proposed a one-pass debiased stochastic gradient descent method for online confidence intervals. Its goal is statistical inference rather than maintaining a sparse estimator under a changing regularization path, so it does not address our constrained sparse-optimization problem. Recently, \cite{fan2018statistical}
proposed a two-stage algorithm that first identifies the support set by solving a penalized offline problem during a burn-in stage, and then performs truncated gradient descent restricted to that support. Their method, however, requires both a minimum signal strength assumption and a sufficiently large burn-in sample size, while our method does not rely on either. Our work is partly motivated by \cite{liu2020between}, who studied optimal thresholding algorithms for offline sparse linear regression. There are, however, two key differences: (i) we consider an online setting where the loss function changes from round to round, and (ii) we establish a new descent lemma and leverage it to provide last-iterate guarantees, from which we derive the optimal statistical rate for our algorithm. By contrast, \cite{liu2020between} did not provide last-iterate guarantees, and hence their results cannot be applied in the online setting.  Taken together, our work introduces the first algorithm that simultaneously addresses all four challenges: it eliminates the need for dynamic regularization, ensures storage and memory efficiency, admits closed-form updates, and achieves optimal statistical guarantees under realistic assumptions.

\subsection{Notation}

We summarize here the notation that will be used throughout the paper. We  use $c$ and $C$ to denote generic constants which may change from line to line. For two sequences of real numbers $\{ a_n \}_{n\geq 1}$ and $\{ b_n \}_{n\geq 1}$, we write $a_n = \cO(b_n)$ or $a_n \lesssim b_n$ if there exists some constant $C>0$ such that $a_n \leq C b_n$ for all $n\geq 1$.  
We use $a_n \asymp b_n$ to denote $a_n\gtrsim b_n$ and $a_n\lesssim b_n$.
The $\log$ operator is understood to be with respect to the base $e$. For a function $f(x)$, we use $\nabla f(x)$  to denote its  derivative.
For a vector $u$ and any $p\geq 1$, we use $\|u\|_p$ to denote its $p$-th norm, use $\|u\|_0$ to denote the cardinality of $u$, and use $\| a \|_\infty = \max_{j \leq d } | a_j|$ to denote the maximum of absolute  entries of $a \in \RR^d$.
For a matrix $\Theta$, we use  $\| \Theta \|_{2} $ to denote its spectral norm and use
$\|\Theta\|_\rF$ to denote its Frobenius norm, $\tr(\Theta)$ denotes its trace. We shall note that the Frobenius norm coincides with the $\ell_2$ norm when $\Theta$ is a vector.
For any vector $x \in \RR^d$, we denote by $x_{\cS} \in \RR^d$ a vector whose entries within $\cS$ are the same as those in $x$ but the rest entries are all zeros.
For any matrix $X \in \RR^{d_1 \times d_2}$, we denote by $X_{\cS} \in \RR^{d_1 \times d_2}$ a matrix such that the columns within $\cS$ are the same as those in $X$ but the rest entries are all zeros.
{We use $\text{vec}(X) \in \RR^{d_1d_2}$ to denote the vectorized representation of a matrix $X \in \RR^{d_1\times d_2}$.}
We denote by $\cS^*$ the support of the underlying coefficient $\Theta^*$.
For any two linear spaces $S_1$ and $S_2$, we denote by $S_1 + S_2 = \{ x_1 + x_2 \mid x_1 \in S_1, x_2 \in S_2\}$.
For any matrix $A$ and any linear space $S$ with the basis matrix being $U$ (i.e., $U$ consists of the basis vectors as its columns),  let $P_S(A) = U U^\top A$, that is,  $P_S$ denotes the operator that projects the columns of $A$  onto the linear space $S$. Constants $C$ and $c$ may vary from line to line.

\section{Why Overparameterized Online Hard Thresholding Works}

In this section, we formally define the problem and highlight the challenges involved in extending offline penalized regression to the online setting. We also discuss how overparameterization can help overcome the difficulty of solving a seemingly NP-hard problem. Finally, we introduce our proposed algorithm and provide an analysis of its runtime complexity.

\subsection{Problem Setup}

We consider a scenario where independent and identically distributed (i.i.d.) data  arrive sequentially and are generated according to the following realizable linear model
\#\label{eq:model}
Y_{j}= \lv X_{j}, \Theta^* \rv  +\epsilon_{j}, ~j\geq 1,
\#
where $Y_{j}\in \RR$ is the response,  $X_{j}\in \RR^{d_1 \times d_2}$ is the covariate, $\epsilon_{j}$ is the random error, and $\Theta^* \in \RR^{d_1 \times d_2} $ is the ground-truth coefficient matrix  with a generalized sparse structure.
When $d_2=1$, the matrices $X_j$ and $\Theta^*$ reduce to vectors.  We assume that $\Theta^*$ is generalized $s^*$-sparse, i.e., $\Upsilon(\Theta^*) = s^*$, which corresponds to cardinality $s^*$ in sparse linear regression and rank $s^*$ in low-rank matrix regression.

We focus on the squared loss, denoting by $\ell_j$  the loss for the $j$-th data point:
\begin{equation}\label{eq:loss_each_batch}
\ell_j(\Theta ) : =  \ell \Big (Y_{j}, \langle  \Theta, X_{j} \rangle  \Big )=\frac{1}{2}  \Big (Y_{j}- \langle  \Theta, X_{j} \rangle    \Big )^2.
\end{equation}
Suppose we have access to an initial batch of sample size $t_0$ consisting of \iid~samples $(X_{0_j} , Y_{0_j} )_{1 \leq j \leq t_0}$.  Then the  loss for the initial batch  and the cumulative loss up to (including) round $t$ are
\$
\ell_0(\Theta) = \frac{1}{2} \sum^{t_0}_{j = 1} \Big (Y_{0_j} - \langle \Theta , X_{0_j} \rangle \Big )^2 ~~\text{and}~~\cL_t(\Theta) =    \frac{1}{t_0+t}\Big (\ell_0(\Theta) + \sum_{j=1}^t\ell_j(\Theta )\Big), ~~\text{respectively}.
\$

\subsection{Why Extending Offline Penalized Regression to the Online Setting Is Challenging}

We focus on the linear regression problem with $d = d_1$ and $d_2=1$ in \eqref{eq:model} as an illustration, though the same reasoning applies to low-rank matrix sensing. A natural online extension of the offline penalized regression \eqref{eq:offline} is
\#\label{online:penalized}
\widehat \Theta_t^\pen =
\argmin_{\Theta \in \RR^{d}} \left \{ \cL_t(\Theta) +\lambda_t \|\Theta\|_1 \right\},
\#
where $\|\Theta\|_1 = \sum_{j=1}^d |\Theta_j|$ is the vector $\ell_1$ norm.

As discussed earlier, the main challenges in the online setting are: (1) the infeasibility of dynamic regularization, (2) managing memory and storage limitations, and (3) the requirement for real-time computation. Previous works \citep{han2022inference, yang2022streaming} adopted related regularized streaming formulations and addressed the second challenge by storing only summary statistics. However, these approaches update regularization or tuning parameters over time, which becomes difficult especially when the errors are heterogeneous. In the case of i.i.d.~errors, one can make $\lambda_t$ adaptive to the sample size, reducing dynamic regularization to careful tuning of the standard error in early rounds. Moreover, these estimators are defined through regularized optimization problems at successive rounds or data batches. A naive alternative is to perform a few simple updates, such as sub-gradient steps, and use the last iterate as the solution for each round. Yet, this approach does not guarantee the desired statistical properties. We illustrate this issue below.

\begin{figure}[t]
    \centering
\includegraphics[width=.75\linewidth]{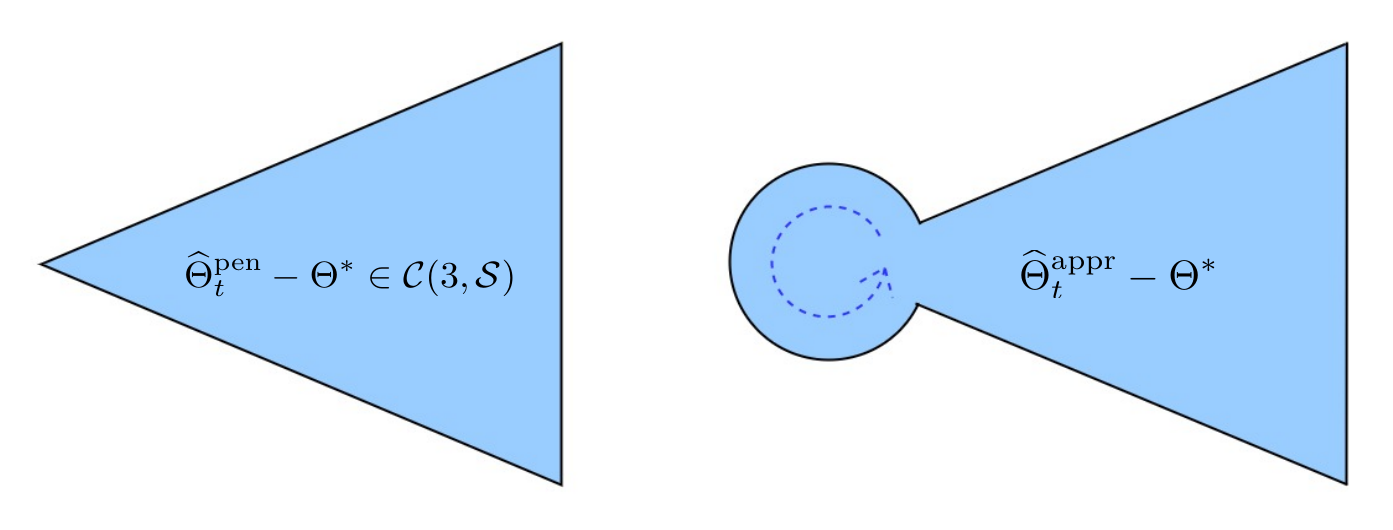}
      \caption{ The left panel shows the cone-like set $\cC(3, \cS)$ while the right panel depicts the region where $\widehat\Theta_t^\appr-\Theta^*$ may fall. }
    \label{fig:cone}
\end{figure}

We denote the approximate solutions obtained after a few simple updates as $\widehat\Theta_t^\appr$. Since the objective function in \eqref{online:penalized} changes from round to round, the problem is  a moving target optimization problem, where the updates must converge sufficiently fast, ideally at a geometric rate, to keep up. However, geometric convergence is only guaranteed within a cone-like set
\$
\cC(3, \cS) = \left\{\Theta: \|(\Theta-\Theta^*)_{\cS^c}\|_{1}\leq   3 \|(\Theta-\Theta^*)_{\cS}\|_{1},\,  \cS \text{ is the support of } \Theta^* \right\}.
\$
It is known that the exact solution at round $t$ falls into the above set whenever $\lambda_t\geq 2\|\nabla \cL_t(\Theta^*)\|_\infty $ \cite[Lemma 3.1]{yang2022streaming}. In contrast, approximate solutions need not remain in this set.
In fact, $\widehat\Theta_t^\appr - \Theta^*$, obtained from small perturbations of $\widehat\Theta_t^\pen-\Theta^*$, can drift outside  $\cC(3, \cS)$ and instead form a neighborhood around it; see Figure \ref{fig:cone}. Once outside the set, the restricted strong convexity \citep{raskutti2010restricted},  fails to hold for  $\widehat\Theta_t^\appr - \Theta^*$ because it runs out of $\cC(3, \cS)$, i.e., there may NOT exist a $\kappa >0$ such that
\$
{\langle \nabla \cL_t(\widehat\Theta_t) - \nabla \cL_t(\Theta^*), \widehat\Theta_t^\pen - \Theta^*\rangle}\geq \kappa \,  \|\widehat\Theta_t^\pen -\Theta^*\|_2^2.
\$
This leads to non-identifiability of approximate solutions, preventing geometric convergence or even desirable statistical guarantees. Related online sparse-regression procedures therefore rely on repeated global optimization: \cite{kale2017adaptive} solve a linear program at periodic rounds, \cite{yang2022streaming} solve an $\ell_1$-regularized problem at every round, and \cite{han2022inference} solve successive regularized estimating problems as data batches arrive. Moreover, dynamically updating $\lambda_t$ is practically difficult, as its optimal value depends on unknown problem-dependent quantities. By contrast, online constrained optimization sidesteps this issue entirely, as it does not require dynamically updated regularization parameters.

\subsection{Moving to Online Constrained Regression}

We now turn to  online constrained regression, which circumvents both the challenges of dynamic regularization and the violation of the restricted strong convexity condition. Upon receiving the $t$-th data point in round $t$, a straightforward approach is to run offline generalized-sparsity-constrained least squares to estimate the regression coefficients $\Theta^*$:
\begin{align}\label{eq:beta_t}
\hat \Theta_t
= \argmin_{\Theta \in \RR^{d_1 \times d_2}}  \Big \{  \ell_0(\Theta) + \sum_{j=1}^t\ell_j(\Theta) \Big \}
~~\text{s.t.}~~  \Upsilon( \Theta ) \leq s^*,
\end{align}
where $\Upsilon( \Theta ) \leq s^*$ imposes a hard constraint on the generalized sparsity of  $\Theta^*$.

Two main challenges arise: (1) due to storage and memory limitations, it is infeasible to retain all past data to optimize \eqref{eq:beta_t} directly; and (2) solving cardinality-constrained regression is NP-hard in general \citep{natarajan1995sparse, foster2015variable}. Challenge (1) can be addressed by maintaining suitable summary statistics, which we introduce later. To address (2), we first consider the case of linear regression. While \cite{foster2015variable} established strong computational lower bounds for sparse linear regression under standard complexity assumptions, efficient globally convergent algorithms are possible when the design matrix is well-conditioned.

Now let us examine the optimization problem in the $t$-th round more carefully:
\$
\hat \Theta_t
= \argmin_{\|\Theta\|_0\leq s^*} \frac{ \|\YY - \XX\Theta\|_2^2}{t+t_0},
\$
where $\YY$ consists of all labels and the rows of $\XX$ consist of all covariates up to round $t$ (we omit the dependence on $t$).  As a warm-up, consider the special case where the feature matrix $\XX$ is orthonormal, i.e., $\XX^\T \XX = (t+t_0) I$. In this setting, $\hat \Theta_t$ has a closed form:
\$
\hat\Theta_t = \Pi_{\cC_{s^*}}\left( \widehat\Theta^\ols \right) =  \Pi_{\cC_{s^*}}\left(\left(\XX^\T \XX\right)^{-1}\XX^\T \YY\right),
\$
where $\cC_{s^*}$ is the set of $d$-dimensional vectors with at most {$s^*$} nonzeros, and  $\Pi_{\cC_{s^*}}$ is the projection operator onto this set $\cC_{s^*}$.  In other words,  $\widehat\Theta_t$ can be obtained by first computing the ordinary least square estimator and then hard thresholding  to retain the $s^*$ largest entries in magnitude.

Equivalently, $\hat\Theta_t$ can be obtained as the limiting solution of the iterative hard thresholding algorithm:
\#\label{eq:t_iter}
\Theta_{t,k+1} = \Pi_{\cC_{s^*}}\left(  \Theta_{t, k} - \eta\cdot \frac{\XX^\T\left( \XX\Theta - \YY \right)}{t+t_0} \right),
\#
starting from $\Theta_{t,0} = 0$, where $\eta$ is the step size.
Although this derivation assumes orthonormal designs, it suggests that this iterative algorithm may still converge to $\Theta^*$ under more general conditions.

Indeed, by assuming the restricted isometry property (RIP) constant $\delta_{3s^*}\leq 1/\sqrt{32}$, which bounds the sparse condition number to $\leq 1.43$ and allows slightly more general designs than orthonormal ones,  the seminal work by \cite{blumensath2009iterative} proved that the iterative hard thresholding algorithm \eqref{eq:t_iter} with $\eta =1$ recovers an approximate $\Theta_{t,k}$ satisfying
\$
\left\| \Theta_{t,k} -\Theta^*\right\|_2 \leq \frac{6\|\epsilon\|_2}{\sqrt{t+t_0}}
\$
for sufficiently large $k$, where $\epsilon \in \RR^{t+t_0}$ is the noise vector\footnote{As pointed out by \cite{blumensath2009iterative}, the step size is taken to be $\eta=1$ for simplicity; using a more general step sizes require the RIP constant to depend on the choice of $\eta$.}.

For the realizable model \eqref{eq:model} with \iid data points, a slightly modified analysis yields, for $k\gtrsim \log \left( \|\Theta_{t,0}-\Theta^*\|_2 \big/ \sqrt{(s\log d)/(t+t_0)} \right)$:
\$
\left\| \Theta_{t,k} -\Theta^*\right\|_2
&\leq \frac{1}{2} \left\|\Theta_{t,k-1} -\Theta^* \right\|_2 + \frac{2\|\XX_{\cS_k}^\T \epsilon\|_2}{t+t_0}\\
&\leq 2^{-k} \left\|\Theta_{t,0} -\Theta^* \right\|_2 + \frac{4\sqrt{s}\cdot \|\XX^\T \epsilon\|_\infty}{t+t_0}\\
&\lesssim \sqrt{\frac{4s\log d}{t+t_0}}
\$
where the last inequality holds if $\|\XX^\T \epsilon\|_\infty\lesssim  \sqrt{(t+t_0)\log d}$, a standard high-dimensional score bound \citep{fan2018lamm}.

However, the condition number bound $\leq 1.43$ is very restrictive, and real-world high-dimensional data often exhibit much larger condition numbers due to strongly correlated features. This motivates the natural question:
\begin{quote}
\it Is it possible to modify the hard thresholding algorithm to allow potentially arbitrary condition numbers while still achieving last-iterate convergence?
\end{quote}

\subsection{An Overparameterized Online Hard Thresholding Algorithm}

Designing algorithms for the online case requires answering the above question first. Our key intuition is that when the linear system has a large condition number,  restricting the projection set in \eqref{eq:t_iter} to exactly size $s^*$ may be too limiting. A small projection set can miss true coefficients due to strong correlations among features. To address this, we {overparameterize} the projection set by using $\cC_{s(\kappa, s^*)}$ instead of $\cC_{s^*}$ in each hard-thresholding step, where $s = s(\kappa, s^*) \geq s^*$ depends on the condition number $\kappa$ and the true sparsity $s^*$. The intuition is that a larger condition number warrants a larger projection set. This approach allows including a few additional features beyond the ground truth, mitigating the risk of missing important coefficients while keeping the estimation error nearly unaffected.

We now introduce our online hard thresholding (OHT) algorithm. First, we formally  define the projection operator. Let $\cC_s$ be the set of $\Phi$ that is at most generalized  $s$-sparse:
\$
\cC_s = \left\{\Phi: \Upsilon(\Phi) \leq s\right\},
\$
where $\cC_s$ is nonconvex and represents the set of at most $s$-sparse vectors for cardinality-constrained linear regression,  or the set of at most rank-$s$ matrices for rank-constrained matrix sensing.  The projection operator $\Pi_{\cC_s}$ is defined as
\$
\Pi_{\cC_s}(\Theta) := \argmin_{\Phi \in \RR^{d_1 \times d_2}} \Big \{ \| \Phi - \Theta \|_\rF^2 \mid \Phi \in \cC_s \Big \},
\$
which returns the best generalized $s$-generalized sparse approximation of $\Theta$.
Concretely, this reduces to keeping the $s$ largest entries of $\Theta$ in magnitude and truncating the rest to zero for cardinality-constrained optimization, and to returning the rank-$s$ approximation for rank-constrained matrix sensing.

\begin{algorithm}[t]
\caption{An online hard thresholding algorithm.}\label{alg:1}
\begin{algorithmic}[1]
\REQUIRE $\Theta_0$, $\{ \eta_{t,j} \}$,   $s$,  $K$.
\ENSURE $\Theta_T$
\STATE \textbf{Initialization Phase:} Generate $t_0$ samples
\FOR{$t= 1, \cdots, T$}
\STATE  Generate the $t$-th sample outside the initialization phase. Set $\Theta_{t,0}  = \Theta_{t-1}$.
	\FOR{$k=0,\cdots, K-1$}
	\STATE Update $\Theta_{t,k+1}  =  \Pi_{\cC_s} \left \{ \Theta_{t,k}
-	\eta_{t,k} \nabla \cL_t(\Theta_{t,k}) \right \}
	$.
 	\ENDFOR
	\STATE Set $\Theta_t = \Theta_{t,K}$.
 \ENDFOR
\end{algorithmic}
\end{algorithm}

Our OHT algorithm is as follows.  Upon receiving each observation $(X_t,Y_t)$ in  round $t\geq 0$, OHT performs $K$ iterations of  overparameterized   hard thresholding with sparsity $s=s(\kappa, s^*)$:
\#\label{eq:OHT}
\Theta_{t,j+1}
= \Pi_{\cC_s} \left \{ \Theta_{t,j} -\eta_{t,j} \nabla \cL_t(\Theta_{t,j}) \right \}, ~~{0\leq j\leq K-1}
\#
where $\Theta_{t, 0}= \Theta_{t-1, K}$ is the last iterate from the previous round, $\Theta_{0,0}$ is some initialization,  and the gradient is
\begin{equation}\label{eq:trace_grad}
\begin{split}
\nabla \cL_t(\Theta) & = \frac{1}{t + t_0} \sum_{i\in \cI_t}  X_i \, \left(\langle X_i,  \Theta\rangle - Y_i\right),
\end{split}
\end{equation}
with $\cI_{t} = \{ i : 0 \leq i \leq t_0\} \cup \{ j : 1 \leq j \leq t\}$ indexing all data points up to round $t$. Algorithm~\ref{alg:1} summarizes the details.

The main technical challenge in analyzing the sequence $\{\Theta_{t,k}: t\geq 0, k\geq 0\}$ comes from the nonconvex constraint set $\cC_s$, which prevents  the direct use of classical nonconvex optimization analysis.  For comparison, if we replace $\cC_s$ with a convex set $\cX$,  the projected gradient descent update
\#
\label{eq:OHT_convex}
\Theta_{t,k+1} = \argmin_{\Theta \in \cX} \left \{  \| \Theta - (\Theta_{t,k} - \eta_{t,k} \nabla \cL_t(\Theta_{t,k}))\|_2^2 \right \} ,  \  0 \leq k \leq K-1
\#
satisfies the first-order optimality condition:
\begin{equation}\label{eq:optimality_euclidean}
    \lv \eta_{t,k} \nabla  \cL_t (\Theta_{t,k}) +  \Theta_{t,k+1} - \Theta_{t,k}, \Theta - \Theta_{t,k+1} \rv \geq 0, \ \forall \  \Theta \in \cX.
\end{equation}
By setting $\Theta = \Theta_{t,k}$, we obtain
\begin{equation*}
\eta_{t,k} \lv \nabla  \cL_t (\Theta_{t,k}) , \Theta_{t,k+1} - \Theta_{t,k} \rv   \leq - \| \Theta_{t,k+1} - \Theta_{t,k} \|_2^2.
\end{equation*}
Using the smoothness  of $\cL_t$, the  function value satisfies
\#\label{eq:descent_convex}
\cL_t(\Theta_{t,k+1}) - \cL_t(\Theta_{t,k})
    & \leq  \lv \nabla  \cL_t (\Theta_{t,k}) , \Theta_{t,k+1} - \Theta_{t,k} \rv + \frac{L}{2} \| \Theta_{t,k+1} - \Theta_{t,k} \|_2^2 \nn\\
    & \leq - \left   (\frac{1}{\eta_{t,k}} - \frac{L}{2} \right ) \| \Theta_{t,k+1}  - \Theta_{t,k} \|_2^2 \leq 0,
\#
for $\eta_{t,k}\leq 2/L$, where $L$ is the sparse strong smoothness parameter.

However, this argument fails for the nonconvex set $\cC_s$, requireing a new understanding of the hard thresholding operator. Specifically, we establish  descent lemmas, i.e., Lemmas \ref{lemma:sdl} and \ref{lemma:descent_lowrank}:
\begin{align*}
\cL_{t}(\Theta_{t,k+1}) -\cL_{t}(\Theta_{t,k})&\leq -\frac{\eta_{t,k} (1-L\eta_{t,k})}{2}\|(\nabla \cL_{t}(\Theta_{t,k}))_{\cS_{t,k}\cup \cS_{t,k+1}} \|_2^2
\leq 0
\end{align*}
for $\eta_{t,k} \leq 1/L$. This new understanding  enables a recursive relationship between outputs from successive online rounds, leading to  the final convergence result \eqref{result:informal}.

To summarize, Algorithm~\ref{alg:1} simultaneously addresses the four main challenges:
\begin{enumerate}
\item	No dynamic regularization: The algorithm does not require updating any regularization parameter dynamically. Once the projection sparsity $s$ is chosen, it remains fixed throughout all rounds.
\item	Storage and memory efficiency: Each update only requires the previous solution \(\Theta_{t-1}\) and the gradient \(\nabla \cL_t(\Theta_{t-1})\). For linear regression, we maintain \(\sum_{i \in \cI_t} X_i X_i^\top\) and \(\sum_{i \in \cI_t} X_i Y_i\); for matrix sensing, we store \(\sum_{i \in \cI_t} \vec(X_i)\vec(X_i)^\top\) and \(\sum_{i \in \cI_t} X_i Y_i\), with memory costs \(\cO(d^2)\) and \(\cO(d_1^2 d_2^2) \), respectively.
\item	Real-time computation: Each round only requires a few projected gradient steps.
\item   Realistic assumptions: We allow arbitrary sparse condition numbers.
\end{enumerate}
Finally, the projection sparsity $s = s(\kappa, s^*)$ can be chosen as small as $s > \kappa^2 s^*$ when $\eta_{t,k} = 1/L$. For orthonormal designs $(\kappa = 1)$, this reduces to $s > s^*$, which is nearly the minimal requirement to achieve good statistical accuracy, since any $s < s^*$ would miss true signals.


\section{Global Convergence Analysis}\label{sec:main}

In this section, we prove that our proposed online hard thresholding algorithm (Algorithm~\ref{alg:1}) achieves global convergence for both online sparse linear regression and low-rank matrix sensing. To begin, we assume that the loss functions \(\cL_t\) satisfy the following generalized sparse strong smoothness (GSSS) and generalized sparse strong convexity (GSSC) conditions with appropriate parameters.

\begin{condition}[Generalized Sparse Strong Smoothness, GSSS]\label{def:smooth}
For each $t\geq 0$, the loss function $\cL_t$ is $(s_1, s_2)$-generalized sparse strongly smooth, or $(s_1, s_2)$-GSSS for short, with parameter $L_{s_1, s_2}$, meaning that
\$
\cL_t(\Theta') -\cL_t(\Theta) \leq \left\langle \nabla \cL_t(\Theta), \Theta' - \Theta\right\rangle + \frac{L_{s_1, s_2}}{2}\left\|\Theta'-\Theta \right \|_\rF^2, \, \text{~for all } \Theta' \in \cC_{s_1} \text{ and } \Theta \in \cC_{s_2}.
\$
\end{condition}

\begin{condition}[Generalized Sparse Strong Convexity, GSSC]\label{def:strong_convexity}
For each $t\geq 0$, the loss function $\cL_t$ is $(s_1, s_2)$-generalized sparse strongly convex with parameter $\alpha_{s_1, s_2}$, meaning that
\$
\cL_t(\Theta') -\cL_t(\Theta) \geq \left\langle \nabla \cL_t(\Theta), \Theta' - \Theta\right\rangle + \frac{\alpha_{s_1, s_2} }{2}\left\| \Theta' - \Theta
\right\|_\rF^2, \, \text{ for all } \Theta' \in \cC_{s_1} \text{ and } \Theta \in \cC_{s_2} .
\$
\end{condition}

These GSSS and GSSC conditions characterize the smoothness and convexity properties required  for a wide range of problems.  In the rest of this section, we will specify the generalized sparsity levels $s_1$ and $s_2$, the convexity parameter $\alpha_{s_1, s_2}$, and the smoothness parameter $L_{s_1, s_2}$ for both online sparse linear regression and online low-rank matrix sensing.

\subsection{Online Sparse Linear Regression}\label{sec:multistep_sparse}

 We first investigate the performance of Algorithm~\ref{alg:1} for online sparse linear regression, for which $d_2 = 1$ and we write $d = d_1$. Accordingly, we have
$X_i \in \RR^{d}$, $\Theta^* \in \RR^d$, $\Upsilon(\Theta) = \| \Theta\|_0 $, and
\begin{equation}\label{def:cardinality_projection}
 \Pi_{\cC_s}(\Theta) = \argmin_{\Phi \in \RR^{d}} \Big \{ \| \Phi  - \Theta\|_2^2 \mid \|\Phi \|_0 \leq s \Big \},
\end{equation}
which the hard thresholding operator that keeps the $s$ largest entries of $\Theta$ in magnitude and truncates the rest to zeros.

For this problem, the GSSS and GSSC conditions reduce to the familiar  sparse strong smoothness (SSS) and sparse strong convexity (SSC) conditions, respectively.

\begin{assumption}\label{assumption:rsc_2s}
For all $t \geq 1$, the loss function $\cL_t$ satisfies $(s,s)$-SSC  and $(s,s)$-SSS conditions with parameters $\alpha$ and $L$, respectively.  We define the sparse condition number as $\kappa = {L}/{\alpha}$.
\end{assumption}

Assumption \ref{assumption:rsc_2s} is standard  in the high-dimensional statistics literature \citep{bickel2009simultaneous, raskutti2010restricted}. By adapting the proofs of \cite[Theorem 1 and Corollary 1]{raskutti2010restricted}, one can show that this assumption holds with high probability under Gaussian covariates, provided that $t_0 \geq cs \log d$ for some constant $c >0$ and the relevant population sparse eigenvalues are bounded. For simplicity, we also assume that the random noises  $\epsilon_i$ and covariates $X_i$ satisfy  the following standard assumption.

\begin{assumption}\label{assumption:noise}
The random noises  $\epsilon_i$  are mean-zero and have bounded second moments
$ \EE[ \epsilon_i^2] \leq \sigma^2$. Moreover, the scaled sum of noise-weighted covariates satisfies
$$
\EE \left[  \Big \| \sum_{i \in \cI_t} \frac{ \epsilon_i X_i }{\sqrt{ t+t_0} }  \Big \|_\infty^2 \right] \leq \sigma_x^2.
$$
\end{assumption}

The following result shows that $\sigma_x^2
= \cO(  \log d)$ when both $X_i$ and $\epsilon_i$ are sub-Gaussian\footnote{A mean-zero random variable $Z$ is said to be subG($\sigma^2$) if the $\psi_2$ Orlicz norm of $Z$ is bounded by $\sigma$, that is, $\|Z\|_{\psi_2}\leq \sigma$, where $\psi_2(x)= e^{x^2/t^2} -1$.}.

\begin{proposition}\label{prop:subW}
Suppose the coordinates of $X_i$ are sub-Gaussian, $X_{ij}\sim \text{subG}(\sigma_1^2)$, and $\EE \epsilon_i^2 \leq \sigma^2$. Then
    $$ \EE \left[  \Big \| \sum_{i \in \cI_t} \frac{ \epsilon_i X_i }{\sqrt{ t+t_0} } \Big \|_\infty^2 \right]  \leq C \sigma^2\sigma_1^2 \log d ,$$
    where $C$ is some universal constant.
\end{proposition}

With these assumptions in place, we now characterize the performance of the best achievable estimator in each online round. Suppose we have received $t$ online data points after the initial batch. Recall that $\hat \Theta_t$ in \eqref{eq:beta_t} is the optimal $s^*$-sparse solution under the loss function $\cL_t({\cdot})$. This estimator serves as the moving target that a globally convergent algorithm aims to track. Upon receiving an additional observation $(X_{t+1}, Y_{t+1})$, the loss function updates to $\cL_{t+1}({\cdot})$, and the target estimator becomes $\hat \Theta_{t+1}$. This gives rise to a moving-target optimization problem.

\begin{lemma}\label{lemma:stats_error}
Suppose Assumptions~\ref{assumption:rsc_2s} and \ref{assumption:noise} hold.
Then, for all $t\geq 1$,
\$
 \EE \left[ \|  \hat \Theta_t -  \Theta^* \|_2^2  \right] \leq   \frac{8s^*\sigma_x^2 }{\alpha^2( t + t_0)}.
\$
\end{lemma}

The above expected mean squared error bound can be interpreted as the minimum statistical error achievable by any estimator under the loss function $\cL_t({\cdot})$, given the data available up to round $t$. Intuitively, in round $t$, no estimator can outperform $\hat \Theta_t$, since only $t+t_0$ observations are available. Having established this benchmark, we now analyze the performance of our OHT algorithm for online sparse linear regression, partially motivated by the analysis of the offline iterative hard thresholding algorithm in \cite{liu2020between}.
For any  $s$-sparse initialization $\Theta_{t,0}$ in round $t$,
\cite{liu2020between} established a linear convergence result under the projection sparsity level {$s >  \kappa^2 s^*$}:
\begin{equation}\label{eq:jain_multiple_step}
\min_{k = 1,\cdots, K}  \cL_{t} (\Theta_{t,k}) - \cL_t ( \Theta^*)\leq   \frac{L}{2} \left (  \frac{1- \kappa }{ 1 - 2\gamma_{s,s^*} }  \right )^K  \| \Theta_{t,0}  - \Theta^* \|_2^2,
\end{equation}
where $\gamma_{s,s^*}$ is a constant defined later,
and the step size is taken as $1/L$. Actually, \cite{liu2020between}'s result \eqref{eq:jain_multiple_step} can be easily extended to $\eta\geq 1/L$ but not $\eta< 1/L$, whereas our new results later hold for any $\eta\leq 1/L$. The requirement $\eta\leq 1/L$ is necessary for our descent lemma, Lemma \ref{lemma:sdl}, to hold.

However, the online and offline settings  differ in ways that make extending this convergence result challenging:
\begin{enumerate}
\item[(i)] \textbf{Last-iterate convergence:}  The result~\eqref{eq:jain_multiple_step} does not guarantee last-iterate convergence. In the online setting, the output of round $t$ serves as the initialization for round $t+1$. Without last-iterate guarantees for $\Theta_{t,K}$, establishing a recursive relationship between consecutive rounds is impossible.

\item[(ii)] \textbf{Moving target:} The offline result assumes a fixed loss, while in the online setting, \(\cL_t\) changes with each incoming observation. This leads to different target estimators \(\hat\Theta_t\) in each round, requiring a new analysis to understand how the moving target affects convergence.

\item[(iii)] \textbf{Metric mismatch:} The left and right sides of \eqref{eq:jain_multiple_step} involve different metrics, which do not naturally form a recursive relationship across multiple online rounds.
\end{enumerate}
These differences necessitate a new understanding of the hard thresholding operator and a novel analysis to establish theoretical convergence guarantees for Algorithm~\ref{alg:1} in the online setting.

To overcome challenge (i), we need to establish a last-iterate convergence result, which in turn requires a deeper understanding of the hard thresholding operator. To this end, we introduce a new descent lemma that guarantees the objective value decreases after each hard thresholding step.

\begin{lemma}[Descent Property]\label{lemma:sdl}
Suppose $\cL_{t}$ is  $(s,s)$-SSS  with smoothness parameter $L$, and  the step sizes are set as $\eta_{t,k}\leq 1/L$. Let $\cS_{t,k}$ and $\cS_{t, k+1}$ denote the supports of $\Theta_{t,k}$ and $\Theta_{t, k+1}$ respectively. Then
\$
\cL_{t}(\Theta_{t,k+1}) -\cL_{t}(\Theta_{t,k})\leq   -\frac{\eta_{t,k} (1-L\eta_{t,k})}{2}\|(\nabla \cL_{t}(\Theta_{t,k}))_{\cS_{t,k}\cup \cS_{t,k+1}} \|_2^2\leq 0.
\$
\end{lemma}

It is worth emphasizing that this  descent result holds for any projection sparsity level $s \geq 1$   and  any loss function $\cL_t(\cdot)$ satisfying the sparse smoothness assumption.  This generalizes  the descent property of projected gradient descent with convex constraints (see \citep[Lemma 3.3]{lan2020first} or \eqref{eq:descent_convex}) to the nonconvex cardinality-constrained setting, showing that the loss $\cL_t (\Theta_{t,k})$ is non-increasing across consecutive hard thresholding steps. Consequently, the last iterate $\Theta_{t,K}$ achieves the minimum loss among $\{\Theta_{t,k}: 0\leq k \leq K \}$.

Leveraging this descent property, we can further quantify the loss incurred by the last-iterate solution $\Theta_{t,K}$:
\begin{equation}
\cL_t (\Theta_{t,K}) - \cL_t ( \Theta^*  ) \leq \frac{1- 2\gamma_{s, s^* } }{2\eta}  \left ( \frac{1-\eta \alpha }{ 1- 2 \gamma_{s, s^* }}\right )^{K } \| \Theta_{t,0} - \Theta^*    \|_2^2.
\end{equation}
Here, the loss $\cL_t$ can be directly related to the MSE $\| \Theta_{t,K} - \Theta^*\|_2^2$ at the beginning of round $t$. However, as highlightened in Challenges (ii) and (iii), issues remain: (1) the loss function $\cL_t$ changes dynamically as new observations arrive, yielding a moving optimization target, and (ii) the metrics on both sides of the inequality above are not naturally recursive, which prevents straightforward iteration across multiple online rounds.

To overcome these issues, we employ the $(s, s)$-SSC property to derive a lower bound for the loss difference $\cL_t(\Theta_{t,K}) - \cL_t(\Theta^*)$:
\begin{equation*}
\begin{split}
 \cL_t ( \Theta_{t,K} )  - \cL_t ( \Theta^*  )  & \geq \frac{\alpha}{2} \|  \Theta_{t,K}-   \Theta^* \|_2^2    + \lv \nabla \cL_t ( \Theta^*   ),  \Theta_{t,K} -  \Theta^* \rv
\\
&= \frac{\alpha}{2} \|  \Theta_{t,K}-   \Theta^* \|_2^2     -  \frac{1}{t_0 + t}    \sum_{i\in \cI_t} \lv   X_i \epsilon_i  ,   \Theta_{t,K} -  \Theta^*   \rv.
\end{split}
\end{equation*}
Using the Cauchy-Schwarz inequality and the fact that both $\Theta_{t,K}$ and $\Theta^*$ are $s$-sparse, we have
$$
\frac{1}{t_0 + t}    \sum_{i\in \cI_t} \lv   X_i \epsilon_i  ,   \Theta_{t,K} -  \Theta^*   \rv \leq  \frac{2s}{\alpha} \left \| \frac{\sum_{i\in \cI_t}  X_i \epsilon_i }{t_0 + t}   \right \|_\infty^2    + \frac{\alpha}{4}\| \Theta_{t,K} -  \Theta^* \|_2^2.
$$
Consequently, the loss difference can be bounded from below in terms of the mean squared error and a statistical error term:
\begin{equation}\label{eq:multiple_3}
\begin{split}
 \cL_t ( \Theta_{t,K}   )  - \cL_t ( \Theta^*  )  \geq \frac{\alpha}{4}  \|  \Theta_{t,K} - \Theta^*  \|_2^2 -  \frac{2s}{\alpha}  \left \| \frac{\sum_{i\in \cI_t}  X_i \epsilon_i }{t_0 + t}   \right \|_\infty^2.
\end{split}
\end{equation}

Using the above results, we are now ready to establish a recursive relationship between the solutions obtained in successive online rounds. We summarize our result in the following lemma. We first formally define
$$
\gamma_{s,s^*} := \sup \left \{  \frac{\langle Y - \Pi_{\cC_s}(Z), Z - \Pi_{\cC_s}(Z)\rangle }{ \| Y - \Pi_{\cC_s}(Z)\|_2^2}: Y, Z \in \RR^d, \| Y \|_0 \leq s^*, Y \neq \Pi_{\cC_s} (Z)\right \}
$$
as the relative concavity of $s$-sparse hard-thresholding operator $\Pi_{\cC_s}$ with respect to the ground truth sparsity level $s^*$.

\begin{lemma}\label{lemma:multi_step_descent}
Suppose Assumption~\ref{assumption:rsc_2s} holds. We have $\gamma_{s,s^*} = \sqrt{ s^*/s}/2$. Let $\eta_{t,k} = \eta \leq {1}/{L}$ and choose $K$ large enough such that
\$
K \geq K_0 := \frac{\log \{2(1- 2\gamma_{s,s^*})/({ \eta \alpha})\}}{ \log \{(1 - 2\gamma_{s,s^*})/({ 1 - \eta \alpha})\}},
\$
then
\begin{equation}\label{def:delta}
 \delta := \frac{2(1 - 2\gamma_{s,s^*} )}{\eta \alpha}  \left ( \frac{1 - \eta \alpha}{ 1 - 2\gamma_{s,s^*} }\right )^K  < 1.
 \end{equation}
Consequently,
\begin{equation}\label{eq:multi_step_descent}
\| \Theta_{t+1,0}  - \Theta^* \|_2^2  = \| \Theta_{t,K}  - \Theta^* \|_2^2  \leq \underbrace{ \delta\,  \| \Theta_{t,0} - \Theta^*  \|_2^2  }_{\Rom{1}} + \underbrace{ \frac{8s}{\alpha^2} \left \| \frac{ \sum_{i \in \cI_t} \epsilon_i X_i  }{ t + t_0} \right \|_\infty^2 }_{\Rom{2}}.
\end{equation}
\end{lemma}

The result above bounds the mean squared error $\|\Theta_{t+1,0} - \Theta^*\|_2^2$ by two terms. The first term (\Rom{1}) contracts the error from the previous round $\|\Theta_{t,0} - \Theta^* \|_2^2$ by a factor of $\delta$ through $K$ iterations of hard-thresholded gradient descent. The second term (\Rom{2}) represents the optimal statistical error with $t+t_0$ samples in round $t$. Therefore, when $\eta$, $K$, and $s$ are chosen appropriately, the successive solutions form a contraction, ensuring that the sequence of iterates ${\Theta_{t,K} : t \geq 0}$ converges to the true sparse coefficient $\Theta^*$.
By applying this result recursively, we establish the convergence rate of ${\Theta_{t,K}}$ for any round $t \geq 1$.

\begin{theorem}\label{thm:multistep_sparse}
Suppose Assumptions~\ref{assumption:rsc_2s} and \ref{assumption:noise} hold.
Let $\{ \Theta_{t,k}: t \geq 1, 1 \leq k \leq K \}$ denote the solution sequence generated by Algorithm~\ref{alg:1} with step size $\eta_{t,k} = \eta  \leq 1/L$, $K\geq K_0$, and projection sparsity $s > {s^*}/({\eta^2 \alpha^2})$. Then
\begin{equation*}
\begin{split}
\EE[  \| \Theta_{t+1,0}- \Theta^* \|_2^2 ]
\leq  \delta^{t+1} \EE \left[   \| \Theta_{0} - \Theta^*  \|_2^2 \right]
    + \frac{8s \sigma_x^2 }{\alpha^2 \log(1/\delta)}
        \left (  \frac{  \delta^{t-1}}{t_0 + 1}  +   \frac{3}{\delta (t+1 + t_0) } +  \frac{\delta^{(t+1)/2-1}  }{t_0 + 1 }  \right )
\end{split}
\end{equation*}
where $\delta <1$ is defined in \eqref{def:delta}.  Consequently, when
\$
t\gtrsim \frac{\log\left\{\EE\|\Theta_0-\Theta^*\|_2^2/ (s^* \sigma_x^2)\right\}}{\log(1/\delta)} + \frac{\log(t+t_0)}{\log(1/\delta)},
\$
we have
$
\EE[  \| \Theta_{t+1,0}- \Theta^* \|_2^2 ] \lesssim {s\sigma_x^2}/(t+t_0).
$
\end{theorem}

The above theorem indicates that, for projection sparsity satisfying $s^*/(\eta^2 \alpha^2)< s \lesssim s^*/(\eta^2\alpha^2)$, Algorithm~\ref{alg:1} achieves a convergence rate of $ \cO( s^*\sigma_x^2/ (t+t_0) )$, matching the optimal offline statistical error given in  Lemma~\ref{lemma:stats_error}. Furthermore, if we select the step size $\eta_{t,k} = \eta = {1}/{L}$,  the projection sparsity requirement reduces to $s>\kappa^2 s^*$. When the sparse condition number $\kappa$ is small, this requirement is minimal. For example, in isotropic designs where $\kappa = 1$, the OHT algorithm converges with $s > s^*$, whereas under-projection ($s < s^*$)  will miss some of the true signals. We note, however, that the dependency of $s$ on $\kappa$ may be sub-optimal and could potentially be improved.

\subsection{Online Low-Rank Matrix Sensing}\label{sec:multistep_lowrank}

In this section, we analyze the OHT algorithm for the online low-rank matrix sensing problem. Here, the rank function serves as the sparsity measure, i.e., $\Upsilon(\Theta) = \rank(\Theta)$, and the corresponding projection operator is
\begin{equation}\label{def:lowrank_projection}
\Pi_{\cC_s}(\Theta) = \argmin_{X \in \RR^{d_1\times d_2}}\Big \{ \| X - \Theta \|_\rF^2
\mid \rank(X) \leq s \Big \}.
\end{equation}
Let $\Theta = U \Sigma V^\top$ denote the singular value decomposition (SVD) of $\Theta \in \RR^{d_1 \times d_2}$, where the singular values in $\Sigma$ are arranged in decreasing order, and let $U_s$ be the first $s$ columns of $U$. Then, the projection operator can be equivalently written as $\Pi_{\cC_s}(\Theta) = U_s U_s^\top \Theta$.

In this context, we refer to the generalized sparse strong smoothness and generalized sparse strong convexity assumptions as low-rank strong smoothness (LowRankSS) and low-rank strong convexity (LowRankSC) assumptions, respectively. As in the online sparse linear regression setting, the following assumptions are required.


\begin{assumption}\label{assumption:rsc_lowrank_2s}
 For all $t \geq 0$, the loss function $\cL_t$ is $(s,s)$-LowRankSC  and $(s,s)$-LowRankSS assumptions with parameter $\alpha$ and $L$, respectively. The corresponding $(s,s)$-low-rank condition number is defined as $\kappa = L/\alpha$.
\end{assumption}

A sufficient condition for LowRankSC and LowRankSS is the RIP condition \citep{maros2023decentralized}. In particular, Theorem 2.3 of \cite{candes2010tight} shows that RIP holds with high probability for random measurement ensembles satisfying an appropriate concentration inequality. Properly normalized i.i.d.~Gaussian, Bernoulli, and sub-Gaussian measurement ensembles are examples \citep{vershynin2000large}.


\begin{assumption}\label{assumption:noise_lowrank}
The noise terms  $\epsilon_j$ are mean-zero with bounded second moments, i.e., $\EE[\epsilon_j ] = 0$ and $\EE[ \epsilon_j^2] \leq \sigma^2$. Moreover, there exists a constant $\sigma_X>0$ such that
$$
\EE \left[
        \Big \| \sum_{i \in \cI_t}      \frac{ \epsilon_i X_i }{\sqrt{t_0 + t} } \Big \|_{2}^2
    \right]
\leq \sigma_X^2.
$$
\end{assumption}


With these assumptions, we first characterize the performance of $\hat \Theta_t$, the target estimator for the OHT algorithm in round $t \geq 0$.

\begin{lemma}\label{lemma:stats_error_lowrank}
Suppose Assumptions~\ref{assumption:rsc_lowrank_2s} and \ref{assumption:noise_lowrank} hold.
For all $t \geq 1$,
\begin{equation*}
\begin{split}
 \EE[ \|  \hat \Theta_t -  \Theta^* \|_\rF^2  ] \leq   \frac{8s^*\sigma_X^2 }{\alpha^2( t_0 + t)}.
\end{split}
\end{equation*}
\end{lemma}

We point out that  \cite[Corollary~1]{negahban2011estimation} established a similar high-probability bound. Our result here is in expectation, which is needed for our subsequent analysis.

\begin{lemma}\label{lemma:descent_lowrank}
Suppose $\cL_{t}$ is $(s,s)$-LowRankSS with smoothness parameter $L$, and let  $\eta_{t,k}   \leq 1/L$. Let $S_{t,k}$ and $S_{t,k+1}$ denote the column spaces  of $\Theta_{t,k}$ and $\Theta_{t,k+1}$, respectively. Then
\$
\cL_{t}(\Theta_{t,k+1}) -\cL_{t}(\Theta_{t,k})\leq   -\frac{\eta_{t,k} (1-L \eta_{t,k})}{2}\|P_{S_{t,k} + S_{t,k+1}} \nabla \cL_{t}(\Theta_{t,k})) \|_{\rF}^2.
\$
\end{lemma}

This lemma implies that the objective value is non-increasing between consecutive hard-thresholding gradient steps within each round, for any projection rank $s\geq 1$ and any low-rank strong smooth loss function. It generalizes the descent property of projected gradient descent with a convex constraint \cite[Lemma 3.3]{lan2020first} to accommodate rank-cardinality constraints.

The next two results establish the recursive relation between successive rounds and the overall convergence guarantee. Recall the definitions of $K_0$ and $\delta$ in Lemma~\ref{lemma:multi_step_descent}, where $\alpha$ denotes the low-rank strong convexity parameter and $\gamma_{s,s^*}$ is the relative concavity parameter
\$
\gamma_{s,s^*} = \sup \left \{  \frac{\langle Y - \Pi_{\cC_s}(Z), Z - \Pi_{\cC_s}(Z)\rangle }{ \| Y - \Pi_{\cC_s}(Z)\|^2}: Y, Z \in \RR^{d_1 \times d_2}, \rank(Y) \leq s^*, Y \neq \Pi_{\cC_s} (Z)\right \}.
\$
It can be shown that, for the low-rank hard-thresholding operator, $\gamma_{s,s^*} = \sqrt{s^*/s}/2$; see Lemma~\ref{lemma:relative_concavity_lowrank} in the appendix.


\begin{lemma}\label{lemma:multi_step_descent_lowrank}
Suppose Assumption~\ref{assumption:rsc_lowrank_2s} holds.  Take $\eta_{t,k} = \eta \leq {1}/{L}$. If $s > {s^*}/({\eta^2 \alpha^2})$  and {$K \geq K_0 = \log \{2(1- 2\gamma_{s,s^*})/({ \eta \alpha})\}/ \log \{(1 - 2\gamma_{s,s^*})/({ 1 - \eta \alpha})\}$}, then
 \begin{equation*}
\| \Theta_{t+1,0}  - \Theta^* \|_{\rF}^2  \leq \delta \, \| \Theta_{t,0} - \Theta^*  \|_{\rF}^2  + \frac{8s}{\alpha^2} \left \| \frac{ \sum_{i\in \cI_t}   \epsilon_i X_i      }{ t_0 + t} \right \|_{2}^2,
\end{equation*}
where $\delta < 1$ is as in Lemma \ref{lemma:multi_step_descent}.
\end{lemma}

\begin{theorem}\label{thm:multistep_lowrank}
Suppose Assumptions \ref{assumption:rsc_lowrank_2s} and  \ref{assumption:noise_lowrank} hold.  Let $\{ \Theta_{t,k}\}$ be  generated by Algorithm~\ref{alg:1} with $ \eta_{t,k} = \eta  \leq 1/L$, $ K \geq K_0 $, and projection rank $s > {s^*}/({\eta^2 \alpha^2})$. Then
\begin{equation*}
\begin{split}
\EE \left[  \| \Theta_{t+1,0}- \Theta^* \|_{\rF}^2 \right]
\leq  \delta^{t+1}  \, \| \Theta_{0} - \Theta^*  \|_{\rF}^2  + \frac{8s \sigma_X^2 }{\alpha^2 \log (1/\delta)} \left (  \frac{ \delta^{t}}{t_0 + 1}   +   \frac{3}{\delta (t+2 + t_0) } +  \frac{\delta^{t/2}  }{t_0 +1 }  \right ),
\end{split}
\end{equation*}
for some $\delta<1$. Consequently, when
\$
t\gtrsim \frac{\log\left\{\EE\|\Theta_0-\Theta^*\|_\rF^2/ (s^* \sigma_x^2)\right\}}{\log(1/\delta)} + \frac{\log(t+t_0)}{\log(1/\delta)},
\$
we have
\$
\EE[  \| \Theta_{t+1,0}- \Theta^* \|_\rF^2 ] \lesssim \frac{s\sigma_X^2}{\alpha^2(t+t_0)}.
\$
\end{theorem}

This theorem indicates that Algorithm~\ref{alg:1} produces a sequence that converges to the ground truth coefficient matrix at a rate of $\cO(s^* \sigma_X^2 / t)$ for online rank-constrained matrix sensing, under projection rank levels $s^*/(\eta^2 \alpha^2) < s \lesssim s^*$. These results parallel those obtained for online sparse linear regression, with $\sigma_x^2$ replaced by $\sigma_X^2$.

\section{No Batch Initialization}\label{sec:extensions}

In previous sections, we introduced our online hard thresholding algorithm, Algorithm \ref{alg:1},  for efficient generalized sparse online regression. The algorithm relies on the key assumption that the generalized sparse strong convexity and smoothness assumptions hold, which is only possible by starting with an initial batch of observations. This raises a natural question:
\begin{quote}
\textit{What if such an initialization batch is not available?}
\end{quote}

This section addresses this question by proposing a \emph{single-phase} algorithm that computes a solution upon receiving each observation, without a batch initialization.
Specifically, this corresponds to $t_0=0$ in the previous section, so that
\$
\cL_t(\Theta) =    \frac{1}{t} \sum_{j=1}^t\ell_j(\Theta ).
\$
We assume that the GSSS and GSSC assumptions hold after receiving $t_w$ observations, although the exact value of $t_w$ may not be known in practice. For simplicity, we shall take $t_w=t_0$.  With a slight abuse of notation, we denote the data pairs received up to round $t$ as $(X_1,Y_1),\cdots, (X_t,Y_t)$  and the solution computed in round $t$ as  $\Theta_t$. The proposed algorithm proceeds exactly as  Algorithm~\ref{alg:1}, but without an initial batch.

\subsection{Online Sparse Linear Regression}

When analyzing the performance of Algorithm~\ref{alg:1} without batch initialization, we assume that the $(2s,s)$-SSS  assumption\footnote{This is slightly stronger than the $(s,s)$-SSS assumption in the setting with an initial batch; it can be relaxed to $(s,s)$-SSS assumption with appropriately defined higher-moment bounds on the covariates $X_i$.} with parameter $L$ holds for all $t \geq 0$.
A key challenge in this setting is that the SSC assumption may not hold during early rounds, potentially leading to poor early solutions. Specifically, without the SSC assumption, errors from the hard-thresholding steps for $t \leq t_0$ may propagate and amplify across rounds.

\begin{proposition}
\label{prop:free_initial_sparse}
Suppose Assumption~\ref{assumption:noise} holds, and $\cL_t(\cdot) $ is $(2s,s)$-SSS with parameter $L$ for $t\geq 1$. Considering running Algorithm~\ref{alg:1} for $t_0$ rounds  without an initial  batch, with step size $\eta_t  = \eta \leq \frac{1}{L}$ {and arbitrary $s\geq 1,K\geq 1$.} Then, there exists a constant $C_1$ independent of $t$ such that, for all $1\leq t\leq t_0$,
\begin{equation*}
 \EE[\cL_{t}(\Theta_{t,K})  ]
\leq  C_1  4^{tK}  + \sigma^2\log(t+1),
\end{equation*}
where $C_1= C_0(\|\Theta_{0,0} -\Theta^*\|_2 + ( 4s \eta^2 \sigma_x^2  + \|\Theta^*\|_2^2 )/3   )$  and $C_0= (s+s^*) \EE\big[\|X_0\|_\infty^2\big]$.
\end{proposition}

Once $t_0$ observations are collected, the SSC assumption begins to hold.
Then, $\| \Theta_{t_0,K} - \Theta^*\|^2 $ can be treated as the initial gap. By applying Eq.~\eqref{eq:multiple_3}, we see that  the expected initial gap is bounded:
\begin{equation*}
\begin{split}
\frac{\alpha}{4} \EE[ \|  \Theta_{t_0,K} - \Theta^*  \|_2^2 ] & \leq  \EE [\cL_{t_0} ( \Theta_{t_0,K}   )  - \cL_{t_0} ( \Theta^*  ) ] +   \frac{2s}{\alpha} \EE \left [  \Big \| \frac{\sum_{i\in \cI_{t_0}}  X_i \epsilon_i }{t_0 }  \Big \|_\infty^2  \right ]
\\
& \leq C_1  4^{t_0 K}  + \sigma^2\log(t_0+1) + \cO \Big (\frac{ s \sigma_x^2 }{t_0} \Big ) ,
\end{split}
\end{equation*}
and this gap is geometrically discounted in each subsequent round $t > t_0$ by Theorem \ref{thm:multistep_sparse}.

\begin{corollary}\label{corollary:free_initial}
Suppose $\cL_t$ is $(s,s)$-SSC  with parameter $\alpha$ for all $t> t_0$,  $(2s, s)$-SSS with parameter $L$ for all $t\geq 1$,  and
    Assumption~\ref{assumption:noise} holds. Let $\{ \Theta_{t,k}, t\geq 1, 1\leq k \leq K,\}$ be the sequence generated by Algorithm~\ref{alg:1} with $\eta_{t,k} =\eta \leq 1/L$, $s > {s^*}/({\eta^2 \alpha^2})$, and
$$
K \geq K_0 := \frac{\log \{2(1- 2\gamma_{s,s^*})/({ \eta \alpha})\}}{ \log \{(1 - 2\gamma_{s,s^*})/({ 1 - \eta \alpha})\}}, ~~~\gamma_{s,s^*} = \sqrt{s^*/s}/2,
$$
 but without an initial batch. Then for all $t\geq t_0$,
\begin{equation*}
\EE[ \| \Theta_{t,K} - \Theta^* \|_2^2]
\leq \delta^{t-t_0} \EE[ \|  \Theta_{t_0,K} - \Theta^*  \|_2^2 ]   + \cO \Big (\frac{s \sigma_x^2 }{t} \Big )  \leq \frac{4C_1  \delta^{t-t_0}  4^{t_0 K}  }{\alpha}    + \cO \Big (\frac{s \sigma_x^2 }{t} \Big ),
\end{equation*}
where $\delta < 1$ is defined in \eqref{def:delta}.
\end{corollary}

The above result indicates that Algorithm~\ref{alg:1} without batch initialization still achieves the $\cO(s^* \sigma_x^2 / t)$ convergence rate once $t$ is only logarithmically large.

\subsection{Online Low-Rank Matrix Sensing}\label{sec:nodata_lowrank}

 We further study the performance of Algorithm~\ref{alg:1} for rank-constrained matrix sensing without batch initialization.

\begin{proposition}\label{prop:free_initial_lowrank}

Suppose Assumption~\ref{assumption:noise_lowrank} holds and $\cL_t(\cdot) $ is $(2s,s)$-LowRankSS with parameter $L$ for all $t\geq 1$.
Suppose we run Algorithm~\ref{alg:1} for $t_0$ rounds  without an initial  batch, using  step size $\eta_t  = \eta \in [ 0, {1}/{L}]$ {and arbitrary $s\geq 1,K\geq 1$}.
Then, there exists a constant
\$
C_1= 3C_0(\|\Theta_0 -\Theta^*\|_\rF^2 + ( 4s \eta^2 \sigma_X^2  + \|\Theta^*\|_\rF^2 )/3 )
\$
independent of $t$ such that for all $1\leq t\leq t_0$
\begin{equation*}
 \EE[\cL_{t}(\Theta_{t,K})  ]
\leq    C_1  4^{tK}  + \sigma^2\log(t+1),
\end{equation*}
where $C_0= (s+s^*) \EE\big[\|X_0\|_2^2\big]$.

\end{proposition}

\begin{corollary}\label{corollary:free_initial_lowrank}
Suppose Assumption~\ref{assumption:noise_lowrank} holds,  $\cL_t(\cdot)$ is $(2s,s)$-LowRankSS with parameter $L$ for all $t\geq 1$, and $ \cL_t(\cdot) $ is  $(s,s)$-LowRankSC with parameter $\alpha$ for all $t\geq t_0$. Let $\{ \Theta_{t,k}, t \geq 1, 1\leq k \leq K\}$ be the sequence generated by Algorithm~\ref{alg:1} with $s > {s^*}/({\eta^2 \alpha^2})$,  $\eta_t  = \eta \leq 1/L$, and
\$
K \geq K_0 := \frac{\log \{2(1- 2\gamma_{s,s^*})/({ \eta \alpha})\}}{ \log \{(1 - 2\gamma_{s,s^*})/({ 1 - \eta \alpha})\} },~~~\gamma_{s,s^*} = \sqrt{s^*/s}/2,
\$
but without an initial batch.
Then, for all $t\geq t_0$,
\begin{equation*}
\begin{split}
\EE[ \| \Theta_{t,K} - \Theta^* \|_\rF^2]
&\leq \delta^{t-t_0} \EE[ \|  \Theta_{t_0,K} - \Theta^*  \|_\rF^2 ]   + \cO \Big (\frac{s \sigma_X^2 }{t} \Big )  \leq \frac{4C_1  \delta^{t-t_0}  4^{t_0 K}  }{\alpha}    + \cO \Big (\frac{s \sigma_X^2 }{t} \Big ),
    \end{split}
\end{equation*}
where $\delta < 1$ is defined in \eqref{def:delta}.
\end{corollary}

We can also see that Algorithm~\ref{alg:1} exhibits the $\cO(s^* \sigma_X^2/t)$ rate of convergence for rank-constrained matrix sensing without an initial batch.


\section{Numerical Experiments}\label{sec:numerics}

This section presents numerical experiments to validate the practical performance of our proposed algorithms.

\subsection{Online Sparse Linear Regression}

For the linear regression problem, let $d = d_1 =1000$ and $d_2 = 1$. We consider a \emph{weak} signal setting, where the sparse coefficient vector $\Theta^* \in \RR^d$ has its first 15 entries drawn from a normal distribution  $\Theta_j^* \sim \cN(0,0.2)$ for $j=1,\cdots, 15$,  and rest entries set to zero. Hence, the true sparsity is $s^* = \|\Theta^* \|_0 = 15$.

We evaluate the performance of our {\tt OHT} algorithm, Algorithm~\ref{alg:1}, using different numbers of hard thresholding steps $K = 1,5,10,20$ per round. As baselines, we implement {\tt OS\_LASSO\_K=1} and {\tt OS\_LASSO\_K=20} \citep{fan2018statistical}, which first determine the support via LASSO on an initial batch of size $t_0$, and then perform $K$ gradient descent steps over the pre-determined support. All algorithms share the same initial batch and online data.

We consider two covariate correlation structures.
\begin{itemize}
\item[(i)] Independent: each entry $[X_i]_j \sim \cN(0,1)$ independently for $1\leq j \leq d$;
\item[(ii)] Toeplitz with $\rho = 0.5$: $X_i \sim \cN(0,\Sigma)$, where $\Sigma_{jk} = \rho^{|j-k|}$ for $1\leq j, k\leq d$.
\end{itemize}
For each covariate design, the response is generated as $Y_i = X_i^\top \Theta^* + \epsilon_i$ with independent noise $\epsilon_i \sim \cN(0,1)$. Each experiment first generates an initial batch of size $t_0$, followed by $10^4$ online rounds, where each round observes a single data point $(X_t, Y_t)$. For {\tt OHT\_K=1}, {\tt OHT\_K=5}, {\tt OHT\_K=10}, and {\tt OHT\_K=20}, we update $\Theta_t$ via $K$ hard-thresholding gradient descent steps with projection sparsity $s$.

\begin{figure}[t]
    \centering
        \includegraphics[width=0.45\linewidth]{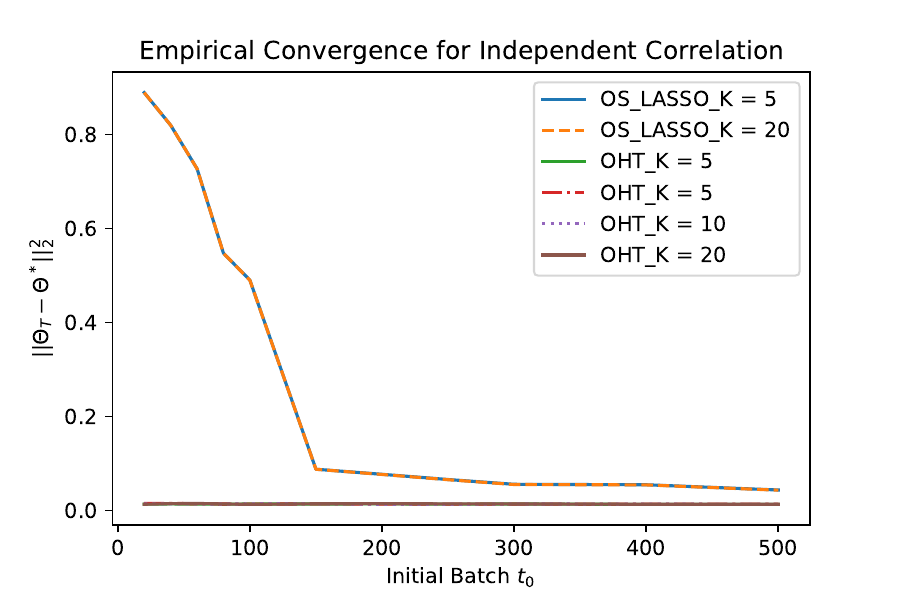}
  \includegraphics[width=0.45\linewidth]{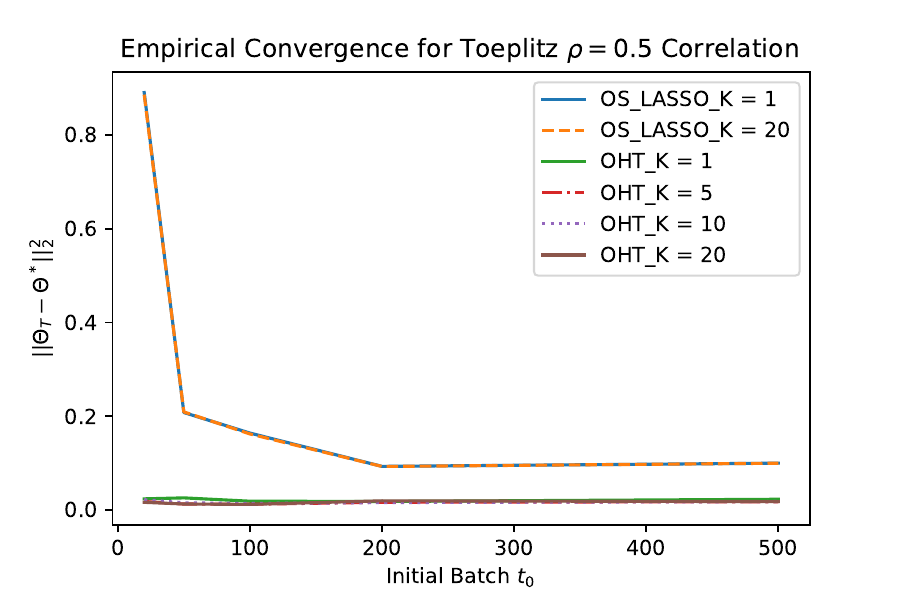}
        \caption{Online sparse linear regression under weak signal setup for $t_0 \in [20,500]$ and $t = 10000$ online learning rounds for  independent and Toeplitz $\rho = 0.5$ covariate designs. }
    \label{fig:weak_t0}
\end{figure}

\begin{figure}[t]
    \centering
        \includegraphics[width=0.48\linewidth]{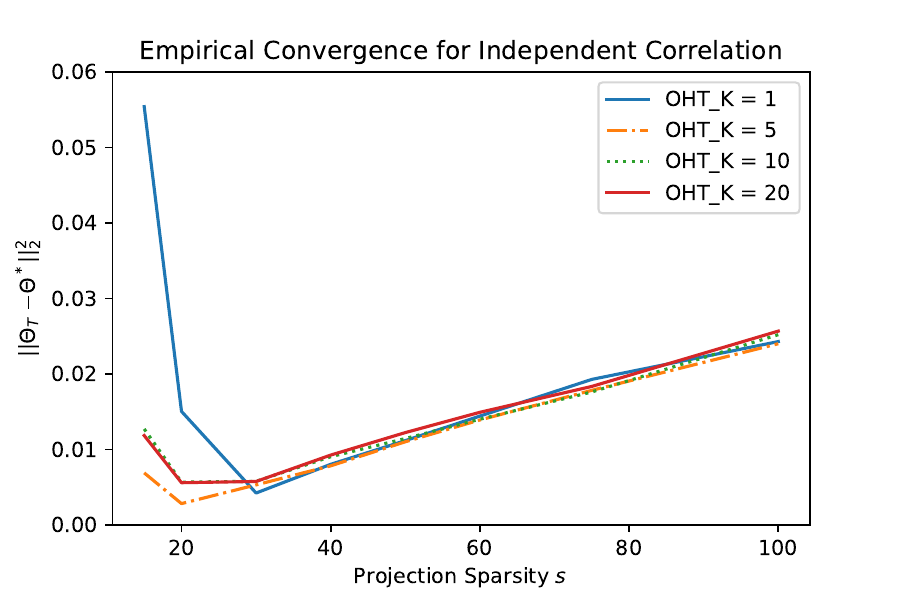}
  \includegraphics[width=0.48\linewidth]{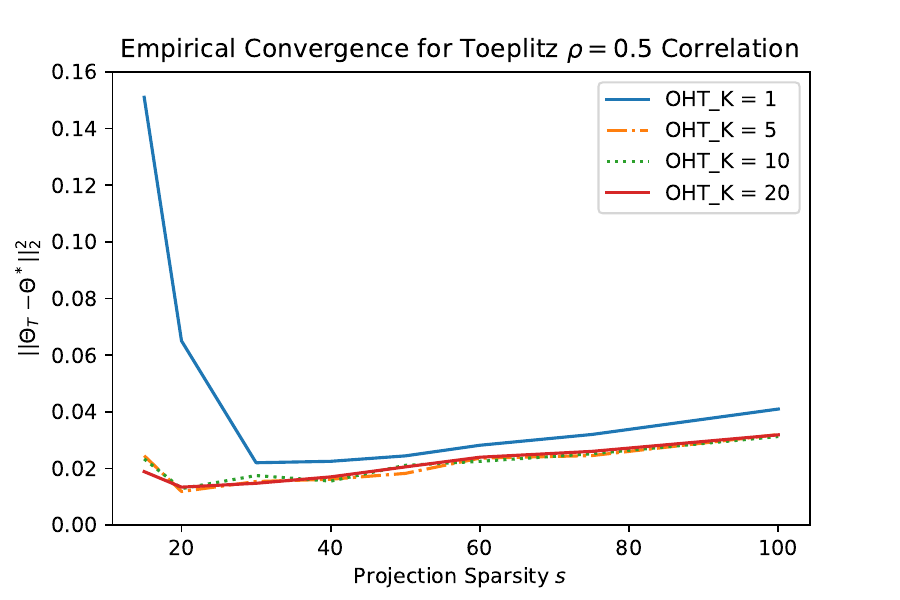}
    \caption{Online sparse linear regression under weak signal setup for projection sparsity level $s\in [15,100]$, $t_0 = 100$, and $T = 10^4$ online rounds, under  independent and Toeplitz $\rho = 0.5$ covariate correlation designs.}
    \label{fig:weak_MSE_ss}
\end{figure}

\begin{figure}[t]
    \centering
        \includegraphics[width=0.45\linewidth]{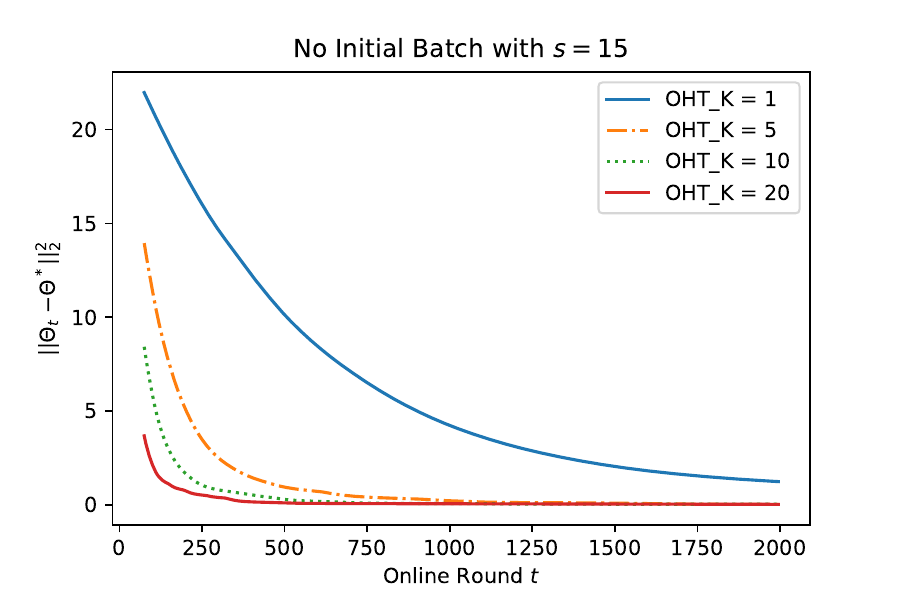}
  \includegraphics[width=0.45\linewidth]{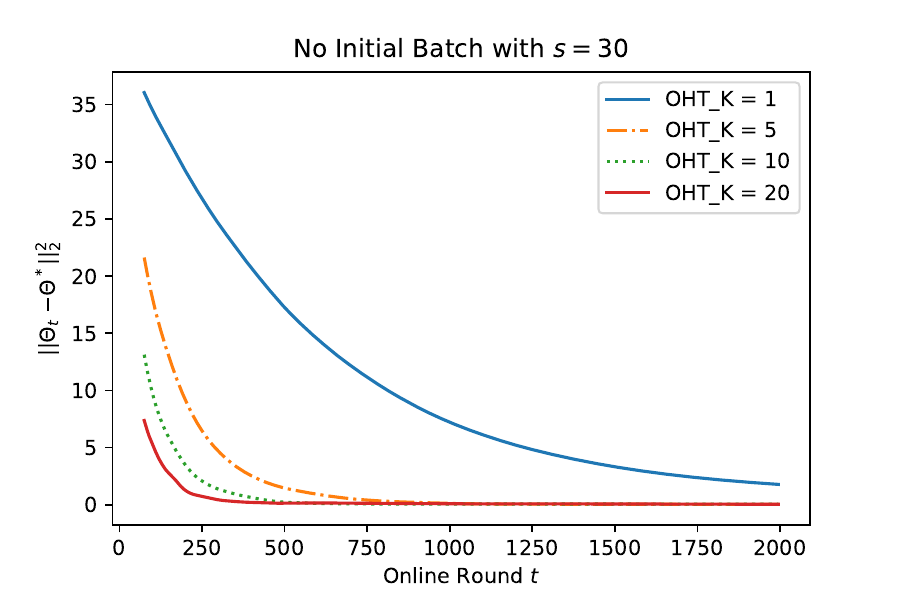}
      \caption{Online sparse linear regression under weak signal setup without the initial batch under projection sparsity level $s \in \{15,30\}$. }
    \label{fig:weak_MSE_free_initial}
\end{figure}

For the baseline algorithms {\tt OS\_LASSO\_K=1} and {\tt OS\_LASSO\_K=20},  we first estimate the support by running LASSO over the initial batch, where the regularization parameter is selected using cross-validation. Subsequently, in each online round $t$, we update $\Theta_t$ by conducting $K$  truncated gradient descent steps that keep only the components in the estimated support and truncate the rest  to zero. From now on, we shall write $\Theta_{t}$ as the last iterate solution $\Theta_{t,K}$ for simplicity. We consider the following three sets of experiments.
\begin{enumerate}
\item We test the algorithms over different initial batch sizes $t_0 = 20,40,\cdots, 500$ and $T = 10^4$ online  rounds under both covariate designs. We report the MSE of the output solution, $\| \Theta_{t} - \Theta^*\|_2^2$, against $t_0$ in Figure~\ref{fig:weak_t0}. The projection sparsity is set to $s = 50$ and the step size to $\eta = 0.001$.

\item We test our algorithms for $s = 15,20,30,\cdots, 100$ under both covariate correlation designs, with $\eta = 0.001$, $t_0 = 100$, and $T = 10^4$ online rounds. Figure \ref{fig:weak_MSE_ss} summarizes the results.

\item We evaluate the algorithm for $s \in \{ 15,30\}$ under the independent covariate  design, with  $\eta = 0.001$, $t_0 = 100$, and $T = 10^4$ online learning rounds. The MSE $\| \Theta_{t} - \Theta^*\|_2^2$  is plotted against $t$ in Figure~\ref{fig:weak_MSE_free_initial}.

\item Finally, we test our OHT algorithms on highly correlated covariates with a large sparse condition number $\kappa$. We generate an orthonormal basis $U \in \RR^{p \times p}$ and a diagonal positive-definite matrix $\Lambda$, then generate $X_i \sim \cN(0, \Sigma)$ with $\Sigma = U \Lambda U^\top$. Specifically, the first 12 diagonal entries of $\Lambda$ are $500, 100, 10, 9, 8, \dots, 1$ and the remaining entries are 1, giving $\kappa = 500$. The signal $\Theta^*$ is generated as before, with an initial batch size $t_0 = 100$ and $T = 10^4$ online rounds. We test sparsity levels $s = 14, 15, 20, 25, 30, 40$, and results are reported in Table~\ref{table:sparse}.
\end{enumerate}

From Figure~\ref{fig:weak_t0}, we observe that the performance of  {\tt OS\_LASSO\_K=1} and {\tt OS\_LASSO\_K=20} is sensitive to the initial batch size $t_0$ under both covariate designs. In particular, both algorithms converge to solutions far away from the true coefficients $\Theta^*$ when the initial batch size is small $(t_0 \leq 150)$ and their convergence performance improves with larger initial batches. By contrast, our OHT algorithms are robust to the initial batch size,  achieving comparable performance for all $t_0 \in [20,500]$.

From Figure~\ref{fig:weak_MSE_ss}, when $s \leq 30$, {\tt OHT\_K=5}, {\tt OHT\_K=10}, and {\tt OHT\_K=20} outperform {\tt OHT\_K=1}. When $s > 30$, all OHT variants perform comparably under the independent design, while {\tt OHT\_K=5,10,20} slightly outperform {\tt OHT\_K=1} under the Toeplitz design. These observations are consistent with Theorem~\ref{thm:multistep_sparse}, which requires sufficiently large $s$ and $K$ to ensure global convergence. Furthermore, the performance of {\tt OHT\_K=5,10,20} is comparable, suggesting that $K=5$ already exceeds the intrinsic $K_0$, and increasing $K$ further does not improve accuracy further.

From Figure~\ref{fig:weak_MSE_free_initial}, under both sparsity levels $s = 15, 30$, our OHT algorithms with $K = 5, 10, 20$ generate sequences ${\Theta_t}$ that converge to the true sparse coefficient $\Theta^*$ even without an initial batch, while {\tt OHT\_=1} converges much more slowly. At $t = 10^4$, a performance gap between {\tt OHT\_=1} and {\tt OHT\_=5,10,20} remains evident.

\begin{table}[t!]
\begin{center}
\begin{tabular}{ | c | c | c |  c | c | c | c | }
 \hline
MSE & $s=14$ & $s=15$ & $s=20$ & $s = 25$ & $s = 30$ & $ s = 40$ \\ \hline
 {\tt OHT\_K=1} & 0.519131  & 0.529622 &  0.376415 &   0.239696  &  0.138760 & 0.126439  \\
 {\tt OHT\_K=5} & 0.115403  & 0.070367 &  0.011384 &  0.007039 & 0.010385 & 0.012848 \\
 {\tt OHT\_K=10} & 0.070169 & 0.046840 & 0.007254 & 0.009965 & 0.011081 & 0.013396
 \\
 {\tt OHT\_K=20} & 0.046217 & 0.009727 & 0.007280 & 0.008915 & 0.011370 & 0.014919
 \\
 \hline
\end{tabular}
 \caption{ Empirical last-iterate MSE $\| \Theta_T - \Theta^*\|_2^2$ for $t_0  = 100$ and $T = 10^4$ under projection sparsity levels $s = 14,15,20,25,30,40$ and $[\Lambda_{i,i} , 1\leq i \leq p ] = [500,100,10,8,\cdots, 1,1,\cdots, 1]$. }\label{table:sparse}
\end{center}
\end{table}

Based on Table~\ref{table:sparse}, we observe that Algorithm~\ref{alg:1} performs well and only requires a slightly over-selected projection sparsity level, $s = 20$, even when the sparse condition number $\kappa = 500$ is large. This suggests that the dependence of the projection sparsity on $\kappa$ in Theorem~\ref{thm:multistep_sparse} may not be tight.

These numerical experiments indicate the superior performance of our algorithm compared to {\tt OS\_LASSO}, regardless of the covariate design, particularly in scenarios with weak signals or small initial batches. Additional experiments in Appendix~\ref{app:numerics} further demonstrate empirical convergence across online rounds and under various design and signal strength settings, highlighting the efficiency and robustness of our approach.

\subsection{Online Low-Rank Matrix Sensing}

We evaluate the performance of our algorithm for the online low-rank matrix sensing problem. We consider $d_1 = d_2 = 50$ and an underlying coefficient matrix $\Theta^* \in \RR^{d_1 \times d_2}$ of rank $s^* = 5$. In our experiments, $\Theta^*$ is obtained as the rank-5 approximation of a matrix $\tilde \Theta \in \RR^{d_1\times d_2}$ with entries $\tilde \Theta_{i,j} \sim \cN(0,1)$ independently. Each feature-label pair $(X_t,Y_t)$ is generated by independently sampling $X_t \in \RR^{d_1 \times d_2}$ with entries $[X_t]_{i,j} \sim \cN(0,1)$, followed by $Y_t = \langle X_t, \Theta^* \rangle + \epsilon_t$ with $\epsilon_t \sim \cN(0,1)$.
 We conduct the following experiments.
 \begin{enumerate}
\item With an initial batch of size $t_0 = 100$, we run $T=10^4$ online rounds for various projection ranks $s \in \{3,4,5,10,15,20\}$. For $s=5$, we plot the MSE $\|\Theta_t - \Theta^*\|_\rF^2$ versus the round $t$ in Figure~\ref{fig:lowrank_r5}, and the log-MSE $\log(\|\Theta_t - \Theta^*\|_\rF^2)$ versus $\log t$ to assess empirical convergence rates, including a reference line of slope $-1$. Similar results for $s=15$ are presented in Figure~\ref{fig:lowrank_r15}. The MSE of the last-iterate solution $\Theta_T$ at $T=10^4$ is reported for all algorithms in Table~\ref{table:1}.

\item Without an initial batch ($t_0 = 0$), we run $T=2000$ online rounds for projection ranks $s \in \{5,15\}$ and hard-thresholding gradient descent steps $K\in\{1, 5, 10, 20\}$, reporting $\|\Theta_t - \Theta^*\|_\rF^2$ versus $t$ and $\log t$ in Figure~\ref{fig:lowrank_free_initial}.

\end{enumerate}

From Figures~\ref{fig:lowrank_r5} and \ref{fig:lowrank_r15}, Algorithm~\ref{alg:1} generates solution sequences converging to $\Theta^*$ for $s=5,15$ and $K = 1,5,10,20$, with an initial batch of size $t_0=100$. Since the covariance matrix is isotropic and $s^* = 5$, only a slightly larger projection rank is required to ensure convergence. The slopes of $\log(\|\Theta_t - \Theta^* \|_\rF^2)$ versus $\log t$ are approximately $-1$ for both $s=5,15$ when $K\geq 5$, consistent with Theorem~\ref{thm:multistep_lowrank}, which predicts an $\cO(s/(t+t_0))$ convergence rate.

Table~\ref{table:1} shows that under-picking the projection level (e.g., $s=3,4$) prevents convergence. Additionally, the MSE for $s=20$ is roughly four times that for $s=5$, confirming the linear dependence of the MSE on the projection sparsity $s$ as indicated in Theorem~\ref{thm:multistep_lowrank}.

Finally, Figure~\ref{fig:lowrank_free_initial} demonstrates that our algorithms perform well even without an initial batch, with ${\Theta_t}$ converging to $\Theta^*$ for $s \in \{5,15\}$. These results validate the practical efficiency and robustness of our method and corroborate the theoretical guarantees in Section~\ref{sec:nodata_lowrank}.

\begin{table}[t]
\begin{center}
\begin{tabular}{ | c | c | c |  c | c | c | c | }
 \hline
MSE & $s=3$ & $s=4$ & $s=5$ & $s = 10$ & $s = 15$ & $ s = 20$ \\ \hline
 {\tt OHT\_K=1} & 10.8453  & 5.6856 &  0.0546 &  0.1265 &  0.1848 & 0.2395  \\
 {\tt OHT\_K=5} & 10.8447  & 5.6460 &  0.0508 &  0.1357 & 0.1968 & 0.2499 \\
 {\tt OHT\_K=10} & 10.8367 & 5.6276 & 0.0503 & 0.1367 & 0.1958 & 0.2451
 \\
 {\tt OHT\_K=20} & 10.8299 & 5.2613 & 0.0499 & 0.1350 & 0.1990 & 0.2516
 \\
 \hline
\end{tabular}
 \caption{ Empirical last-iterate MSE $\| \Theta_T - \Theta^*\|_\rF^2$ for $t_0  = 100$ and $T = 10^4$ under projection rank levels $s = 3,4,5,10,15,20$. }\label{table:1}
\end{center}
\end{table}

\begin{figure}[t]
    \centering
        \includegraphics[width=0.45\linewidth]{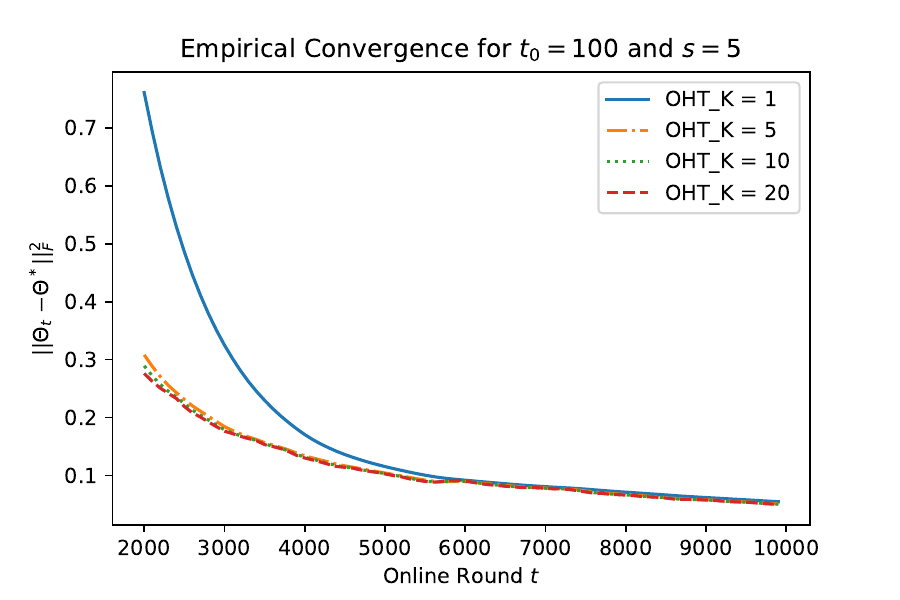}
  \includegraphics[width=0.45\linewidth]{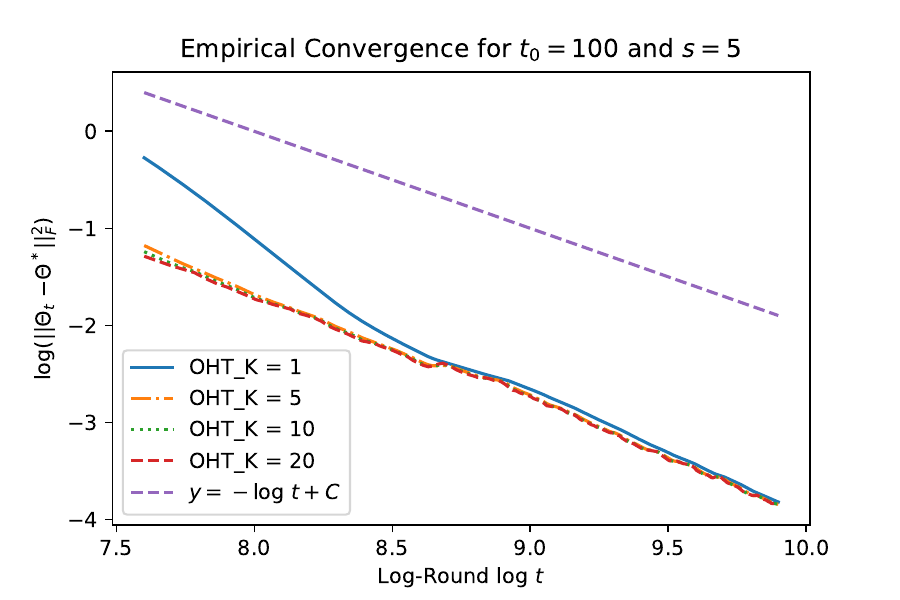}
    \caption{Online low-rank matrix sensing for $t_0 = 100$ and $T = 10000$ online learning rounds  with projection sparsity level $s=5$.}
    \label{fig:lowrank_r5}
\end{figure}

\begin{figure}[t]
    \centering
        \includegraphics[width=0.45\linewidth]{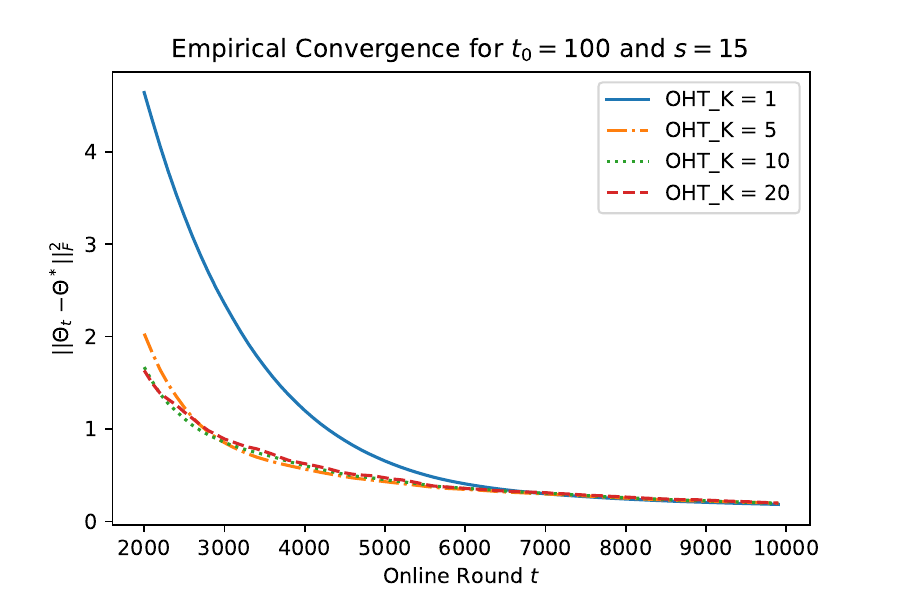}
  \includegraphics[width=0.45\linewidth]{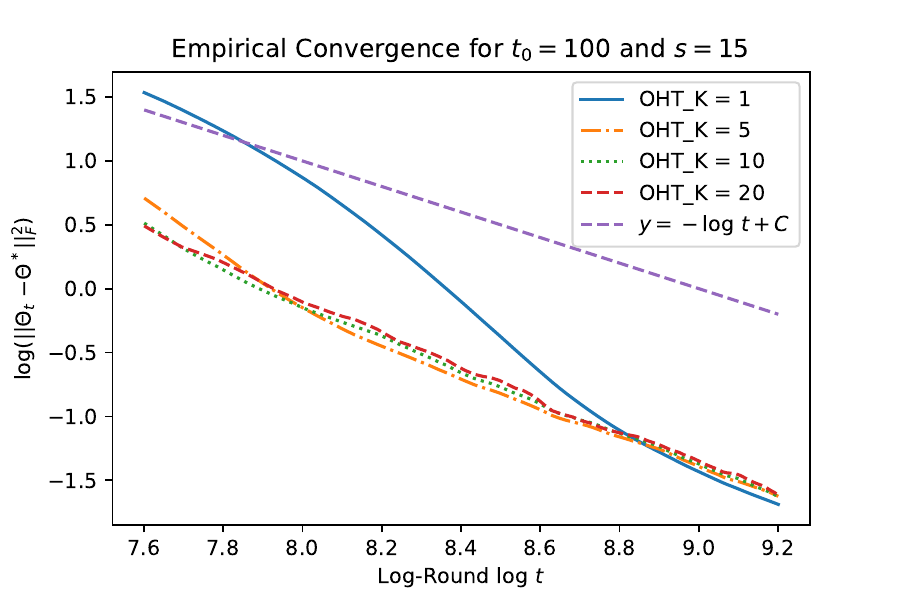}
    \caption{Online low-rank matrix sensing for $t_0 = 100$ and $T = 10000$ online learning rounds  with projection sparsity level $s=15$.}
    \label{fig:lowrank_r15}
\end{figure}

\begin{figure}[t]
    \centering
        \includegraphics[width=0.45\linewidth]{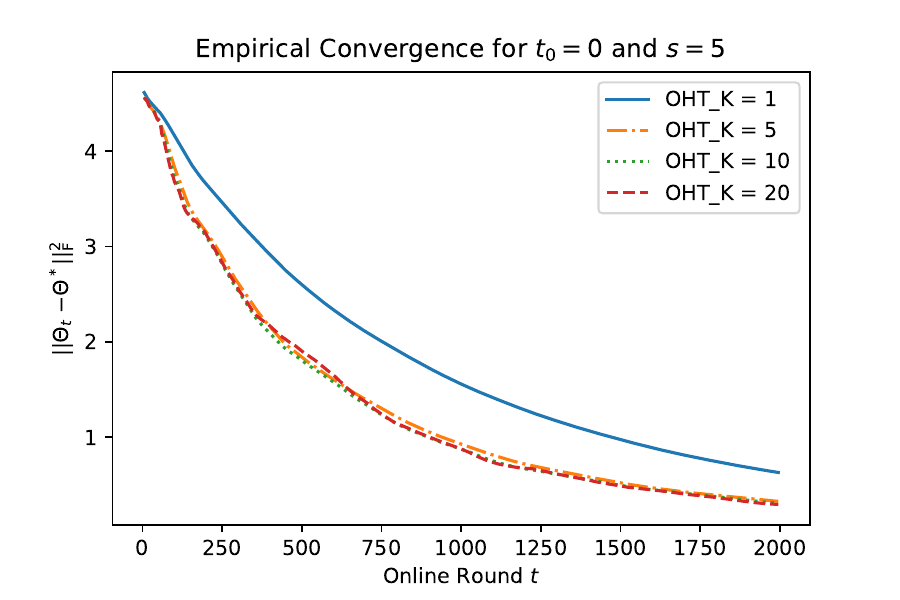}
  \includegraphics[width=0.45\linewidth]{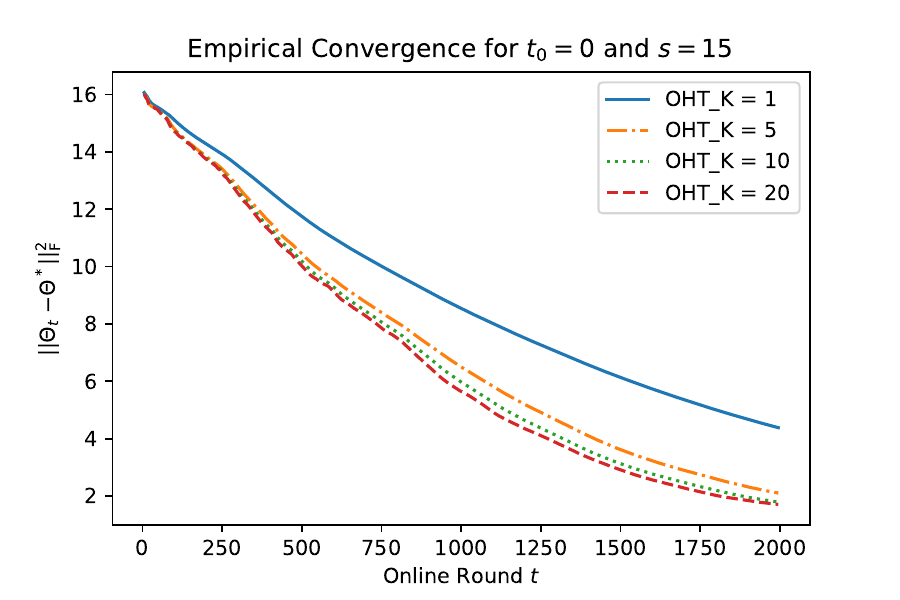}
    \caption{Online low-rank matrix sensing \emph{without} the initial batch ($t_0 = 0$) and $T = 2000$ online learning rounds  with projection rank levels $s \in \{ 5, 15 \}$.}
    \label{fig:lowrank_free_initial}
\end{figure}

\section{Conclusions}

This paper proposes an online hard thresholding algorithm for online generalized sparse regression problems, with a focus on online sparse linear regression and low-rank matrix sensing. The algorithm does not need dynamic regularization, features closed-form updates in each online round, requires only the storage of summary statistics, making it memory and storage efficient, and converges globally to the ground-truth coefficients at the optimal statistical rate under realistic assumptions. To our knowledge, it is the first algorithm that simultaneously achieves all four of these properties. One practical limitation is that key quantities such as the per-round iteration number $K$ and projection sparsity $s$ are typically unavailable in real-world applications.
In practice, the projection sparsity level can be chosen via cross-validation using an initial batch. If no initial batch is available, cross-validation can be performed during the early online rounds, after which the algorithm proceeds in a streaming fashion. Moreover, numerical experiments demonstrate that the algorithm is robust to the choices of $s$ and $K$, provided that $s$ is not under-picked.

\section*{Acknowledgments}
Qiang Sun is supported in part by an NSERC Discovery Grant (RGPIN-2018-06484), a Data Sciences Institute Catalyst Grant, and a computing grant from Compute Canada.

\bibliographystyle{plainnat}
\bibliography{online}

\clearpage
\beginsupplement
\supplementtableofcontents

\vspace{10pt}

This appendix presents additional numerical experiments, and  collects  proofs for the main results and technical lemmas. Additional numerical experiments are presented  in Appendix \ref{app:numerics}.   Appendix~\ref{app:1} collects proofs for online sparse linear regression in Section \ref{sec:multistep_sparse}, and Appendix~\ref{app:2} collects the proofs  for online low-rank matrix sensing in Section \ref{sec:multistep_lowrank}.  Appendix \ref{app:3} collects proofs for the no initial batch case in Section \ref{sec:extensions}.

\emph{Notation.} For any two spaces $S_1$ and $S_2$, with a slight abuse of notation, we denote by $S_1 \cap S_2^\perp $ the intersection between $S_1$ and the orthogonal complement of $S_2$.


\section{Additional Numerical Results}\label{app:numerics}
In this section, we conduct additional numerical experiments to investigate the performance of our  OHT algorithms under various signal strength and covariate correlation setups. We consider the following two experiments with sparsity level $s= 50$ and step-size $\eta = 0.001$:
\begin{enumerate}
    \item
\textbf{Weak signal under different covariate correlations:} We test the algorithms for initial batch size $t_0 = 100, 500$ and $10^4$ online learning rounds, under both independent and Toeplitz $\rho = 0.3,0.5,0.7$ covariate correlation designs provided in Section~\ref{sec:numerics}. We report the MSE $\|\Theta_t - \Theta^* \|_\rF^2$ against online learning rounds $t$ in Figures \ref{fig:weak_MSE_independent}, \ref{fig:weak_MSE_toeplitz_03}, \ref{fig:weak_MSE_toeplitz_05}, and \ref{fig:weak_MSE_toeplitz_07}.

\item \textbf{Strong signal under independent covariate correlation:} We generate a vector $\hat \Theta^* \in \RR^d$ where each entry is normally distributed such that $\hat \Theta_i^*\sim \cN(0,0.4)$; we conduct hard thresholding truncation such that $\Theta_i^* = \hat \Theta_i^*$ if $|\hat \Theta_i^*| \geq 1$ and $\Theta_i^* = 0$ otherwise. The cardinality of the generating parameter $\Theta^*$ is 12. We then test our algorithms for initial batch size $t_0 = 100,500$ and $10^4$ subsequent online learning rounds, under the independent covariate correlation design. We plot the  MSE $\| \Theta_t - \Theta^*\|_\rF$ against online learning rounds $t$ in Figure \ref{fig:strong_MSE}.
\end{enumerate}

\begin{figure}[t]
    \centering
            \includegraphics[width=0.45\linewidth]{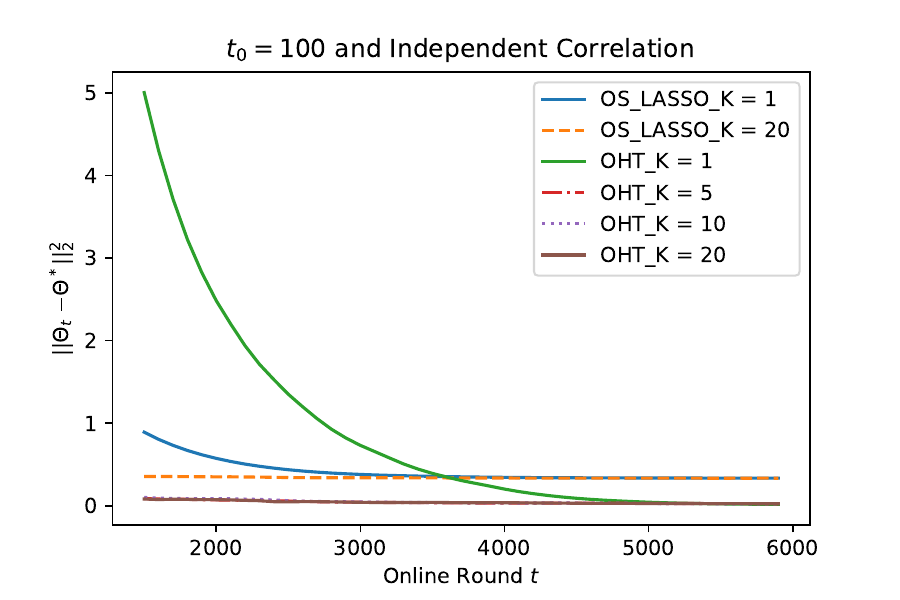}
  \includegraphics[width=0.45\linewidth]{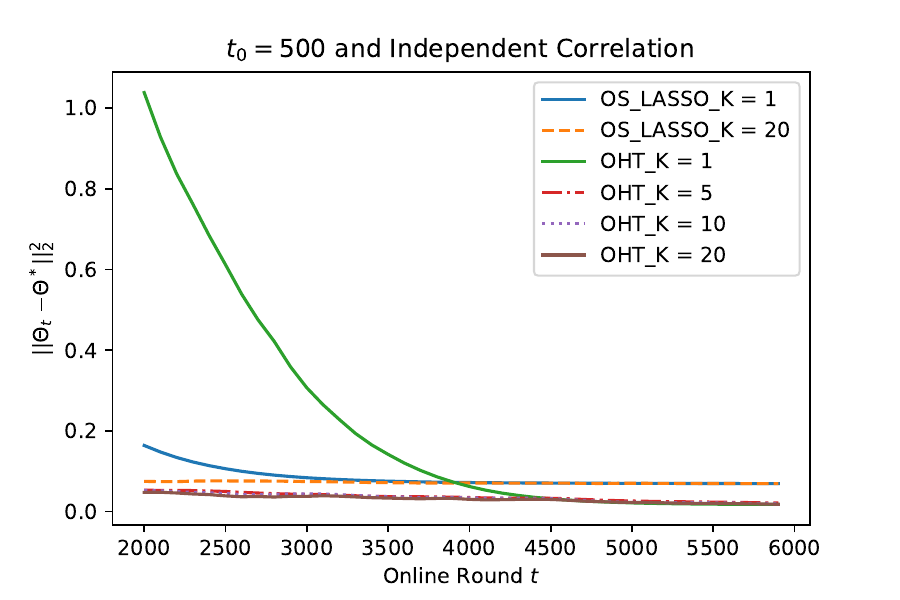}
    \caption{Empirical convergence of  MSE $\| \Theta_t - \Theta^*\|_\rF^2$ against online learning round $t$ under weak signal setup and independent covariate correlation}
    \label{fig:weak_MSE_independent}
\end{figure}

\begin{figure}[t]
    \centering
            \includegraphics[width=0.45\linewidth]{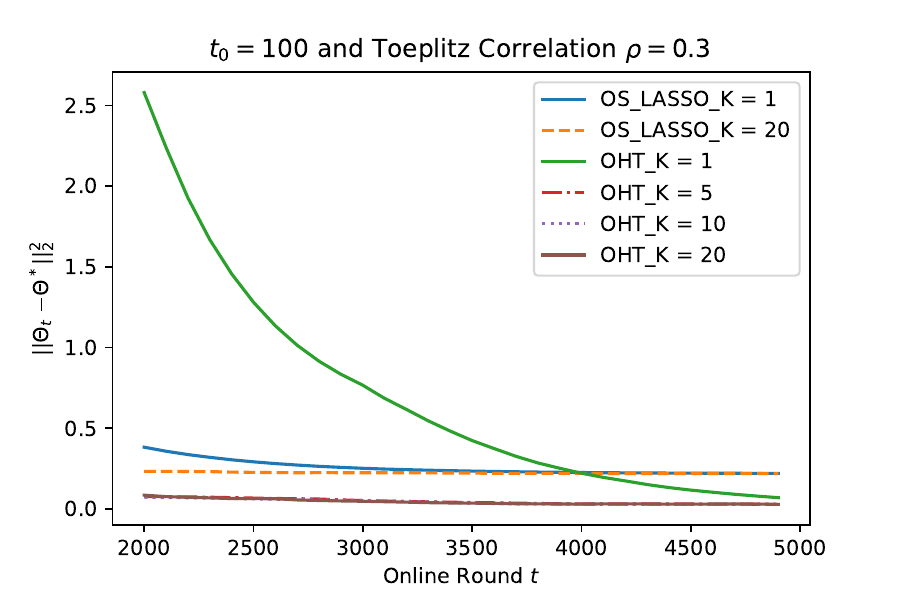}
  \includegraphics[width=0.45\linewidth]{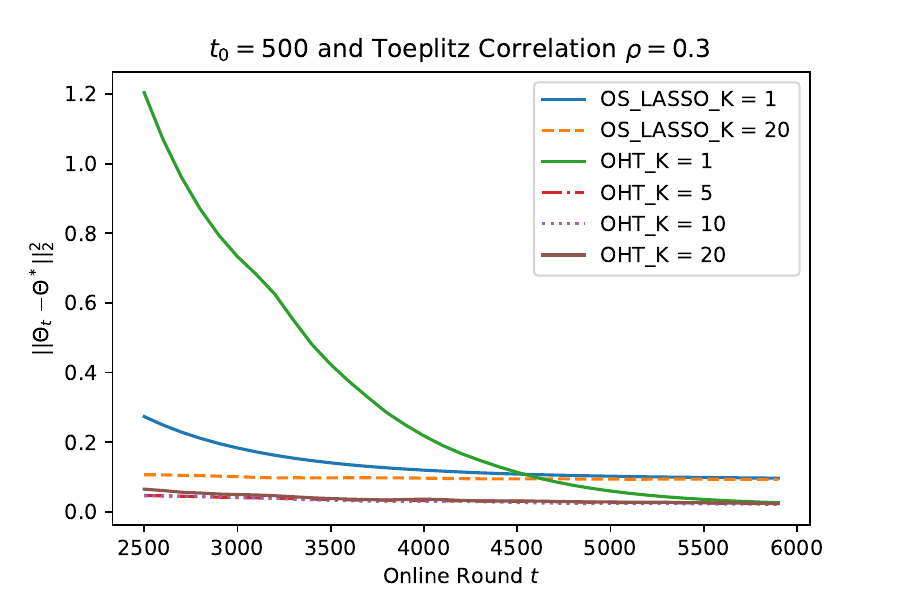}
    \caption{Empirical convergence of MSE $\| \Theta_t - \Theta^*\|_\rF^2$ against online learning round $t$ and weak signal setup with Toeplitz $\rho = 0.3$ covariate correlation}
    \label{fig:weak_MSE_toeplitz_03}
\end{figure}

\begin{figure}[t]
    \centering
            \includegraphics[width=0.45\linewidth]{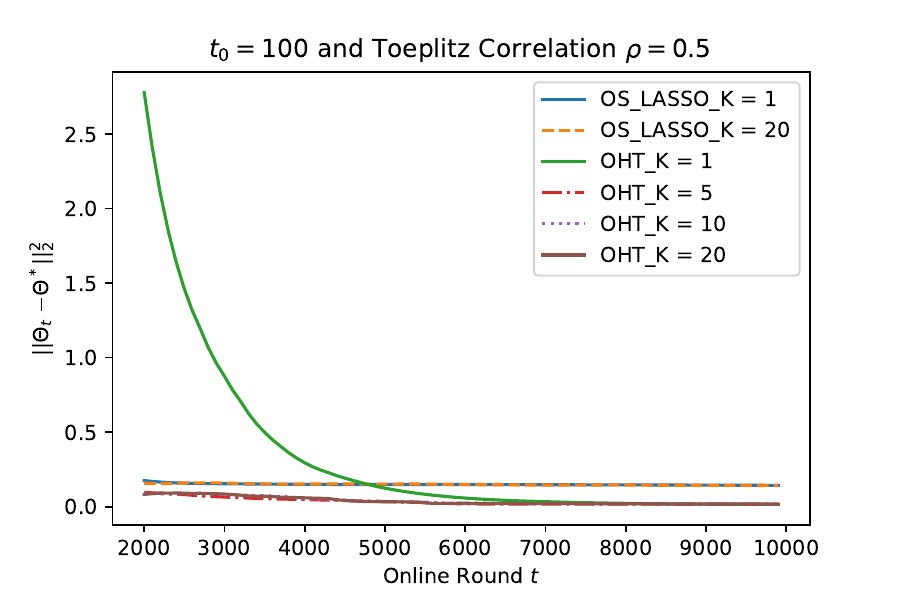}
  \includegraphics[width=0.45\linewidth]{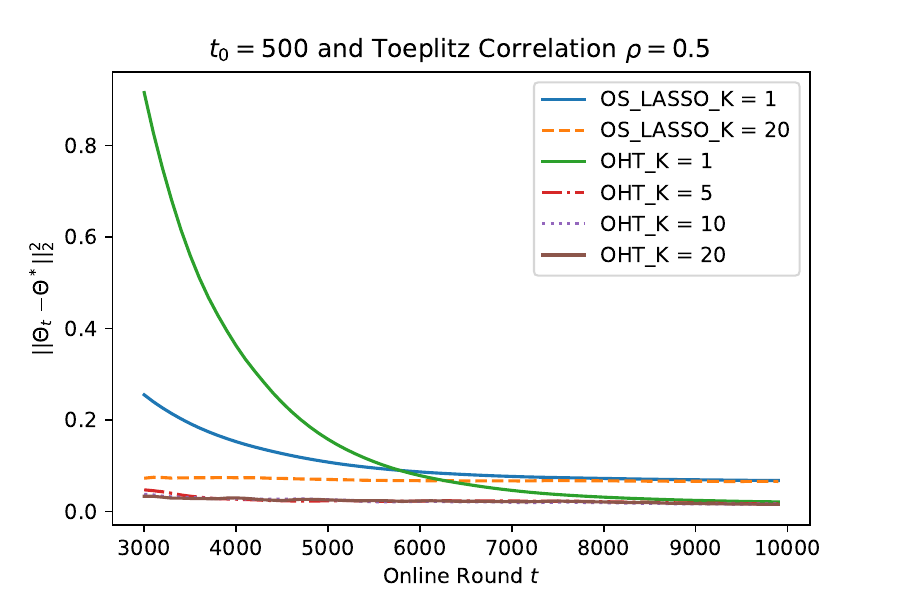}
    \caption{Empirical convergence of  MSE $\| \Theta_t - \Theta^*\|_\rF^2$ against online learning round $t$ under weak signal setup and Toeplitz $\rho = 0.5$ covariate correlation}
    \label{fig:weak_MSE_toeplitz_05}
\end{figure}

\begin{figure}[t]
    \centering
            \includegraphics[width=0.45\linewidth]{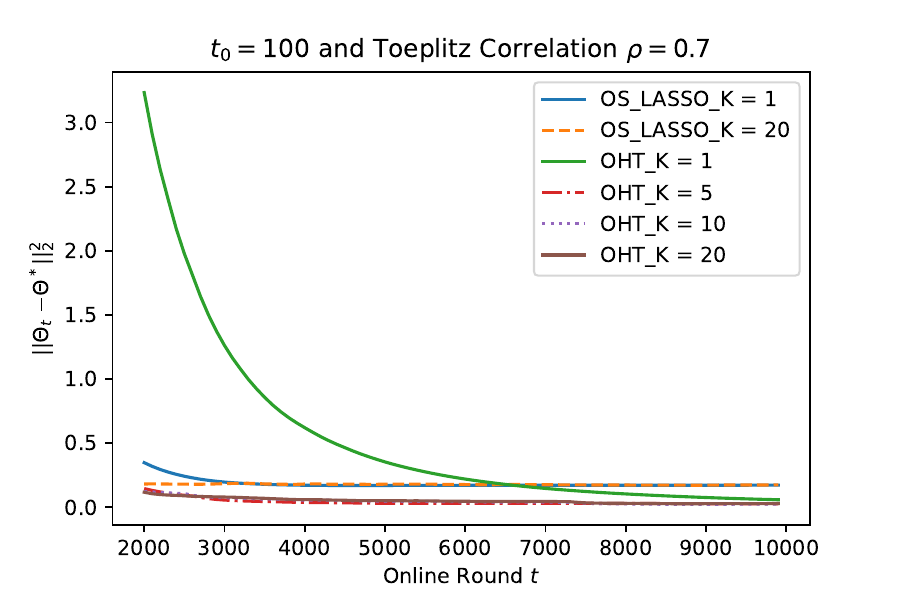}
  \includegraphics[width=0.45\linewidth]{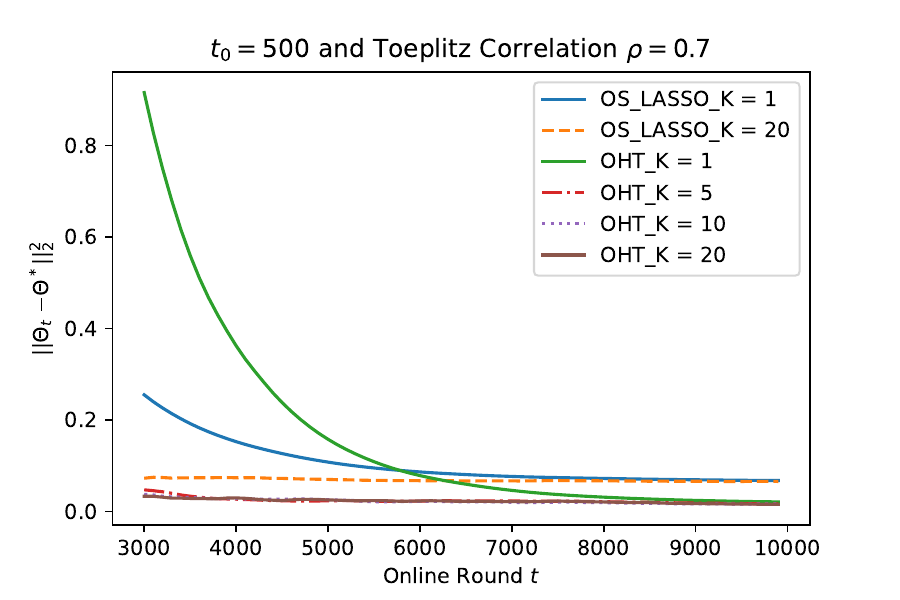}
    \caption{Empirical convergence of MSE $\| \Theta_t - \Theta^*\|_\rF^2$ against online learning round $t$ under weak signal setup and Toeplitz $\rho = 0.7$ covariate correlation}
    \label{fig:weak_MSE_toeplitz_07}
\end{figure}

\begin{figure}[t]
    \centering
            \includegraphics[width=0.45\linewidth]{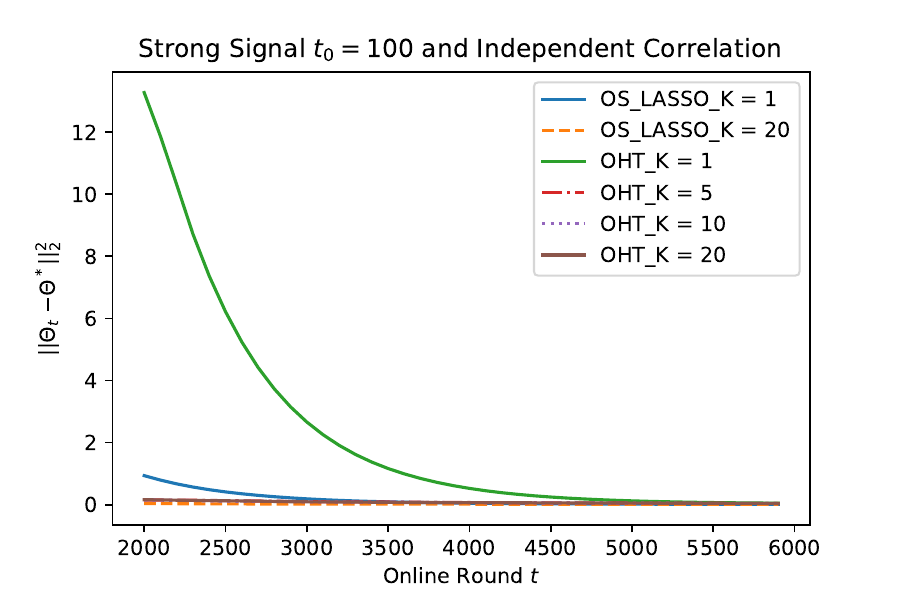}
  \includegraphics[width=0.45\linewidth]{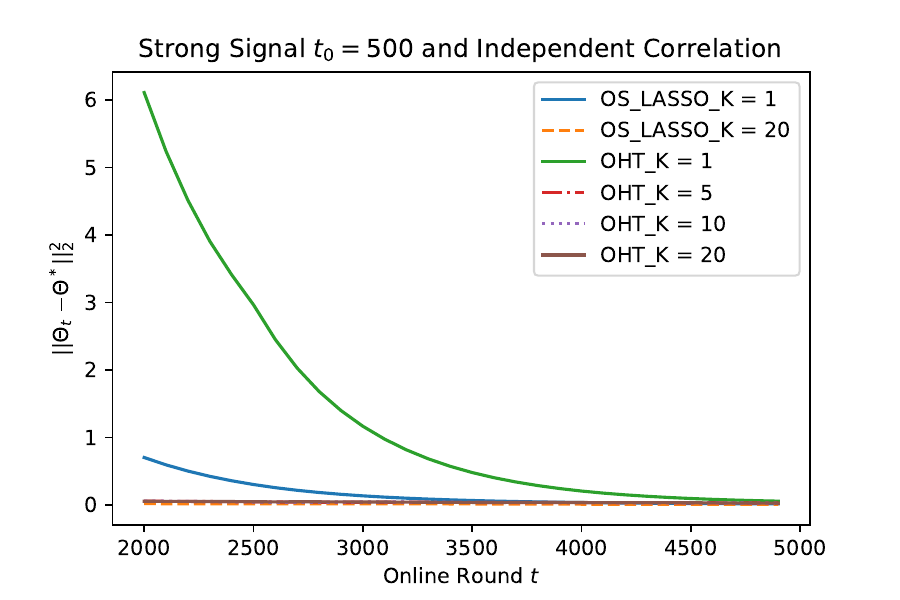}
    \caption{ Empirical convergence of  MSE $\| \Theta_t - \Theta^*\|_\rF^2$ against online learning round $t$ under strong signal setup and independent covariate correlation}
    \label{fig:strong_MSE}
\end{figure}

From Figures \ref{fig:weak_MSE_independent}, \ref{fig:weak_MSE_toeplitz_03}, \ref{fig:weak_MSE_toeplitz_05}, and \ref{fig:weak_MSE_toeplitz_07}, we observe that our OHT algorithms outperform \allowbreak{ \tt OS\_LASSO\_K=1} and \allowbreak{ \tt OS\_LASSO\_K=20} in the weak signal setting for $t_0 = 100,500$, regardless of the covariate correlation design.
Further, we observe that { \tt OS\_LASSO\_K=1} and { \tt OS\_LASSO\_K=20} exhibit comparable performance to our algorithms for strong-signal instances. This is because under strong-signal settings, { \tt OS\_LASSO\_K=1} and { \tt OS\_LASSO\_K=20} are able to identify the correct set of nonzero entries $\cS$ by using a small initial batch $t_0 = 100$. Consequently, the solutions are truncated to the correct set in each subsequent online learning round. These numerical observations further demonstrate the efficiency and robustness of our algorithm in the weak-signal and small-initial-batch scenarios.

\section{Proofs for Section~\ref{sec:multistep_sparse}} \label{app:1}

\subsection{Proof of Proposition \ref{prop:subW}}

\begin{proof}
Conditional on $\epsilon_i$, we apply \cite[Proposition 2.6.1]{vershynin2018high} to obtain
\$
\|Z_j\|_{\psi_2}
&= \left\| \left[\sum_{i \in \cI_t} \frac{ \epsilon_i X_i}{ \sqrt{ t+t_0} } \right]_j \right\|_{\psi_2}
\leq \sqrt{\frac{C}{t+t_0} \sum_{i\in I_t} \epsilon_i^2 \left\| X_{ij} \right\|^2_{\psi_2} }
\leq \sqrt{\frac{C\sigma_x^2}{t+t_0} \sum_{i\in I_t}\epsilon^2_i},
\$
where $C$ is some constant.

Since $Z_j$ is sub-Gaussian with $\|Z_j\|_{\psi_2}$, conditional on $\epsilon_i$, we have
\$
\|Z_j^2\|_{\psi_1}
= \|Z_j\|^2_{\psi_2}
= \frac{C\sigma_x^2}{t+t_0} \sum_{i\in I_t} \epsilon_i^2.
\$

Let $V_{i}= \epsilon_i X_{ij}$. Then, we apply the maximum inequality  \cite[Lemma 2.2.2]{van1996weak}, and obtain
\begin{align*}
   \EE\left[ \max_{1\leq j \leq d}  Z_j^2 ~\Big|~ \epsilon_i, i \in I_t \right]
&\leq \left\|\max_{1\leq j \leq d}  Z_j^2 \right\|_{\psi_1}
\\
& \leq C \psi_1^{-1}(d) \, \left\| \frac{1}{t + t_0} \sum_{i\in \cI_t} \epsilon_i^2 X_{ij}^2\right\|_{\psi_1} \leq C \log d \cdot \frac{\sigma_x^2}{t+t_0}\sum_{i\in I_t}\epsilon_i^2,
\end{align*}
and thus
\$
\EE\left[ \max_{1\leq j \leq d}  Z_j^2 \right]
&\leq C\sigma^2\sigma_x^2 \log d,
\$
where $C$ is some constant.

\end{proof}

\subsection{Proof of Lemma~\ref{lemma:stats_error}}   \label{app:proof_of_lemma:stats_error}
\begin{proof}
Recall that $\hat \Theta_t = \arg \min_{\| \Theta \|_0 \leq s^*} \cL_t(\Theta)$ for cardinality-constrained linear regression where $\Upsilon(\Theta) = \| \Theta\|_0$, we have
\begin{equation*}
\begin{split}
\cL_t(\hat \Theta_t) & = \frac{1}{2( t_0 +t)}   \sum_{i \in \cI_t} \Big ( \lv X_i,  \hat \Theta_t \rv  - Y_i  \Big  )^2     =   \frac{1}{2( t_0 +t)}   \sum_{i \in \cI_t} \Big ( \lv X_i,   \hat \Theta_t - \Theta^* \rv  - \epsilon_i   \Big )^2 \\
&  \leq \frac{1}{2( t_0 +t)}\sum_{i \in \cI_t}\Big  ( \lv X_i,  \Theta^* \rv  - Y_i  \Big  )^2    = \frac{1}{2( t_0 +t)}  \sum_{i \in \cI_t}  \epsilon_i ^2   =  \cL_t( \Theta^*),
\end{split}
\end{equation*}
implying that
\begin{equation*}
 \frac{1}{2( t_0 +t)}   \sum_{i \in \cI_t} \lv  X_i , \hat \Theta_t  - \Theta^*\rv^2  \leq \frac{1}{t_0 +t}\sum_{i \in \cI_t}  \epsilon_i  \lv  X_i , \hat \Theta_t  - \Theta^*\rv .
\end{equation*}
Because $\cL_t({\cdot})$ satisfies the $(s,s)$-SSC condition under Assumpption \eqref{assumption:rsc_2s}, using the fact that $\| \hat \Theta_t - \Theta^*\|_0 \leq 2s^*$
and applying Proposition~\ref{prop:insample_error}  in Appendix Section~\ref{sec:tech_lemma}
 arrives at
\begin{equation}\label{eq:insample}
\begin{split}
\frac{\alpha }{2}\| \hat \Theta_t  - \Theta^*\|_2^2
& \leq  \frac{1}{2( t_0 +t)}   \sum_{i \in \cI_t} \lv  X_i , \hat \Theta_t - \Theta^*\rv^2
\leq \frac{1}{t_0 +t}\sum_{i \in \cI_t}  \epsilon_i  \lv  X_i ,  \hat \Theta_t  - \Theta^* \rv
\\
& \leq  \sqrt{2s^*} \|  \Theta'  - \Theta^*  \|_2  \Big \| \frac{1}{t_0 + t}   \sum_{i \in \cI_t}   \epsilon_i   X_i \Big \|_\infty
\\
& \leq \frac{2s^*}{\alpha }   \Big   \| \frac{1}{t_0 + t}   \sum_{i \in \cI_t}   \epsilon_i   X_i \Big   \|_\infty^2 + \frac{\alpha }{4} \|  \hat \Theta_t - \Theta^*  \|_2^2.
\end{split}
\end{equation}
Rearranging the terms, we obtain
\begin{equation*}
\begin{split}
  \|  \hat \Theta_t - \Theta^*  \|_2^2   &   \leq \frac{8s^*}{\alpha^2 }   \Big \|  \frac{1}{t_0 + t}   \sum_{i \in \cI_t}   \epsilon_i   X_i  \Big \|_\infty^2  .
\end{split}
\end{equation*}
By using Assumption \ref{assumption:noise}  that $\EE  \Big  [ \Big   \|     \sum_{i \in \cI_t}  \epsilon_i   X_i \Big   \|_\infty^2 \Big   ]  \leq  ( t_0 + t)  \sigma_x^2$, and taking expectations on both sides of the above inequality, we conclude that
 \begin{equation*}
\begin{split}
 \EE[  \|  \hat \Theta_t -  \Theta^* \|_2^2  ] &   \leq \frac{8s^*}{\alpha^2 }  \EE \Big [ \Big \| \frac{1}{t_0 + t}   \sum_{i \in \cI_t}   \epsilon_i   X_i   \Big \|_\infty^2 \Big ] \leq  \frac{8s^*\sigma_x^2 }{\alpha^2( t_0 + t)}  .
\end{split}
\end{equation*}
This completes the proof.
\end{proof}


\subsection{Proof of Lemma~\ref{lemma:sdl}} \label{app:proof_of_lemma_descent}
\begin{proof}
When $\Upsilon(\Theta) = \| \Theta\|_0$, we consider a hard thresholding step of Algorithm~\ref{alg:1}  that
\$
\Theta_{t, k +1} & = \Pi_{\cC_s} \left( \Theta_{t,k} - \eta_{t,k}  \nabla \cL_{t}( \Theta_{t, k }) \right)
\\
& =  \argmin_{z \in \RR^d, \|z\|_0\leq s} \left \{ \|z - (\Theta_{t,k} - \eta_{t,k} \nabla \cL_t(\Theta_{t,k}))   \|_2^2 \right \}.
\$
By Assumption~\ref{assumption:rsc_2s}, $\cL_t(\cdot)$ satisfies the $(s,s)$-SSS  condition  \eqref{assumption:rsc_2s}  with parameter $L$.
Let $g_k =\nabla \cL_t( \Theta_{t,k})$.  We then analyze the change of objective value from $\Theta_{t,k}$ to $\Theta_{t,k+1}$ w.r.t. $\cL_t(\cdot)$ as
\begin{equation} \label{eq:sdl:1}
\begin{split}
\cL_t(\Theta_{t, k+1})-\cL_t(\Theta_{t, k})
&\leq \left\langle \nabla  \cL_t( \Theta_{ t, k}), \Theta_{t, k+1}-\Theta_{t,k}\right\rangle  +\frac{L}{2}\left\|\Theta_{t, k+1}- \Theta_{t, k}\right\|_2^2 \\
&= \left\langle g_k, \Theta_{t,k+1}-\Theta_{t,k}\right\rangle  +\frac{L}{2}\left\|\Theta_{t, k+1}- \Theta_{t, k}\right\|_2^2  = \Rom{1}+\Rom{2},
\end{split}
\end{equation}
where $\Rom{1}: = \left\langle g_k, \Theta_{t,k+1}-\Theta_{t,k}\right\rangle $ and $\Rom{2} : =\frac{L}{2}\left\|\Theta_{t,k+1}- \Theta_{t,k}\right\|_2^2 $. For simplicity, we write
$\eta = \eta_{t,k}$, $\Theta_k = \Theta_{t,k}$, $\Theta_{k+1} = \Theta_{t,k+1}$, $\cS_{k-1} = \cS_{{t,k-1}}$ and $\cS_k = \cS_{{t,k}}$.
We proceed to bound \Rom{1} and \Rom{2} respectively.  We start with \Rom{1}:
\begin{equation*}
\begin{split}
    \Rom{1} & = \left\langle g_k, \Theta_{k+1}-\Theta_{k}\right\rangle
    \\
&= -\lv \Theta^{k}_{\cS_k \setminus\cS_{k +1}}, g^k_{\cS_k\setminus\cS_{k+1}} \rv +\lv \Theta^{k+1}_{\cS_{k+1}}   -\Theta^k_{\cS_{k+1}}, g^k_{\cS_{k+1}} \rv
\\
&= -\lv \Theta^k_{\cS_k\setminus\cS_{k+1}}, g^k_{\cS_k \setminus\cS_{k+1}} \rv   -\eta\|g^k_{\cS_{k+1}}\|_2^2
\\
& =  -\lv \Theta^k_{\cS_k\setminus\cS_{k+1}} - \eta g^k_{\cS_k\setminus\cS_{k+1}}   , g^k_{\cS_k \setminus\cS_{k+1}} \rv  - \eta \| g^k_{\cS_k\setminus\cS_{k+1}}  \|_2^2  -\eta\|g^k_{\cS_{k+1}}\|_2^2
\\
& \leq \frac{1}{2\eta} \|   \Theta^k_{\cS_k\setminus\cS_{k+1}} - \eta g^k_{\cS_k\setminus\cS_{k+1}} \|_2^2 - \frac{\eta}{2} \| g^k_{\cS_k\setminus\cS_{k+1}}  \|_2^2  -\eta\|g^k_{\cS_{k+1}}\|_2^2
\\
&\leq \frac{\eta}{2}\|g^k_{\cS_{k+1}\setminus \cS_k}\|_2^2-\frac{\eta}{2}\| g^k_{\cS_k\setminus \cS_{k+1}} \|_2^2-\eta\|g_{\cS_{k+1}}^k\|_2^2
\\
&\leq -\frac{\eta}{2}\|g^k_{\cS_k\cup\cS_{k+1}}\|_2^2,
\end{split}
\end{equation*}
where the last  second inequality follows from the fact
\begin{equation*}
\| \Theta^k_{\cS_k\setminus \cS_{k+1}}-\eta g^k_{\cS_k\setminus \cS_{k+1}}\|_2^2\leq \| \Theta^{k+1}_{\cS_{k+1}\setminus \cS_{k}}\|_2^2=\eta^2\| g^k_{\cS_{k+1}\setminus \cS_{k}}\|_2^2.
\end{equation*}
For \Rom{2}, by using the fact that $\| a\|_2^2 = \| a+b\|_2^2 - \| b \|_2^2 - 2 \lv a, b \rv$, we have
\begin{equation*}
    \begin{split}
  \Rom{2}&=\frac{L}{2}\left\|\Theta_{k+1}- \Theta_{k}\right\|_2^2\\
&= \frac{L}{2}\left\|\Theta_{k+1}- \Theta_{k}+\eta g^k_{\cS_{k+1}\cup\cS_t}\right\|_2^2
-\frac{L}{2} \left\|\eta g^k_{\cS_{k+1}\cup\cS_k}\right\|_2^2-L\eta\cdot  \left\langle g^k_{\cS_{k+1}\cup\cS_k} , \Theta_{k+1}-\Theta_{k}\right\rangle \\
&= \frac{L }{2}\inf_{z\in \cC_s}\left\|z- \Theta_{k}+\eta g^k_{\cS_{k+1}\cup\cS_k}\right\|_2^2-\frac{L }{2} \left\|\eta g^k_{\cS_{k+1}\cup\cS_k}\right\|_2^2-L  \eta\cdot \Rom{1}\\
&\leq -L  \eta\cdot \Rom{1} .
    \end{split}
\end{equation*}
Plugging the above bounds for \Rom{1} and \Rom{2} into \eqref{eq:sdl:1} acquires
\begin{equation*}
    \begin{split}
     \cL(\Theta_{k+1})-\cL(\Theta_k)
&\leq  (1-L\eta)\cdot \Rom{1} \leq -\frac{\eta (1-L\eta)}{2}\|g_{\cS_k\cup\cS_{k+1}}^k\|_2^2,
    \end{split}
\end{equation*}
as desired.
\end{proof}

\subsection{Proof of Lemma~\ref{lemma:multi_step_descent}}\label{app:proof_lemma_multi_step_desecnt}

\begin{proof}
By using the $(s,s)$-SSC and $(s,s)$-SSS conditions \eqref{assumption:rsc_2s} with $\Upsilon(\Theta) = \| \Theta\|_0$ and the fact that $\| \Theta_{t,k-1} \|_0, \| \Theta_{t,k} \|_0, \| \Theta^*\|_0 \leq s$,  we have the followings:
\begin{equation*}
\begin{split}
\cL_t (   \Theta^*  ) & \geq \cL_t ( \Theta_{t, k-1}) + \lv \nabla \cL_t ( \Theta_{t, k -1}) ,  \Theta^*   -  \Theta_{t, k-1} \rv + \frac{\alpha }{2} \|   \Theta^*  -  \Theta_{t, k-1} \|_2^2,
\\
 \cL_t ( \Theta_{t, k}) & \leq  \cL_t ( \Theta_{t, k - 1}) + \lv \nabla \cL_t ( \Theta_{t, k -1}), \Theta_{ t, k } -  \Theta_{t, k -1} \rv   + \frac{L }{2} \|  \Theta_{t, k} -  \Theta_{t, k -1}\|_2^2
\\
& \leq  \cL_t (\Theta_{t, k -1}) + \lv \nabla \cL_t (\Theta_{t, k -1}),  \Theta_{ t, k } -  \Theta_{t, k -1}\rv + \frac{1}{2\eta } \|  \Theta_{ t, k } -  \Theta_{t, k -1}\|_2^2.
\end{split}
\end{equation*}
The above two inequalities suggest that
\begin{equation*}
\cL_t( \Theta_{ t, k}  ) - \cL_t (   \Theta^*   ) \leq  \lv \nabla \cL_t ( \Theta_{t,k-1}),  \Theta_{t,k}-   \Theta^* \rv + \frac{1}{2 \eta } \| \Theta_{t,k} - \Theta_{t,k-1}\|_2^2  - \frac{\alpha  }{2} \|   \Theta^*  - \Theta_{t, k -1}\|_2^2.
\end{equation*}
Meanwhile, we have
\begin{equation*}
\begin{split}
& \frac{1}{2\eta} \| \Theta_{t,k} - \Theta^* \|_2^2
\\
& =  \frac{1}{2\eta} \| \Theta_{t,k-1} - \Theta^*\|_2^2  -  \frac{1}{2\eta} \| \Theta_{t,k} - \Theta_{t,k-1}\|_2^2 + \frac{1}{\eta} \lv \Theta_{t,k-1} - \Theta_{t,k}, \Theta^* - \Theta_{t,k} \rv
\\
& =  \frac{1}{2\eta} \|\Theta_{t,k-1} -\Theta^* \|_2^2  -  \frac{1}{2\eta} \|\Theta_{t,k} - \Theta_{t,k-1}\|_2^2 + \frac{1}{\eta} \lv \Theta_{t,k-1} -  \eta \nabla \cL_t(\Theta_{t,k-1}) - \Theta_{t,k}, \Theta^* - \Theta_{t,k} \rv
\\
& \quad
+   \lv  \nabla \cL_t(\Theta_{t,k-1}) ,  \Theta^* - \Theta_{t,k} \rv
\\
& \leq \frac{1}{2\eta} \| \Theta_{t,k-1} -  \Theta^*  \|_2^2  -  \frac{1}{2\eta} \| \Theta_{t,k} - \Theta_{t,k-1}\|_2^2 + \frac{\gamma_{s,s^*}}{\eta}  \|  \Theta^* - \Theta_{t,k} \|_2^2
 +   \lv  \nabla \cL_t(\Theta_{t,k-1}) ,  \Theta^*  - \Theta_{t,k} \rv,
\end{split}
\end{equation*}
where we use the update rule that $\Theta_{t,k} =\Pi_{\cC_s} (\Theta_{t,k-1} - \eta \nabla \cL_t(\Theta_{t,k-1}))$ and
\begin{equation*}
\gamma_{s, s^*}= \sup \left \{  \frac{ \lv Z - \Pi_{\cC_s} (Z)  , Y - \Pi_{\cC_s}(Z) \rv  }{ \| Y - \Pi_{\cC_s} (Z)\|_2^2}, Y, Z \in \RR^{d}, \|  Y \|_0 \leq s^*, Y \neq \Pi_{\cC_s}(Z)\right \}
\end{equation*}
for any sparsity pair $(s^*,s)$.
By combining the above two inequalities, we have
\begin{equation*}
\cL_t ( \Theta_{t,k}) - \cL_t ( \Theta^* ) \leq \frac{1}{2\eta } \left [  (1-\eta  \alpha ) \|\Theta_{t,k-1} - \Theta^*  \|_2^2  - (1-2\gamma_{s,s^*}) \| \Theta_{t,k} - \Theta^*  \|_2^2
\right ].
\end{equation*}
By multiplying $2\eta \left ( \frac{1-\eta \alpha }{ 1- 2 \gamma_{s, s^* }}\right )^{K - k}$ to both sides of the above inequality and summing over $k=1,\cdots, K$, we obtain
\begin{equation*}
\begin{split}
& \sum_{k=1}^{K} 2\eta \left ( \frac{1-\eta \alpha }{ 1- 2 \gamma_{s, s^*}}\right )^{K -k} \left ( \cL_t  (\Theta_{t,k}) -  \cL_t  ( \Theta^* ) \right )
\\
&  \leq  \sum_{k=1}^K \left ( \frac{1-\eta \alpha}{ 1- 2 \gamma_{s, s^* }}\right )^{K -k}  \left [ (1-\eta \alpha ) \| \Theta_{t,k-1} - \Theta^* \|_2^2 - (1- 2\gamma_{s, s^* }) \|  \Theta_{t,k} -   \Theta^* \|_2^2 \right ]
\\
& \leq (1-2\gamma_{s,s^* }) \left ( \frac{1-\eta \alpha}{ 1- 2 \gamma_{s, s^* }}\right )^{K }  \| \Theta_{t,0} -  \Theta^*   \|_2^2 .
\end{split}
\end{equation*}
By using the descent lemma, aka Lemma~\ref{lemma:sdl},  under the step-size that $\eta_{t,k} = \eta \leq  \frac{1}{L}$, we have $  \cL_t  (\Theta_{t,k+1}) \leq   \cL_t (\Theta_{t,k})$ for all $k = 0,\cdots, K-1$. Consequently, the above inequality implies that
\begin{equation*}
\begin{split}
& \sum_{k=1}^{K} \left ( \frac{1-\eta \alpha }{ 1- 2\gamma_{s, s^*}}\right )^{K - k} \left ( \cL_t  ( \Theta_{t,K}) -  \cL_t  ( \Theta^* ) \right )
\\
& \leq \frac{ 1-2\gamma_{s,s^* }} {2\eta} \left ( \frac{1-\eta \alpha }{ 1- 2 \gamma_{s, s^* }}\right )^{K }  \| \Theta_{t,0} -  \Theta^*   \|_2^2 .
\end{split}
\end{equation*}
Together with the fact that
$$
\sum_{ k =1}^{K}  \left ( \frac{1-\eta \alpha }{ 1- 2\gamma_{s, s^* }}\right )^{K -k} \geq 1,
$$
we obtain
\begin{equation}\label{eq:multiple_1}
\cL_t (\Theta_{t,K}) - \cL_t ( \Theta^*  ) \leq \frac{1- 2\gamma_{s, s^* } }{2\eta}  \left ( \frac{1-\eta \alpha }{ 1- 2 \gamma_{s, s^* }}\right )^{K } \| \Theta_{t,0} -  \Theta^* \|_2^2.
\end{equation}
Note that the above inequality holds regardless of $\text{sgn}(\cL_t (\Theta_{t,K}) - \cL_t ( \Theta^*  ) )$.
Recall that $\cS_{\Theta_{t,K}}$ and $\cS^*$ are the collection of non-zero entries of $\Theta_{t,K}$ and $\Theta^*$, respectively. By using the $(s,s)$-SSC condition of $\cL_t$, we further have
\begin{equation}\label{eq:multiple_2}
\begin{split}
 \frac{\alpha}{2} \|  \Theta_{t,K}-   \Theta^* \|_2^2
 & \leq \cL_t ( \Theta_{t,K} )  - \cL_t ( \Theta^*  )  - \lv \nabla \cL_t ( \Theta^*   ),  \Theta_{t,K} -  \Theta^* \rv
\\
&
=  \cL_t ( \Theta_{t,K}  )  - \cL_t (  \Theta^*  )   +  \frac{1}{t_0 + t}    \sum_{i\in \cI_t} \lv   X_i \epsilon_i  ,   \Theta_{t,K} -  \Theta^*   \rv.
\end{split}
\end{equation}
By using the fact that $\| \Theta_{t,K} -  \Theta^*  \|_0 \leq s + s^* \leq 2s$, we further have
\begin{equation*}
\begin{split}
 \frac{\alpha}{2} \|  \Theta_{t,K}-  \Theta^* \|_2^2
& \leq  \cL_t ( \Theta_{t,K}  )  - \cL_t (  \Theta^*  )   +  \left \| \frac{\sum_{i\in \cI_t}  X_i \epsilon_i }{t_0 + t}   \right \|_\infty   \|    \Theta_{t,K} -  \Theta^*  \|_1
\\
& \leq  \cL_t ( \Theta_{t,K}  )  - \cL_t (  \Theta^*  )   + \sqrt{2s} \left \| \frac{\sum_{i\in \cI_t}  X_i \epsilon_i }{t_0 + t}   \right \|_\infty   \|    \Theta_{t,K} -  \Theta^*  \|_2
\\
& \leq  \cL_t (  \Theta_{t,K}  )  - \cL_t ( \Theta^*  )   + \frac{2s}{\alpha} \left \| \frac{\sum_{i\in \cI_t}  X_i \epsilon_i }{t_0 + t}   \right \|_\infty^2    + \frac{\alpha}{4}\| \Theta_{t,K} -  \Theta^* \|_2^2  ,
\end{split}
\end{equation*}
which further implies that
\begin{equation} \label{eq:lower_bound}
\begin{split}
\frac{\alpha}{4} \|  \Theta_{t,K} - \Theta^*  \|_2^2 & \leq  \cL_t ( \Theta_{t,K}   )  - \cL_t ( \Theta^*  )   +  \frac{2s}{\alpha} \left \| \frac{\sum_{i\in \cI_t}  X_i \epsilon_i }{t_0 + t}   \right \|_\infty^2   .
\end{split}
\end{equation}
By substituting the above inequality into \eqref{eq:multiple_1}, we conclude that
\begin{equation*}
\frac{\alpha}{4} \| \Theta_{t,K}  - \Theta^* \|_2^2  \leq \frac{1- 2\gamma_{s, s^* } }{2\eta}  \left ( \frac{1-\eta \alpha}{ 1- 2 \gamma_{s, s^* }}\right )^{K }  \| \Theta_{t,0} - \Theta^*  \|_2^2  + \frac{2s}{\alpha} \left \| \frac{\sum_{i\in \cI_t}  X_i \epsilon_i }{t_0 + t}   \right \|_\infty^2  .
\end{equation*}
We divide both sides  of the above inequality by $\frac{\alpha}{4}$ and acquire the desired result.
\\
\noindent
Further, by using Lemma \ref{lemma:relative_concavity_sparse} in Appendix Section~\ref{sec:tech_lemma} that $\gamma_{s,s^*} = \sqrt{s^* / s}/ 2 $ when $\Upsilon(\Theta) = \| \Theta\|_0$ and recalling the step-size $\eta_{t,k} = \eta \leq  1/L$ and sparsity level $s > \frac{s^*}{\eta^2 \alpha^2}$, we obtain
$$
2 \gamma_{s,s^*} = \sqrt{s^* / s} <  \eta \alpha,
$$
implying that
$$
\frac{1-\eta \alpha}{ 1- 2 \gamma_{s, s^* }} < 1.
$$
Skipping some algebra, we conclude that
$$
\frac{2(1- 2\gamma_{s, s^* } )}{\alpha \eta}  \left ( \frac{1-\eta \alpha}{ 1- 2 \gamma_{s, s^* }}\right )^{K }  < 1 \iff K \geq  \frac{\log \{2(1- 2\gamma_{s,s^*})/({ \eta \alpha})\}}{ \log \{(1 - 2\gamma_{s,s^*})/({ 1 - \eta \alpha})\}}.
$$
This completes  the proof.
\end{proof}


\subsection{Proof of Theorem~\ref{thm:multistep_sparse}}
\begin{proof}
Here we observe $\delta < 1$ according to \eqref{def:delta}. By taking expectations on both sides of \eqref{eq:multi_step_descent} and using Assumption \ref{assumption:noise} and the fact that $\Theta_{t,0} = \Theta_{t-1,K}$, we have
\begin{equation*}
\EE[  \| \Theta_{t,K}- \Theta^* \|_2^2 ] \leq \delta \EE[   \| \Theta_{t-1,K} - \Theta^*  \|_2^2 ] + \frac{8s\sigma_x^2 }{\alpha^2(t_0 + t)} .
\end{equation*}
By recursively applying the above inequality and using Lemma~\ref{lemma:tele_sum}, we have that
\begin{equation*}
\begin{split}
\EE[  \| \Theta_{t,K}- \Theta^* \|_2^2 ] & \leq \delta^{t+1}  \| \Theta_{0} - \Theta^*  \|_2^2  + \frac{8s\sigma_x^2 }{\alpha^2} \left ( \sum_{j=1}^{t+1} \frac{\delta^{t-j}}{t_0 + j}\right )
\\
& \leq  \delta^{t+1}  \| \Theta_{0} - \Theta^*  \|_2^2  + \frac{8s\sigma_x^2 }{\alpha^2 \log(1/\delta)} \left (  \frac{  \delta^{t}}{t_0 + 1}   +    \frac{3}{\delta (t+2 + t_0) } +  \frac{\delta^{t/2}  }{t_0 +1 }  \right ).
\end{split}
\end{equation*}
This concludes the proof.
\end{proof}


\subsection{Supporting Lemmas}\label{sec:tech_lemma}

\begin{proposition}[In-sample prediction error]\label{prop:insample_error}
Suppose $\cL_t$ satisfies Assumption \ref{assumption:rsc_2s} under the sparsity level $s \geq  s^*$, then for any $\Theta \in \RR^d$ such that $\| \Theta\|_0 \leq s$, we have
\\
\noindent (a)
\begin{equation}\label{eq:insample_error}
\alpha \| \Theta - \Theta^*\|_2^2 \leq \frac{1}{t_0 + t}  \sum_{i \in \cI_t}   \lv X_i, \Theta  - \Theta^* \rv^2  \leq L \| \Theta - \Theta^*\|_2^2.
\end{equation}
\\
\noindent (b)
\begin{equation}\label{eq:loss_to_MSE}
\begin{split}
 \frac{\alpha }{4} \| \Theta - \Theta^* \|_2^2 \leq  \cL_{t}(\Theta) - \cL_{t}(\hat \Theta_{t})    +
\frac{2s}{\alpha } \Big  \|  \sum_{j \in \cI_{t}}  \frac{\epsilon_j X_j}{ t_0 + t} \Big \|_\infty^2
.
\end{split}
\end{equation}
\end{proposition}
\begin{proof}
(a)
Suppose $\| \Theta \|_0 \leq s $ and $ \|  \Theta^* \|_0 \leq s$.
Under the $(s,s)$-SSS condition, we have
\begin{equation*}
\begin{split}
\cL_t(\Theta) - \cL_t(\Theta^*) & =  \frac{1}{2(t+ t_0)} \sum_{i \in \cI_t}  \Big  (  ( \lv X_i, \Theta \rv  - Y_i )^2 -  \big ( \lv X_i, \Theta^* \rv  - Y_i \big )^2 \Big ) \\
& = \frac{1}{2(t+ t_0)}\sum_{i \in \cI_t}  \Big  ( \big  ( \lv X_i, \Theta  - \Theta^* \rv  - \epsilon_i  \big )^2 - \epsilon_i^2  \Big )
\\
& = \frac{1}{2(t+ t_0)} \sum_{i \in \cI_t}  \Big  (  \lv X_i, \Theta  - \Theta^* \rv^2 - 2 \epsilon_i \lv X_i, \Theta  - \Theta^* \rv  \Big )
\\
&  \geq \left \langle \nabla \cL_t(\Theta^*) , \Theta - \Theta^* \right\rangle   + \frac{\alpha}{2}\left\|\Theta  -\Theta^*
\right\|_2^2.
\end{split}
\end{equation*}
Using the fact that $\nabla \cL_t(\Theta ^*) =  - \frac{1}{t_0+ t} \sum_{i \in \cI_t} X_i \epsilon_i $, the above inequality implies
\begin{equation*}
\frac{1}{t_0 + t} \sum_{i \in \cI_t}  \lv X_i, \Theta  - \Theta^* \rv^2  \geq  \alpha \left\|\Theta -\Theta^*
\right\|_2^2.
\end{equation*}
Similarly, under Assumption~\ref{assumption:rsc_2s} and using the facts that $\| \Theta\|_0 \leq s $ and  $\| \Theta^*\|_0 \leq s$, we have
\begin{align*}
\cL_t(\Theta) - \cL_t(\Theta^*) & = \frac{1}{2(t+ t_0)} \sum_{i \in \cI_t}  \Big  (  \lv X_i, \Theta  - \Theta^* \rv^2 - 2 \epsilon_i \lv X_i, \Theta  - \Theta^* \rv  \Big )
\\
& \leq \left \langle \nabla \cL_t(\Theta^*) , \Theta - \Theta^* \right\rangle  + \frac{L}{2}\left\|\Theta  -\Theta^*
\right\|_2^2.
\end{align*}
This implies
$$
\frac{1}{t_0 + t}  \sum_{i \in \cI_t}   \lv X_i, \Theta  - \Theta^* \rv^2  \leq L \| \Theta - \Theta^*\|_2^2,
$$
completing the proof of part (a).

(b) We observe that
\begin{equation*}
\begin{split}
& \cL_{t}(\Theta ) - \cL_{t}(\hat \Theta_{t})   \geq  \cL_{t} ( \Theta ) - \cL_{t} ( \Theta^*)
\\
& = \frac{1}{2(t + t_0 )}   \sum_{i \in \cI_{t}} \Big ( \Big ( \langle X_{j},  \Theta \rangle - Y_{j} \Big )^2 -  ( \langle X_{j},  \Theta_{j}^* \rangle - Y_{j} )^2
 \Big )
 \\
 & = \frac{1}{2(t + t_0 )}  \sum_{i \in \cI_{t}}  \Big (  \Big  ( \lv X_{j},  \Theta - \Theta^* \rv - \epsilon_{j} \Big )^2   -    \epsilon_j^2
 \Big )
 \\
 & =  \frac{1}{2(t + t_0  )}   \sum_{i \in \cI_{t}}  \Big ( \lv X_{j},  \Theta - \Theta^* \rv^2  -     2 \epsilon_j \lv X_{j},  \Theta - \Theta^* \rv
 \Big )
 \\
 & \geq  \frac{\alpha }{2} \| \Theta - \Theta^* \|_2^2
 -      \sum_{j\in \cI_t } \frac{\epsilon_j }{t + t_0  }  \lv X_{j},  \Theta - \Theta^* \rv
\\
& \geq  \frac{\alpha }{2} \| \Theta - \Theta^* \|_2^2
- \frac{2s}{ \alpha} \Big  \|  \sum_{j \in \cI_{t}}  \frac{\epsilon_j X_j}{ t_0 + t }    \Big \|_{\infty} -  \frac{ \alpha }{4}\| \Theta  - \Theta^*  \|_2^2
\end{split}
\end{equation*}
where the second last inequality uses part (a) and the last inequality uses the fact that $\| \Theta - \Theta^*\|_0 \leq 2s$ and
\begin{equation*}
\lv a, b\rv  \leq  \sqrt{2s} \| a \|_2 \| b\|_{\infty} \leq \frac{\alpha}{4} \| a\|_2^2  + \frac{2s}{\alpha} \| b\|_{\infty}^2
\end{equation*}
when $\| a \|_0 \leq 2s$. As a result, we conclude that
$$
 \frac{\alpha }{4} \| \Theta - \Theta^* \|_2^2 \leq  \cL_{t}(\Theta) - \cL_{t}(\hat \Theta_{t})    +
\frac{2s}{\alpha } \Big  \|  \sum_{j \in \cI_{t}}  \frac{\epsilon_j X_j}{ t_0 + t} \Big \|_\infty^2,
$$
completing the proof.
\end{proof}


\begin{lemma} \cite[Lemma 4.1]{liu2020between} \label{lemma:relative_concavity_sparse}
When $\Upsilon(\Theta) = \| \Theta\|_0$ and a $s$-sparse hard thresholding operator $\Pi_{\cC_s}$, for any sparsity pair $(s,s^*)$ where $0 < s^* \leq s$, the relative concavity to sparsity level $s_0$ is
    $$
    \gamma_{s,s^*} = \sup \left \{  \frac{\langle Y - \Pi_{\cC_s}(Z), Z - \Pi_{\cC_s}(Z)\rangle }{ \| Y - \Pi_{\cC_s}(Z)\|_2^2}, Y, Z \in \RR^d, \| Y \|_0 \leq s^*, Y \neq \Pi_{\cC_s} (Z)\right \} =
    \frac{\sqrt{ s^*/s} }{2}.$$
\end{lemma}

\begin{lemma}\label{lemma:tele_sum}
For any $\gamma \in (0,1)$, the following holds
$$
\sum_{t=1}^T \frac{\gamma^{T-t}}{t + t_0}  \leq    \frac{  \gamma^{T-1}}{(t_0 + 1)}   +    \frac{3}{\gamma (T+1 + t_0) } +  \frac{ \gamma^{(T+1)/2-1}  }{1+ t_0 } .
$$
\end{lemma}
\begin{proof}
Our proof consists of two steps. In Step 1, we provide an upper bound for the  term  $\sum_{t=1}^T \frac{\gamma^{T-t}}{t + t_0}$  such that 
$$
\sum_{t=1}^T \frac{\gamma^{T-t}}{t_0 + t}  \leq    \frac{  \gamma^{T-1}}{(t_0 + 1)\log( 1/\gamma)} +  \gamma^T \int_{1 }^{T+1}  \frac{1}{( x + t_0 ) \gamma^x}dx.
$$
In Step 2, we further provide a bound on the above integration and acquire the desired result.
\\
\noindent \textbf{Step 1:} We note that $\sum_{t=1}^T \frac{\gamma^{T-t}}{t_0 + t}  = \gamma^T \sum_{t=1}^T \frac{1}{(t_0 + t) \gamma^t}$. We denote by $f(t) := \log ( \frac{1}{ (t_0 + t ) \gamma^t}) =  - \log (t_0 + t ) - t \log \gamma$. By utilizing the monotonicity of $\log(\cdot)$, we observe that $f(t)$ and $\frac{1}{(t_0 + t) \gamma^t}$ share the same minimizer.  By setting $f'(t)$ to be zero, we have
$$
f'(\hat t) = -\frac{1}{ t_0 + \hat t}  -  \log \gamma  = 0 \iff \hat t = \left ( \log  ( 1/ \gamma ) \right )^{-1} - t_0.
$$
We can see that $f(t) $ is decreasing for $t \leq \hat t $ and increasing for $t \geq \hat t$. We consider two scenarios.
\begin{itemize}
\item
\textbf{Scenario 1:} Suppose $ \hat t = \left ( \log ( 1/ \gamma ) \right )^{-1} - t_0    \leq 1$. Then we observe that $f(t)$ is monotonically increasing over $[1,\infty)$. In this case, clearly, we have
\begin{equation*}
\sum_{t=1}^T \frac{\gamma^{T-t}}{t_0 + t} \leq  \gamma^T \int_{1}^{T+1} \frac{1}{(t_0 + x )\gamma^x}dx.
\end{equation*}

\item \textbf{Scenario 2:}
Suppose $ \hat t = \left ( \log ( 1/ \gamma ) \right )^{-1} - t_0  >1$.  We observe that $f(t)$ is decreasing within $t \in [1, \lfloor \hat t \rfloor ] $ and increasing within $[\lceil \hat t \rceil , \infty )$. By utilizing this observation, we have
\begin{equation*}
\begin{split}
\sum_{t=1}^T \frac{\gamma^{T-t}}{t_0 + t} & \leq  \gamma^T \left [ \sum_{t=1}^{\lfloor \hat t \rfloor} \frac{1}{( t_0 + t) \gamma^t} + \int_{\lceil \hat t \rceil }^{T+1}  \frac{1}{ ( x+ t_0 ) \gamma^x}dx \right ]
\\
&
 \leq  \gamma^T \left [  \frac{ \lfloor \hat t \rfloor }{ (t_0 + 1)\gamma} + \int_{\lceil \hat t \rceil }^{T+1}  \frac{1}{( x + t_0 ) \gamma^x}dx \right ]
\\
& \leq \gamma^T \left [  \left ( \frac{1}{\log ( 1/ \gamma ) }- t_0  \right ) \frac{1}{(t_0 + 1) \gamma}
+ \int_{1 }^{T+1}  \frac{1}{ ( x + t_0  )\gamma^x}dx \right ]
 ,
\end{split}
\end{equation*}
where the second inequality comes from the fact that $f(t)$ is decreasing for $t \in [1, \lfloor \hat t \rfloor]$ so that $\frac{1}{( t + t_0 ) \gamma^t} \leq \frac{1}{(t_0 + 1)\gamma}$.
\end{itemize}
By combining the above two scenarios, we conclude that
\begin{equation}\label{eq:sum_bound}
\begin{split}
\sum_{t=1}^T \frac{\gamma^{T-t}}{t_0 + t} & \leq    \frac{  \gamma^{T-1}}{(t_0 + 1)\log( 1/\gamma)} +  \gamma^T \int_{1 }^{T+1}  \frac{1}{( x + t_0 ) \gamma^x}dx.
\end{split}
\end{equation}
\textbf{Step 2:}
By setting $c = \log(1/\gamma) > 0$, we have $\exp(c) = \frac{1}{\gamma}$ and obtain the following.
\begin{equation*}
\begin{split}
& \gamma^{T+1} \int_1^{T+1} \frac{1}{( x + t_0 ) \gamma^x} dx
\\
& =
\gamma^{T+1} \int_1^{T+1}  (x + t_0 )^{-1} \exp(cx) dx
\\
& =  \gamma^{T+1} \Big ( \frac{  \exp(cx) }{c(x + t_0 )}\  \Big |_{x = 1}^{ x = T+1}  + \int_1^{T+1} \frac{e^{cx}}{c(x + t_0 )^2} dx \Big )
\\
&
= \frac{1}{c(T+1 + t_0)} - \frac{ \gamma^{T+1} e^c }{ c(t_0 + 1)} + \gamma^{T+1} \int_1^{T+1} \frac{e^{cx}}{c( x + t_0 )^2} dx
\\
&
\leq \frac{1}{c(T+1 + t_0)} - \frac{ \gamma^{T}  }{c (t_0 + 1)}+ \underbrace{ \frac{ \gamma^{T+1} }{c}\int_{\frac{T+1}{2}}^{T+1} \frac{e^{cx}}{(x + t_0 )^2} dx }_{\Rom{1}} +
 \underbrace{  \frac{ \gamma^{T+1}  }{c}\int_{1}^{\frac{T+1}{2}} \frac{e^{cx}}{( x + t_0)^2} dx }_{\Rom{2}}.
\end{split}
\end{equation*}
For the last two terms,  by using the fact that $e^{c(T+1)} = \frac{1}{\gamma^{T+1}}$, we can see that
\begin{equation*}
\begin{split}
\Rom{1}   \leq \frac{ \gamma^{T+1} }{c} \int_{\frac{T+1}{2}}^{T+1} \frac{e^{c(T+1)}}{(x + t_0 )^2} dx   \leq  \int_{\frac{T+1}{2}}^{T+1} \frac{1}{c( x + t_0 )^2} dx   = -\frac{1}{c(x + t_0 )} \ \Big |_{x = \frac{T+1}{2}}^{ x = T+1} \leq  \frac{2}{c( T+1 +2 t_0 )},
\end{split}
\end{equation*}
and
\begin{equation*}
\begin{split}
\Rom{2} & \leq \frac{  \gamma^{T+1}  }{c}\int_{1}^{\frac{T+1}{2}} \frac{e^{cx}}{(x + t_0 )^2} dx
\\
&
 \leq  \frac{ \gamma^{T+1}}{c}  \int_{1}^{\frac{T+1}{2}}  \frac{\gamma^{-(T+1)/2}}{(x + t_0 )^2} dx   = \frac{  \gamma^{(T+1)/2}  }{c} \int_{1}^{\frac{T+1}{2}}  \frac{1}{(x + t_0 )^2 } dx
\\
& = - \frac{ \gamma^{(T+1)/2}  }{c ( x + t_0) } \ \Big |_{x = 1}^{x = \frac{T+1}{2}}   = - \frac{ \gamma^{(T+1)/2} }{c} \Big ( \frac{2}{T+1 + 2t_0} - \frac{1}{1+ t_0} \Big )
\\
&
=  \frac{  \gamma^{(T+1)/2}  }{c}\left  ( \frac{1}{1+ t_0}  - \frac{2}{T+1 + 2 t_0} \right ),
\end{split}
\end{equation*}
where the second inequality uses the fact that $e^{cx} = \gamma^{-x} \leq \gamma^{-(T+1)/2}$ for $x \in [1, \frac{T+1}{2}]$.
By combining the above terms, we obtain
\begin{equation*}
\begin{split}
&
\gamma^{T+1} \int_1^{T+1} \frac{1}{( x + t_0 ) \gamma^x} dx
\\
 &
 \leq
 \frac{1}{c(T+1 + t_0)} - \frac{ \gamma^{T}  }{c (t_0 + 1)} +  \frac{2}{c( T+1 +2 t_0 )}  +  \frac{ \gamma^{(T+1)/2}  }{c}\left  ( \frac{1}{1+ t_0}  - \frac{2}{T+1 + 2 t_0} \right )
 \\
& \leq
 \frac{3}{c(T+1 + t_0)} +  \frac{ \gamma^{(T+1)/2}  }{c (1+ t_0)}
.
\end{split}
\end{equation*}
By substituting the above inequality into \eqref{eq:sum_bound},
we obtain
\begin{equation*}
\begin{split}
\sum_{t=1}^T \frac{\gamma^{T-t}}{t + t_0} & \leq   \frac{1}{\log(1/\gamma)} \left ( \frac{  \gamma^{T-1}}{(t_0 + 1)}   +    \frac{3}{\gamma (T+1 + t_0) } +  \frac{ \gamma^{(T+1)/2-1}  }{1+ t_0 }  \right )
 ,
\end{split}
\end{equation*}
completing the proof.
\end{proof}


\section{Proofs for Section~\ref{sec:multistep_lowrank}}

\label{app:2}


\subsection{Proof of Lemma~\ref{lemma:stats_error_lowrank}}
\begin{proof}
Recall that $\hat \Theta_t = \arg \min_{\rank(\Theta)\leq s^*} \cL_t(\Theta)$ and $\rank(\hat \Theta_t - \Theta^*) \leq 2s^*$. Suppose $\cL_t({\cdot})$ satisfies the $2s$-LowRankSC condition.
By following the analysis of Lemma~\ref{lemma:stats_error} for sparse linear regression, we note that \eqref{eq:insample} still holds for low-rank matrix sensing. Specifically, we have Proposition~\ref{prop:insample_error_lowrank} (a) that
\begin{equation*}
\frac{\alpha }{2}\| \hat \Theta_t  - \Theta^*\|_\rF^2
 \leq  \frac{1}{2( t_0 +t)}   \sum_{i \in \cI_t} \lv  X_i , \hat \Theta_t - \Theta^*\rv^2
\leq \frac{1}{t_0 +t}\sum_{i \in \cI_t}  \epsilon_i  \lv  X_i ,  \hat \Theta_t  - \Theta^* \rv.
\end{equation*}
For any matrices $A,B \in \RR^{d_1 \times d_2}$ where $\rank(A) \leq 2s$, we can see that
\begin{equation}\label{eq:sparse_norm_lowrank}
\lv A, B\rv  \leq  \| A \|_{\rs_1} \| B\|_{\rs_\infty } \leq \sqrt{2s} \| A \|_{\rF} \| B\|_{2} \leq \frac{\alpha}{4} \| A\|_{\rF}^2  + \frac{2s}{\alpha} \| B\|_{2}^2,
\end{equation}
where $\|A\|_{\rs_p}$ is the Schatten $p$-norm of $A$ and the second inequality uses $\| A\|_{\rs_1} \leq \sqrt{\rank(A)} \| A\|_\rF$ for any matrix $A$ and the last inequality uses the fact that $a b \leq \frac{a^2}{\alpha}  + \frac{\alpha b^2}{4}$.
 Combining the above inequalities and using the fact that $\rank(\hat \Theta_t - \Theta^*) \leq 2s^* $, we have
\begin{equation*}
\begin{split}
\frac{\alpha }{2}\| \hat \Theta_t  - \Theta^*\|_\rF^2
& \leq \frac{1}{t_0 +t}\sum_{i \in \cI_t}  \epsilon_i  \lv  X_i , \hat \Theta_t  - \Theta^*\rv
\\
& \leq  \Big \|   \frac{1}{t_0 +t}\sum_{i \in \cI_t}    \epsilon_i X_i  \Big \|_{\rs_\infty} \|\hat \Theta_t  - \Theta^* \|_{\rs_1}
\\
& \leq \sqrt{2s^*}   \Big \|   \frac{1}{t_0 +t}\sum_{i \in \cI_t}    \epsilon_i X_i  \Big \|_{2} \|\hat \Theta_t  - \Theta^* \|_{\rF}
\\
& \leq \frac{2s^*}{\alpha }  \Big \|   \frac{1}{t_0 +t}\sum_{i \in \cI_t}    \epsilon_i X_i  \Big \|_{2}^2 + \frac{\alpha }{4} \|\hat \Theta_t  - \Theta^* \|_{\rF}^2.
\end{split}
\end{equation*}
Rearranging the terms, we have
\begin{equation*}
\begin{split}
  \|  \hat \Theta_t -  \Theta^* \|_\rF^2   &   \leq \frac{8s^*}{\alpha^2 }   \Big \|  \sum_{i \in \cI_t}  \frac{ \epsilon_i X_i}{t_0 +t}     \Big \|_{2}^2  .
\end{split}
\end{equation*}
By using Assumption \ref{assumption:noise_lowrank}  that $\EE \big [   \big  \|  \sum_{i \in \cI_t} \epsilon_i     X_i  \big  \|_{2}^2    \big ] \leq (t_0 + t)\sigma_X^2$ and taking expectations on both sides of the above inequality, we conclude that
 \begin{equation*}
\EE[   \|  \hat \Theta_t -  \Theta^* \|_\rF^2  ]    \leq \frac{8s^*}{\alpha^2 }   \EE \Big [ \Big  \|  \sum_{i \in \cI_t} \frac{ \epsilon_i    X_i }{t_0 +t}  \Big  \|_{2}^2  \Big ]   \leq  \frac{8s^*\sigma_X^2 }{\alpha^2( t_0 + t)}  .
\end{equation*}
This completes the proof.
\end{proof}

\subsection{Proof of Lemma \ref{lemma:descent_lowrank}} \label{app:proof_of_lemma_descent_lowrank}
\begin{proof}
We denote by $\Theta_{t,k}$ and $\Theta_{t,k+1}$ two consecutive rank-$s$ solutions generated by Algorithm~\ref{alg:1} with projection operator $\Pi_{\cC_s}(\Theta)$ returning the rank-$s$ approximation of $\Theta$ and denote by $\cL_t(\cdot)$ the corresponding loss function.  By Assumption~\ref{assumption:rsc_2s}, $\cL_t(\cdot)$ satisfies the $2s$-LowRankSS  condition~\eqref{def:smooth} with parameter $L$. For notational convenience, we write $\eta_{t,k} = \eta$  and rewrite the  update rule as
\$
\Theta_{t,k+1}
=  \argmin_{Z \in \RR^{d_1\times d_2}, \text{rank}(Z) \leq s } \left\| Z - (\Theta_{t,k} - \eta \nabla \cL_t(\Theta_{t,k}))  \right\|_{\rF}^2.
\$
For simplicity, we denote by $S_{t,k} $ and $S_{t,k+1} $ the column spaces of $\Theta_{t,k}$ and $\Theta_{t,k+1}$, respectively. Let $g_{k}=\nabla \cL( \Theta_{t,k})$.  We then bound the change of the objective function value from $\cL_t(\Theta_{t,k})$ to $\cL_t(\Theta_{t,k+1})$  as
\#
\cL_t (\Theta_{t,k+1})-\cL_t(\Theta_{t,k})
&\leq \left\langle \nabla  \cL_t( \Theta_{t,k+1}), \Theta_{t,k+1}- \Theta_{t,k}\right\rangle_{\rF}  +\frac{L }{2}\left\| \Theta_{t,k+1}- \Theta_{t,k}\right\|_{\rF}^2 \nn\\
&= \left\langle g_k, \Theta_{t,k+1}- \Theta_{t,k}\right\rangle_{\rF}  +\frac{L }{2}\left\| \Theta_{t,k+1}- \Theta_{t,k}\right\|_{\rF}^2\nn \\
& = \Rom{1}+\Rom{2},  \label{eq:rank:1}
\#
where $\Rom{1} = \left\langle g_k, \Theta_{t,k+1}- \Theta_{t,k}\right\rangle_{\rF}$ and $\Rom{2} =\frac{L }{2}\left\| \Theta_{t,k+1}- \Theta_{t,k}\right\|_{\rF}^2 $.
We proceed to bound \Rom{1} and \Rom{2} respectively.  We start with \Rom{1} and have
\begin{equation*}
\begin{split}
\Rom{1}
& =  \left\langle P_{ S_{k}+ S_{k+1}} g_k, P_{ S_k +  S_{k+1}}( \Theta_{t,k+1}- \Theta_{t,k} )\right\rangle_{\rF}
\\
&= -\lv P_{S_{t,k} \cap S_{t,k+1}^\perp} \Theta_{t,k}, P_{S_t\cap S_{t,k+1}^\perp } g_t \rv  +\lv P_{S_{t,k+1}} (\Theta_{t,k+1}  -\Theta_{t,k}), P_{S_{t,k+1}}  g_k \rv
\\
&= -\lv P_{S_{t,k}\cap S_{t,k+1}^\perp } \Theta_{t,k}, P_{S_k\cap S_{t,k+1}^\perp } g_k \rv +\lv P_{S_{t,k+1}} (P_{S_{t,k+1}} (\Theta_{t,k} - \eta g_k)   -\Theta_{t,k}), P_{S_{t,k+1}}  g_t \rv
\\
&= -\lv P_{S_{t,k}\cap S_{t,k+1}^\perp }  \Theta_{t,k}, P_{S_{t,k}\cap S_{t,k+1}^\perp }  g_k \rv  -\eta\| P_{S_{t,k+1}}  g_k  \|_{\rF}^2
\\
& =  -\lv  P_{S_{t,k}\cap S_{t,k+1}^\perp } \Theta_{t,k} - \eta P_{S_{t,k}\cap S_{t,k+1}^\perp } g_k  ,  P_{S_{t,k}\cap S_{t,k+1}^\perp } g_k \rv  - \eta \| P_{S_{t,k} \cap  S_{t,k+1}^\perp }  g_k   \|_{\rF}^2  -\eta\| P_{S_{t,k+1}}  g_k  \|_{\rF}^2
\\
& \leq \frac{1}{2\eta} \|  P_{S_{t,k}\cap S_{t,k+1}^\perp }   \Theta_{t,k}- \eta P_{S_{t,k} \cap  S_{t,k+1}^\perp }  g_k  \|_{\rF}^2 - \frac{\eta}{2} \| P_{S_{t,k} \cap  S_{t,k+1}^\perp}  g_k   \|_{\rF}^2  -\eta\| P_{S_{t,k+1}}  g_k  \|_{\rF}^2
\\
&\leq \frac{\eta}{2}\| P_{S_{t,k+1}\cap  S_{t,k}^\perp } g_k\|_{\rF}^2-\frac{\eta}{2}\| P_{S_{t,k} \cap  S_{t,k+1}^\perp }  g_k   \|_{\rF}^2-\eta\|P_{S_{t,k+1}}  g_k  \|_{\rF}^2\\
&\leq -\frac{\eta}{2}\| P_{S_{t,k} + S_{t,k+1}} g_k \|_{\rF}^2,
\end{split}
\end{equation*}
where the third equality comes from $\Theta_{t,k+1} = P_{S_{t,k+1}} (\Theta_{t,k}  - \eta g_k)$, and the last  second inequality follows from the fact
\begin{align*}
\| P_{S_{t,k}\cap S_{t,k+1}^\perp } ( \Theta_{t,k} -\eta g_{k} )\|_{\rF}^2\leq \| P_{S_{t,k+1}\cap  S_{t,k}^\perp }  \Theta_{t,k+1}\|_{\rF}^2=\eta^2\|   P_{S_{t,k+1}\cap  S_{t,k}^\perp }  g_k\|_{\rF}^2.
\end{align*}
For \Rom{2}, by using the fact that $\| a\|_{\rF}^2 = \| a+b\|_{\rF}^2 - \| b \|_{\rF}^2 - 2 \lv a, b \rv$, we note that
\begin{equation*}
\begin{split}
\Theta_{t,k+1} & = \argmin_{Z: \text{rank}(Z) \leq r} \left \{ \|  Z - \Theta_{t,k} + \eta g_k \|_{\rF}^2   \right \}
\\
& =  \argmin \Big \{ \|  Z - \Theta_{t,k} + \eta g_k  \|_{\rF}^2  \ \Big | \  \text{rank}(Z) \leq r, Z \in  S_{t,k} + S_{t,k+1}
\Big \}
\\
& =  \argmin_{Z: \text{rank}(Z) \leq r} \| P_{S_{t,k} + S_{t,k+1}}( Z - \Theta_{t,k} + \eta g_k  ) \|_{\rF}^2
\leq  \|  \eta  P_{S_{t,k} + S_{t,k+1}} g_k  \|_{\rF}^2,
\end{split}
\end{equation*}
which yields that
\$
\Rom{2}&=\frac{L}{2}\left\| \Theta_{t,k+1}- \Theta_{t,k}\right\|_{\rF}^2\\
&= \frac{L}{2}\left\| \Theta_{t,k+1}- \Theta_{t,k}+\eta P_{S_{t,k} + S_{t,k+1}} g_k \right\|_{\rF}^2
-\frac{L}{2} \left\|\eta P_{S_{t,k} + S_{t,k+1}} g_k \right\|_{\rF}^2
\\
& \quad -L\eta\cdot  \left\langle P_{S_{t,k} + S_{t,k+1}} g_k  , \Theta_{t,k+1}-\Theta_{t,k}\right\rangle  \\
&= \frac{L }{2}\inf_{Z: \text{rank}(Z) \leq r }\left\|Z- \Theta_{t,k}+\eta P_{S_{t,k} + S_{t,k+1}} g_t \right\|_{\rF}^2-\frac{L }{2} \left\|\eta P_{S_{t,k} + S_{t,k+1}} g_k \right\|_{\rF}^2-L  \eta\cdot \Rom{1}\\
&\leq -L  \eta\cdot \Rom{1} .
\$
Plugging the above bounds for \Rom{1} and \Rom{2} into \eqref{eq:rank:1}, we obtain that
\$
\cL_t(\Theta_{t,k+1})-\cL_t(\Theta_{t,k})
&\leq  (1-L\eta)\cdot \Rom{1} \leq -\frac{\eta (1-L\eta)}{2}\| P_{S_{t,k} + S_{t,k+1}} g_k\|_{\rF}^2,
\$
completing the proof.
\end{proof}

\subsection{Proof of Lemma~\ref{lemma:multi_step_descent_lowrank}}

\begin{proof}
(a)
We consider the case where we conduct $K$ gradient descent steps upon receiving the $t$-th data.
Recall the descent Lemma~\ref{lemma:descent_lowrank}, when  $\eta_{t,k} = \eta \leq 1/L$, we have
\$
\cL_t (\Theta_{t,k+1}) - \cL_t (\Theta_{t,k})\leq   -\frac{\eta  (1-L\eta)}{2}\|P_{\cS_{\Theta_{t,k}}+ \cS_{\Theta_{t,k+1}}} \nabla \cL_t (\Theta_{t,k}) \|_{\rF}^2 \leq 0.
\$
Using the $2s$-LowRank SC and SS conditions and the fact that $\rank(\Theta_{t,k-1}),\rank(\Theta_{t,k})$, $\rank(\Theta^*) \leq s$, we see that \eqref{eq:multiple_1} and \eqref{eq:multiple_2} still hold for low-rank matrix sensing. Specifically, for low-rank matrix regression, we rewrite \eqref{eq:multiple_1} as
\begin{equation}\label{eq:multiple_lowrank_1}
\cL_t (\Theta_{t,K}) - \cL_t ( \Theta^*  ) \leq \frac{1- 2\gamma_{s, s^* } }{2\eta}  \left ( \frac{1-\eta \alpha }{ 1- 2 \gamma_{s, s^* }}\right )^{K } \| \Theta_{t,0} -  \Theta^* \|_\rF^2,
\end{equation}
where \begin{equation*}
\gamma_{s, s^*}= \sup \left \{  \frac{ \lv Z - \Pi_{\cC_s} (Z)  , Y - \Pi_{\cC_s}(Z) \rv  }{ \| Y - \Pi_{\cC_s} (Z)\|_\rF^2}, Y, Z \in \RR^{d_1 \times d_2}, \rank( Y) \leq s^*, Y \neq \Pi_{\cC_s}(Z)\right \}
\end{equation*}
for any sparsity pair $(s^*,s)$. As in \eqref{eq:multiple_2}, using \eqref{eq:sparse_norm_lowrank} and the fact that $\rank(\Theta_{t,K} - \Theta^*) \leq 2s$, we have
\begin{equation*}
\begin{split}
\frac{\alpha}{2} \|  \Theta_{t,K}-   \Theta^* \|_{\rF}^2 & \leq  \cL_t ( \Theta_{t,K}  )  - \cL_t (  \Theta^*  )   +  \frac{1}{t_0 + t}  \sum_{i\in \cI_t}   \lv    \epsilon_i X_i  ,   \Theta_{t,K} -  \Theta^*   \rv
\\
& \leq  \cL_t ( \Theta_{t,K}  )  - \cL_t ( \Theta^*  )   + \frac{2s}{\alpha} \left \| \frac{
\sum_{i\in \cI_t} \epsilon_i  X_i
}{ t_0 + t} \right \|_{2}^2  + \frac{\alpha}{4}\| \Theta_{t,K} - \Theta^* \|_{\rF}^2  ,
\end{split}
\end{equation*}
which further implies that
\begin{equation}\label{eq:multiple_lowrank_2}
\begin{split}
\frac{\alpha}{4} \| \Theta_{t,K} - \Theta^*  \|_{\rF}^2 & \leq  \cL_t ( \Theta_{t,K}   )  - \cL_t ( \Theta^*  )   +  \frac{2s}{\alpha} \left \| \frac{
\sum_{i\in \cI_t} \epsilon_i  X_i
}{ t_0 + t} \right \|_{2}^2 .
\end{split}
\end{equation}
By combining \eqref{eq:multiple_lowrank_1} and \eqref{eq:multiple_lowrank_2}, we conclude that
\begin{equation*}
\frac{\alpha}{4} \| \Theta_{t,K}  - \Theta^* \|_{\rF}^2  \leq \frac{1- 2\gamma_{s, s^* } }{2\eta}  \left ( \frac{1-\eta \alpha}{ 1- 2 \gamma_{s, s^* }}\right )^{K }  \| \Theta_{t,0} - \Theta^*  \|_{\rF}^2  + \frac{2s}{\alpha} \left \| \frac{
\sum_{i\in \cI_t} \epsilon_i  X_i
}{ t_0 + t} \right \|_{2}^2  .
\end{equation*}
We divide both sides  of the above inequality by $\frac{\alpha}{4}$ and obtain that
\begin{equation*}
 \| \Theta_{t,K}  - \Theta^* \|_{\rF}^2  \leq \frac{2(1- 2\gamma_{s, s^* }) }{\alpha \eta}  \left ( \frac{1-\eta \alpha}{ 1- 2 \gamma_{s, s^* }}\right )^{K }  \| \Theta_{t,0} - \Theta^*  \|_{\rF}^2  + \frac{8s}{\alpha^2} \left \| \frac{
\sum_{i\in \cI_t} \epsilon_i  X_i
}{ t_0 + t} \right \|_{2}^2  .
\end{equation*}
(b) Further, by using Lemma \ref{lemma:relative_concavity_lowrank} that $\gamma_{s,s^*} = \sqrt{s^* / s}/2$ when $\Upsilon(\Theta) = \rank(\Theta)$ and recalling the step-size $\eta \leq  1/L$ and rank level $s > \frac{s^*}{\eta^2 \alpha^2}$, we
obtain the desired result by following the analysis of Lemma~\ref{lemma:multi_step_descent} (b).
\end{proof}


\subsection{Proof of Theorem~\ref{thm:multistep_lowrank}}
\begin{proof}
Here we observe $\tilde \delta < 1$ by recalling \eqref{def:delta}. By applying Lemma~\ref{app:proof_lemma_multi_step_desecnt}~(a), taking expectations on both sides, and using Assumption \ref{assumption:noise_lowrank} and the fact that $\Theta_{t,0} = \Theta_{t-1,K}$, we have
\begin{equation*}
\EE[  \| \Theta_{t,K}- \Theta^* \|_{\rF}^2 ] \leq \tilde \delta\, \EE[   \| \Theta_{t-1,K} - \Theta^*  \|_{\rF}^2 ] + \frac{8s\sigma_X^2 }{\alpha^2(t_0 + t)} .
\end{equation*}
By recursively applying the above inequality and using Lemma~\ref{sec:tech_lemma}, we have that
\begin{equation*}
\begin{split}
\EE[  \| \Theta_{t,K}- \Theta^* \|_{\rF}^2 ] & \leq \tilde \delta^{t+1}  \| \Theta_{0} - \Theta^*  \|_{\rF}^2  + \frac{8s \sigma_X^2 }{\alpha^2} \left ( \sum_{j=1}^{t+1} \frac{\tilde \delta^{t-j}}{t_0 + j}\right )
\\
& \leq \tilde  \delta^{t+1}  \| \Theta_{0} - \Theta^*  \|_{\rF}^2  + \frac{8s \sigma_x^2 }{\alpha^2 \log (1/\tilde \delta)} \left (  \frac{ \tilde \delta^{t}}{t_0 + 1}   +    \frac{3}{\tilde \delta (t+2 + t_0) } +  \frac{\tilde \delta^{t/2}  }{t_0 +1 }  \right ).
\end{split}
\end{equation*}
This concludes the proof.
\end{proof}

\subsection{Supporting Lemmas}\label{app:support_lemma_lowrank}

\begin{proposition}\label{prop:insample_error_lowrank}
Suppose $\cL_t$ satisfies Assumption \ref{assumption:rsc_lowrank_2s} under the rank level $s \geq  s^*$, then for any $\Theta \in  \RR^{d_1 \times d_2 }$ such that $\rank(\Theta) \leq s$, we have
\\
\noindent (a)
\begin{equation}\label{eq:insample_lowrank}
\alpha \| \Theta - \Theta^*\|_\rF^2 \leq \frac{1}{t_0 + t}  \sum_{i \in \cI_t}   \lv X_i, \Theta  - \Theta^* \rv^2  \leq L \| \Theta - \Theta^*\|_\rF^2.
\end{equation}
\\
\noindent (b)
\begin{equation}\label{eq:loss_to_MSE_lowrank}
    \frac{\alpha }{4} \| \Theta - \Theta^* \|_\rF^2 \leq  \cL_{t}(\Theta) - \cL_{t}(\hat \Theta_{t})    +
\frac{2s}{\alpha } \Big  \|  \sum_{j \in \cI_{t}}  \frac{\epsilon_j X_j}{ t_0 + t} \Big \|_2^2.
\end{equation}
\end{proposition}

\begin{proof}
(a) Clearly, we can see that $\rank(\Theta)\leq s $ and $\rank(\Theta^* ) \leq s^*$. Under Assumption~\ref{assumption:rsc_2s}, part~(a) can be acquired by using a similar analysis to that of  of Proposition~\ref{prop:insample_error}.
\\
\noindent
(b)  We observe that
\begin{equation*}
\begin{split}
& \cL_{t}(\Theta ) - \cL_{t}(\hat \Theta_{t})   \geq  \cL_{t} ( \Theta ) - \cL_{t} ( \Theta^*)
\\
& = \frac{1}{2(t + t_0 )}   \sum_{i \in \cI_{t}} \Big ( \Big ( \langle X_{j},  \Theta \rangle - Y_{j} \Big )^2 -  ( \langle X_{j},  \Theta_{j}^* \rangle - Y_{j} )^2
 \Big )
 \\
 & = \frac{1}{2(t + t_0 )}  \sum_{i \in \cI_{t}}  \Big (  \Big  ( \lv X_{j},  \Theta - \Theta^* \rv - \epsilon_{j} \Big )^2   -    \epsilon_j^2
 \Big )
 \\
 & =  \frac{1}{2(t + t_0  )}   \sum_{i \in \cI_{t}}  \Big ( \lv X_{j},  \Theta - \Theta^* \rv^2  -     2 \epsilon_j \lv X_{j},  \Theta - \Theta^* \rv
 \Big )
 \\
 & \geq  \frac{\alpha }{2} \| \Theta - \Theta^* \|_\rF^2
 -      \sum_{j\in \cI_t } \frac{\epsilon_j }{t + t_0  }  \lv X_{j},  \Theta - \Theta^* \rv
\\
& \geq  \frac{\alpha }{2} \| \Theta - \Theta^* \|_\rF^2
- \frac{2s}{ \alpha} \Big  \|  \sum_{j \in \cI_{t}}  \frac{\epsilon_j X_j}{ t_0 + t }    \Big \|_{2}^2  -  \frac{ \alpha }{4}\| \Theta  - \Theta^*  \|_\rF^2,
\end{split}
\end{equation*}
where the second last inequality uses part (a) and the last inequality uses the fact that $\rank(\Theta - \Theta^*) \leq 2s$ and applying \eqref{eq:sparse_norm_lowrank} that
\begin{equation*}
\lv A, B\rv  \leq  \| A \|_{\rs_1} \| B\|_{\rs_\infty } \leq \sqrt{2s} \| A \|_{\rF} \| B\|_{2} \leq \frac{\alpha}{4} \| A\|_{\rF}^2  + \frac{2s}{\alpha} \| B\|_{2}^2 ,
\end{equation*}
where $\|A\|_{\rs_p}$ is the Schatten $p$-norm of $A$ and the second inequality uses $\| A\|_{\rs_1} \leq \sqrt{\rank(A)} \| A\|_\rF$. As a result, we conclude that
$$
 \frac{\alpha }{4} \| \Theta - \Theta^* \|_\rF^2 \leq  \cL_{t}(\Theta) - \cL_{t}(\hat \Theta_{t})    +
\frac{2s}{\alpha } \Big  \|  \sum_{j \in \cI_{t}}  \frac{\epsilon_j X_j}{ t_0 + t} \Big \|_2^2,
$$
completing the proof.
\end{proof}

\begin{lemma} \cite[Lemma 5.2]{liu2020between} \label{lemma:relative_concavity_lowrank}
For any rank pair $(s,s^*)$, the relative concavity of the low-rank projection operator is
    $$
    \gamma_{s,s^*} = \sup \left \{  \frac{\langle Y - \Pi_{\cC_s}(Z), Z - \Pi_{\cC_s}(Z)\rangle }{ \| Y - \Pi_{\cC_s}(Z)\|^2}, Y, Z \in \RR^{d_1 \times d_2}, \rank(Y) \leq s^*, Y \neq \Pi_{\cC_s} (Z)\right \} =
    \frac{\sqrt{ s^*/s} }{2},$$
    where $0 < s^* \leq s$ and $\Pi_{\cC_s}(\Theta)$ is an operator that returns the rank-$s$ approximation of a matrix $\Theta \in \RR^{d_1 \times d_2}$.
\end{lemma}

\section{Proofs for Section~\ref{sec:extensions} }
\label{app:3}
\subsection{Proof of Proposition~\ref{prop:free_initial_sparse}}
\begin{proof}
We  assume that $\eta_{t,k}=\eta$.
We first analyze the MSE $\| \Theta_{t,k} -\Theta^* \|^2 $ for $t\leq t_0, 1\leq k \leq K$. Recall that $\Theta_{t,k} = \Pi_{\cC_s} \{  \Theta_{t,k-1} - \eta \nabla \cL_{t}(\Theta_{t,k-1})\}$. Let $\cS^* = \text{support}(\Theta^*)$, $\cS_{t,k} = \text{support}(\Theta_{t,k})$, and $g = \nabla \cL_t(\Theta_{t,k-1})$,
we have
\begin{equation} \label{eq:bound_theta}
\begin{split}
\| \Theta_{t,k} - \Theta^*  \|_2^2
&= \| \Pi_{\cC_s} \{  \Theta_{t,k-1} - \eta \nabla \cL_{t}(\Theta_{t,k-1})\} - \Theta^*\|_2^2\\
& = \|  (  \Theta_{t,k-1} - \eta g - \Theta^*)_{\cS_{t,k}} - \Theta^*_{\cS^* \setminus \cS_{t,k} }\|_2^2
\\
&\leq \|\Theta_{\cS_{t,k}}^{t,k-1}-\Theta_{\cS_{t,k} }^*\|_2^2 + \eta^2 \|g_{\cS_{t,k} }\|_2^2- 2\langle \Theta_{\cS_{t,k} }^{t,k-1} -\Theta^*_{\cS_{t,k} }, \eta g_{\cS_{t,k}} \rangle + \|\Theta^*_{\cS^*\setminus \cS_{t,k}} \|_2^2
\\
&
\leq
2\|\Theta_{\cS_{t,k}}^{t,k-1}-\Theta_{\cS_{t,k} }^*\|_2^2 + 2\eta^2 \|g_{\cS_{t,k} }\|_2^2  + \|\Theta^*_{\cS^*\setminus \cS_{t,k}}\|_2^2.
\end{split}
\end{equation}
We write $\tilde X_t = [X_1, X_2,\cdots, X_t]^\top  \in \RR^{t \times d}$. Taking expectation and using the fact
\$
g_{\cS_{t,k} }&= t^{-1} \left[\tilde X_{t}^\top  \tilde X_{t}(\Theta_{t,k-1}-\Theta^*) + \tilde X_t^\top  \tilde\epsilon_t\right]_{\cS_{t,k}} = [ A(\Theta_{t,k-1} - \Theta^*) ]_{\cS_{t,k} } + \frac{  [ \tilde X_t^\top  \tilde\epsilon_t ]_{\cS_{t,k}} }{t},
\$
where  $A = t^{-1}\tilde X_{t}^\top \tilde X_{t}  $. We have
\begin{equation*}
\begin{split}
[ A(\Theta_{t,k-1} - \Theta^*) ]_{\cS_{t,k} }  = A_{\cS_{t,k}} (\Theta_{t,k-1} - \Theta^*),
\end{split}
\end{equation*}
where $A_{\cS_{t,k} }$ agrees with  $A$ when the row index belongs to $\cS_{t,k} $ with other rows being zeroes.
Let $\cA = \text{support}\{\Theta^* - \Theta_{t,k-1} \} $ and $\cD = \cA \cup \cS_{t,k}  $, we can see that the cardinality of $\cD$ is no more than $2s+s*$, and
\begin{equation*}
    \begin{split}
     \Big   \| \big [ A(\Theta_{t,k-1} - \Theta^*) \big ]_{\cS_{t,k} } \Big \|_2^2 & = (\Theta_{t,k-1} - \Theta^*)^\top A_{\cS_{t,k} }^\top A_{\cS_{t,k}} (\Theta_{t,k -1} - \Theta^*)
      \\
      & = (\Theta_{t,k-1} - \Theta^*)^\top A_{\cS_{t,k} \cA}^\top A_{\cS_{t,k} \cA} (\Theta_{t,k-1} - \Theta^*)
      \\
      & = (\Theta_{t,k-1} - \Theta^*)_\cD^\top A_{\cD \cD}^\top A_{\cD \cD} (\Theta_{t,k-1} - \Theta^*)_\cD ,
    \end{split}
\end{equation*}
where $A_{\cS_{t,k}  \cA}$ denotes the collection of entries within $\cA$ whose row index belong to $\cS_{t,k}$ and column index belong to $\cA$, and $\cA_{\cD \cD}$ is defined in a similar manner.

Note that when $\cL_t(\cdot)$ satisfies the  $(2s,s)$-SSS condition with constant $L$, \eqref{eq:insample_error} states that for any $\Theta \in \RR^d$ such that $\| \Theta \|_0 \leq 2s$,
\begin{equation*}
\frac{1}{t}  \sum_{i =1}^t   \lv X_i, \Theta  - \Theta^* \rv^2   = (\Theta - \Theta^*)^\top A (\Theta - \Theta^*) \leq L \| \Theta - \Theta^*\|_2^2,
\end{equation*}
which further implies
$$
\| A_{\cD \cD} \|_2 \leq L \text{ and } \| A_{\cD \cD} \|_2^2 \leq L^2.
$$
As a result, we obtain that
\begin{equation*}
    \begin{split}
     \Big   \| \big [ A(\Theta_{t,k-1} - \Theta^*) \big ]_{\cS_{t,k} } \Big \|_2^2  \leq L^2  \| \Theta_{t,k-1} - \Theta^* \|_2^2.
    \end{split}
\end{equation*}
This together with the fact that $\epsilon_i$'s are i.i.d. distributed further imply
\$
\EE\left[ \|g_{\cS_{t,k}}\|_2^2 \right]
&\leq  \frac{2}{t^2} \EE \left [  \tr\left((\Theta_{t,k-1}-\Theta^*)^\top (\tilde X_t^\top \tilde X_t)^2 (\Theta_{t,k-1}-\Theta^*)\right) \right ]
+
\frac{2}{t^2} \EE \Big [  \Big \| \Big ( \sum_{i=1}^t \epsilon_i X_i\Big)_{\cS_{t,k}} \Big \|_2^2 \Big ]
\\
&\leq  \frac{2}{t^2} \EE \left [  \tr\left((\Theta_{t,k-1}-\Theta^*)^\top (\tilde X_t^\top \tilde X_t)^2 (\Theta_{t,k-1}-\Theta^*)\right) \right ]
+
\frac{2s}{t^2} \EE \Big [  \Big \|  \sum_{i=1}^t \epsilon_i X_i\Big \|_\infty^2 \Big ]
\\
&\leq 2L^2 \EE [\|\Theta_{t,k-1}-\Theta^*\|_2^2 ] +  2s\sigma_x^2 /t.
\$
We then obtain
\$
\EE [\| \Theta_{t,k} - \Theta^*  \|_2^2]
&
\leq
2\EE[ \|\Theta_{\cS_{t,k}}^{t,k-1}-\Theta_{\cS_{t,k} }^*\|_2^2 ]+ 2\eta^2\EE[ \|g_{\cS_{t,k} }\|_2^2]  + \EE[ \|\Theta^*_{\cS^*\setminus \cS_{t,k}}\|_2^2 ]
\\
& \leq 2(1+\eta^2 L^2)\EE [\| \Theta_{t,k-1} - \Theta^*  \|_2^2 ] + 4s \eta^2 \sigma_x^2/t  + \|\Theta^*\|_2^2
\\
& \leq 2(1+\eta^2 L^2)\EE [\| \Theta_{t,k-1} - \Theta^*  \|_2^2 ] + 4s \eta^2 \sigma_x^2  + \|\Theta^*\|_2^2 .
\$
Recursively applying the above relationship  under the condition that $\eta \leq \frac{1}{L}$, we obtain
\begin{equation}\label{eq:theta_k_mse}
  \EE [\| \Theta_{t,k} - \Theta^*  \|_2^2] \leq 4^{(t-1)K+k}\|\Theta_{0,0} -\Theta^*\|_2 + \left(4s \eta^2 \sigma_x^2  + \|\Theta^*\|_2^2 \right )
\frac{(4^{(t-1)K + k}  - 1) }{3}.
\end{equation}
 Next, using the descent lemma with $\eta \leq 1/{L}$ acquires
\begin{equation*}
\begin{split}
 & \cL_{t}(\Theta_{t,K}) - \cL_{t}(\hat \Theta_t)   \leq \cL_t(\Theta_{t,0}) - \cL_t (\hat \Theta_t) = \cL_{t}(\Theta_{t-1,K}) - \cL_t (\hat \Theta_t)
 \\
 & \quad  =  \cL_{t-1}(\Theta_{t-1,K}) - \cL_{t-1} ( \hat \Theta_{t-1}) + \Big ( \cL_t (\Theta_{t-1,K} ) - \cL_{t-1} (\Theta_{t-1,K}) \Big )+ u_{t}
 \\
&
 \quad \leq  \cL_{t-1}(\Theta_{t-1,K}) - \cL_{t-1} ( \hat \Theta_{t-1})  +  \frac{1}{t} \ell_t (\Theta_{t-1,K}) -  \frac{1}{t(t-1)} \sum_{i=1}^{t-1}\ell_i(\Theta_{t-1,K})  +  u_{t}
\end{split}
\end{equation*}
where $u_{t} = \cL_{t-1}(\hat \Theta_{t-1})  - \cL_t(\hat \Theta_t)$. In addition, we have that for $t=1$, $\cL_{1}(\Theta_{1,K}) - \cL_{1}(\hat \Theta_1)  \leq \cL_1(\Theta_{1,0}) - \cL_1 (\hat \Theta_1) =\ell_1(\Theta_{1,0})  - \cL_1 (\hat \Theta_1) = \ell_1(\Theta_0) - \cL_1 (\hat \Theta_1)$. By recursively applying the above relation, we have
\begin{equation}\label{eq:recur_01}
 \cL_{t}(\Theta_{t,K}) - \cL_{t}(\hat \Theta_t)  \leq   \sum_{j=1}^t  \frac{1}{j} \ell_j (\Theta_{j-1,K})   +  \sum_{j=1}^t u_{j} = \sum_{j=1}^t  \frac{1}{j} \ell_j (\Theta_{j-1,K})   - \cL_t(\hat \Theta_t)  .
\end{equation}
Because data are independently generated and collected, the $t$-th feature-label pair $(X_t, Y_t)$ is independent of $\Theta_{t-1,K}$ at the $(t-1)$-th epoch. Using the fact that $\| \Theta_{t-1,K} - \Theta^*\|_0 \leq s+s^*$, we have
\begin{align*}
\EE[ \ell_t (\Theta_{t-1,K}) ]
& = \EE \Big [ \Big (\lv X_t , \Theta_{t-1,K} \rv - Y_t \Big  )^2 \Big ] = \EE \Big [ \Big (\lv X_t , \Theta_{t-1,K}  - \Theta^*\rv - \epsilon_t \Big  )^2 \Big ]
\\
& =  \EE \Big [ \lv X_t , \Theta_{t-1,K}  - \Theta^*\rv^2  \Big ]  + \EE[\epsilon_t^2]
\\
& \leq \EE \Big [ \|  X_t \|_\infty^2  ] \EE[  \| \Theta_{t-1,K}  - \Theta^* \|_1^2  \Big ]  + \EE[\epsilon_t^2]
\\
& \leq (s+s^*)\EE \Big [ \|  X_t \|_\infty^2 ] \EE[\| \Theta_{t-1,K}  - \Theta^* \|_2^2  \Big ]  + \EE[\epsilon_t^2]
\\
& \leq  C_0  \EE[ \| \Theta_{t-1,K} -\Theta^*\|_2^2] + \sigma^2
\end{align*}
 where $C_0= (s+s^*) \EE\big[\|X_0\|_\infty^2\big]$. Clearly, rearranging \eqref{eq:recur_01} and taking expectations on both sides, we obtain
\begin{equation*}
\begin{split}
 \EE[ \cL_{t}(\Theta_{t,K})  ] &   \leq  \sum_{j=1}^t  \frac{1}{j} \EE[\ell_j (\Theta_{j-1,K}) ]   \leq  C_0 \sum_{1\leq j\leq t} \frac{1}{j}\EE \big[\| \Theta_{j-1,K} -\Theta^*\|_2^2\big] + \sigma^2 \log(t+1).
\end{split}
\end{equation*}
Finally, applying \eqref{eq:theta_k_mse} to the above inequality, we conclude that
\begin{equation*}
\begin{split}
 \EE [ \cL_{t}(\Theta_{t,K})  ]
&\leq  C_0 \sum_{j\leq t} \frac{ 4^{(j-1)K } }{j}  \times \left( \|\Theta_{0,0} -\Theta^*\|_2 + \frac{ 4s \eta^2 \sigma_x^2  + \|\Theta^*\|_2^2 }{3}  \right)+ \sigma^2 \log (t+1)\\
&\leq C_1
 4^{tK}+ \sigma^2 \log (t+1)
\end{split}
\end{equation*}
for all  $t \leq t_0$, where $C_1= C_0(\|\Theta_{0,0} -\Theta^*\|_2 + ( 4s \eta^2 \sigma_x^2  + \|\Theta^*\|_2^2 )/3   ).$
This completes the proof.
\end{proof}

\subsection{Proof of Proposition~\ref{prop:free_initial_lowrank}}
\begin{proof}
We first analyze the solution sequence $\{\Theta_{t,k} -\Theta^* \}_{t\leq t_0}$ up to time $t_0$. Recall that $\Theta_{t,k} = \Pi_{\cC_s} \{  \Theta_{t,k-1} - \eta \nabla \cL_{t}(\Theta_{t,k-1})\}$. Let $S_{t,k}$ and  $\tilde S_{t,k}$ be the column spaces of $\Theta_{t,k}$ and  $\Theta_{t,k} \cup \Theta^*$ respectively. Let $g = \nabla \cL_{t}(\Theta_{t,k-1})$.
We have
\begin{equation} \label{eq:bound_theta_lowrank}
\begin{split}
& \| \Theta_{t,k} - \Theta^*  \|_\rF^2\\
& = \| \Pi_{\cC_s}\{  \Theta_{t,k-1} - \eta \nabla \cL_{t}(\Theta_{t,k-1})\} - \Theta^*\|_\rF^2\\
& = \|  P_{S_{t,k}}(  \Theta_{t,k-1} - \eta g - \Theta^*) - P_{S^* \cap S_{t,k}^\perp }\Theta^* \|_\rF^2
\\
&\leq \| P_{S_{t,k}}( \Theta_{t,k-1}-\Theta^*)\|_\rF^2 + \eta^2 \| P_{S_{t,k} } ( g ) \|_\rF^2 \\
& \qquad\qquad - 2\langle  P_{S_{t,k} }(\Theta_{t,k-1} -\Theta^* ), \eta  P_{S_{t,k} }(g )\rangle + \|P_{S^* \cap S_{t,k}^\perp } \Theta^*\|_\rF^2\\
&\leq
\|P_{ S_{t,k} }( \Theta_{t,k-1}-\Theta^* ) \|_\rF^2 + \eta^2 \|P_{S_{t,k} } (g ) \|_\rF^2 \\
&\qquad\qquad + 2\eta\, \| P_{S_{t,k} } ( \Theta_{t,k-1} -\Theta^* ) \|_\rF \, \| P_{S_{t,k}} ( g ) \|_\rF + \|P_{S^* \cap S_{t,k}^\perp }  \Theta^*\|_\rF^2.
\end{split}
\end{equation}
For $P_{S_{t,k}}(g )$, we have  
\$
P_{S_{t,k}}(g )
& = P_{S_{t,k}} \left[ t^{-1} \sum_{i\leq t}  X_i \, \left (\langle X_i,  \Theta\rangle - Y_i\right)\right] \\
& = P_{S_{t,k}} \Big [ t^{-1} \sum_{i\leq t}  X_i \lv X_i,\Theta_{t,k-1} - \Theta^* \rv - \epsilon_i \big )  \Big ]
\\
& = t^{-1} \sum_{i\leq t} P_{S_{t,k}}  X_i  \lv P_{S_{t,k} \cup \tilde S_{t,k-1}} X_i,\Theta_{t,k-1} - \Theta^* \rv - t^{-1} \sum_{i\leq t} P_{S_{t,k}}  X_i  \epsilon_i,
\$
implying that
\begin{equation*}
 \| P_{S_{t,k}}(g  )  \|_\rF^2
\leq  2t^{-2}  \left \| \sum_{i\leq t} P_{S_{t,k}}  X_i  \lv P_{S_{t,k} \cup \tilde S_{t,k-1}} X_i,\Theta_{t,k-1} - \Theta^* \rv  \right \|_\rF^2  +2 t^{-2}  \left \|  \sum_{i\leq t} P_{S_{t,k}}  X_i  \epsilon_i  \right \|_\rF^2.
\end{equation*}
Write  $ S_{t,k} \cup \tilde S_{t,k-1} =  \{
S_{t,k}^\perp  \cap \tilde S_{t,k-1} \} \cup S_{t,k}$, where $\{
S_{t,k}^\perp  \cap \tilde S_{t,k-1} \} \cap S_{t,k} = 0$. Then
$$
P_{S_{t,k} \cup \tilde S_{t,k-1} } X_i  = P_{S_{t,k}} X_i + P_{S_{t,k}^\perp  \cap \tilde S_{t,k-1} }X_i \qquad \text{and} \qquad \lv P_{S_{t,k}} X_i , P_{S_{t,k}^\perp  \cap \tilde S_{t,k-1} }X_i \rv = 0,
$$
which further implies
\begin{equation*}
    \begin{split}
    &\left \| \sum_{i\leq t} P_{S_{t,k}}  X_i  \lv P_{S_{t,k} \cup \tilde S_{t,k-1}} X_i,\Theta_{t,k-1} - \Theta^* \rv  \right \|_\rF^2
    \\
    & =  \left \| \sum_{i\leq t} P_{S_{t,k} \cup \tilde S_{t,k-1} }  X_i  \lv P_{S_{t,k} \cup \tilde S_{t,k-1}} X_i,\Theta_{t,k-1} - \Theta^* \rv  \right \|_\rF^2
      \\
      & \quad -   \left \| \sum_{i\leq t} P_{
S_{t,k}^\perp  \cap \tilde S_{t,k-1} }   X_i \lv P_{S_{t,k} \cup \tilde S_{t,k-1}} X_i,\Theta_{t,k-1} - \Theta^* \rv  \right \|_\rF^2
\\
& \leq  \left \| \sum_{i\leq t} P_{S_{t,k} \cup \tilde S_{t,k-1} }  X_i \lv P_{S_{t,k} \cup \tilde S_{t,k-1}} X_i,\Theta_{t,k-1} - \Theta^* \rv  \right \|_\rF^2.
    \end{split}
\end{equation*}
Combining the above inequalities, we obtain
\begin{equation}\label{eq:bound_theta_lowrank_1}
 \| P_{S_{t,k}}(g  )  \|_\rF^2
\leq  2t^{-2}  \left \| \sum_{i\leq t} P_{S_{t,k} \cup \tilde S_{t,k-1} }  X_i \lv P_{S_{t,k} \cup \tilde S_{t,k-1}} X_i,\Theta_{t,k-1} - \Theta^* \rv  \right \|_\rF^2 + 2t^{-2} \left \|  \sum_{i\leq t} P_{S_{t,k}}  X_i  \epsilon_i  \right \|_\rF^2.
\end{equation}
Let $z_i = \text{vec}( P_{S_{t,k} \cup \tilde S_{t,k-1}}  X_i)$ and $\theta_{t,k-1} = \text{vec}(\Theta_{t,k-1} - \Theta^* )$, then we can  bound the first term as
\#\label{eq:zz_bound}
& \left \| t^{-1}\sum_{i\leq t} P_{S_{t,k} \cup \tilde S_{t,k-1} }  X_i  \lv P_{S_{t,k} \cup \tilde S_{t,k-1}} X_i,\Theta_{t,k-1} - \Theta^* \rv \right \|_\rF^2 \nn
\\
&= \theta_{t,k-1}^\top \Big ( \frac{1}{t} \sum_{i \leq t}  z_i z_i^\top \Big)  \Big ( \frac{1}{t} \sum_{i \leq t}  z_i z_i^\top \Big) \theta_{t,k-1} \nn \\
& \leq \left\|\frac{1}{t}\sum_{i\leq t} z_i z_i^\T \right\|_2^2 \| \Theta_{t,k-1} - \Theta^*\|_\rF^2.
\#
To further bound the right hand side (RHS) of the inequality above, we need to bound $\|\sum_{i\leq t}z_i z_i^\T/t\|_2$. Because $\cL_t(\cdot)$ satisfies the $(2s+s^*)$-LowRankSS with parameter $L$, following the proof of \eqref{eq:insample_lowrank}, we have
\$
 \frac{1}{ t}  \sum_{i \leq t}   \lv    P_{S_{t,k} \cup \tilde S_{t,k-1} }  X_i, P_{S_{t,k} \cup \tilde S_{t,k-1} } A \rv^2  \leq L \| P_{S_{t,k} \cup \tilde S_{t,k-1} } A\|_\rF^2,
\$
for any matrix $A\in \RR^{d_1\times d_2}$. Let $a=\vec(A)$. This further implies
\$
 a^\T \left(\frac{1}{ t} \sum_{i \leq t} z_i z_i^\T\right) a
 &= \frac{1}{ t}  \sum_{i \leq t}   \lv    P_{S_{t,k} \cup \tilde S_{t,k-1} }  X_i, P_{S_{t,k} \cup \tilde S_{t,k-1} } A \rv^2 \\
 &\leq L \| P_{S_{t,k} \cup \tilde S_{t,k-1} } A\|_\rF^2  \leq L \|P_{S_{t,k} \cup \tilde S_{t,k-1} } \|_2^2 \|A\|_\rF^2  \leq L \|a\|_2^2,
\$
and thus
\$
\left\|\frac{1}{ t} \sum_{i \leq t} z_i z_i^\T\right\|_2\leq L.
\$
Plugging the above inequality into \eqref{eq:zz_bound}, we obtain
\$
\left \| t^{-1}\sum_{i\leq t} P_{S_{t,k} \cup \tilde S_{t,k-1} }  X_i  \lv P_{S_{t,k} \cup \tilde S_{t,k-1}} X_i,\Theta_{t,k-1} - \Theta^* \rv \right \|_\rF^2 \leq L^2 \| \Theta_{t,k-1} - \Theta^*\|_\rF^2.
\$
Substituting the above inequality into~\eqref{eq:bound_theta_lowrank_1} and taking expectations on both sides, we have
\begin{equation*}
\begin{split}
\EE[\| P_{S_{t,k}}(g  ) \|_\rF^2 ] &
\leq  2L^2 \EE \|\Theta_{t,k-1}-\Theta^*\|_\rF^2  +
\frac{ 2 }{t^2} \EE  \left  \| P_{S_{t,k}}  \Big (  \sum_{i= 1}^t  X_i  \epsilon_i \Big  ) \right  \|_\rF^2
\\
& \leq  2L^2 \EE \|\Theta_{t,k-1}-\Theta^*\|_\rF^2 +
\frac{ 2 s }{t^2} \EE  \left  \|  \sum_{i= 1}^t  X_i  \epsilon_i  \right  \|_2^2
\\
& \leq  2L^2 \EE \|\Theta_{t,k-1}-\Theta^*\|_\rF^2 +
2s\sigma_X^2  /t,
\end{split}
\end{equation*}
where the second inequality uses the facts that $\rank(P_{S_{t,k}}   (  \sum_{i= 1}^t  X_i  \epsilon_i  )) \leq s$ and $\| A\|_\rF^2 \leq \rank(A) \| A\|_2^2$, and the last inequality comes from Assumption~\ref{assumption:noise_lowrank}.
Following the analysis of Proposition~\ref{prop:free_initial_sparse}, we conclude that
\begin{equation*}
\begin{split}
\EE[ \cL_{t}(\Theta_t)  ] \leq  C_1
 4^{tK}+ \sigma^2 \log (t+1)
\end{split}
\end{equation*}
for all  $t \leq t_0$, where $C_1= C_0(\|\Theta_0 -\Theta^*\|_\rF^2 + ( 4s \eta^2 \sigma_X^2  + \|\Theta^*\|_\rF^2 )/3   ) $ and $C_0= (s+s^*)   \EE\big[\|X_0\|_2^2\big] $.
This completes the proof.
\end{proof}

\end{document}